\documentclass[]{ecai}

\usepackage{latexsym}
\usepackage{amssymb}
\usepackage{amsmath}
\usepackage{amsthm}
\usepackage{enumitem}
\usepackage{graphicx}
\usepackage{color}
\usepackage{mathtools}
\usepackage{amsthm}
\usepackage{enumitem}
\usepackage{soul}

\DeclareRobustCommand{\KL}[2]{\ensuremath{\mathrm{KL}\left(#1\;\|\;#2\right)}}

\DeclareRobustCommand{\Ep}[2]{\ensuremath{\mathds{E}_{#1}\left[#2\right]}}
\DeclareRobustCommand{\E}[1]{\ensuremath{\mathds{E}\left[#1\right]}}

\DeclareRobustCommand{\varp}[2]{\ensuremath{\mathrm{var}_{#1}\left[#2\right]}}

\newcommand{\Norm}{\mathcal{N}}

\newcommand{\veps}{\varepsilon}

\newcommand{\deq}{\triangleq}  %

\DeclareRobustCommand{\sd}[1]{\color{black!80!white}\scriptstyle #1}

\newcommand{\mcA}{\mathcal{A}}

\newcommand{\mcD}{\mathcal{D}}

\newcommand{\mcF}{\mathcal{F}}

\newcommand{\mcH}{\mathcal{H}}

\newcommand{\mcL}{\mathcal{L}}

\newcommand{\mcO}{\mathcal{O}}

\newcommand{\mcS}{\mathcal{S}}

\newcommand{\mcX}{\mathcal{X}}
\newcommand{\mcY}{\mathcal{Y}}
\newcommand{\mcZ}{\mathcal{Z}}

\newcommand{\mdE}{\mathds{E}}

\newcommand{\mdN}{\mathds{N}}

\newcommand{\mdP}{\mathds{P}}

\newcommand{\mdR}{\mathds{R}}

\usepackage{tikz}

\makeatletter

\usetikzlibrary{shapes}
\usetikzlibrary{fit}
\usetikzlibrary{chains}
\usetikzlibrary{arrows}

\tikzstyle{latent} = [circle,fill=white,draw=black,inner sep=1pt,
minimum size=20pt, font=\fontsize{10}{10}\selectfont, node distance=1]
\tikzstyle{obs} = [latent,fill=gray!25]
\tikzstyle{const} = [rectangle, inner sep=0pt, node distance=1]
\tikzstyle{factor} = [rectangle, fill=black,minimum size=5pt, inner
sep=0pt, node distance=0.4]
\tikzstyle{det} = [latent, diamond]

\tikzstyle{plate} = [draw, rectangle, rounded corners, fit=#1]
\tikzstyle{wrap} = [inner sep=0pt, fit=#1]
\tikzstyle{gate} = [draw, rectangle, dashed, fit=#1]

\tikzstyle{caption} = [font=\footnotesize, node distance=0] %
\tikzstyle{plate caption} = [caption, node distance=0, inner sep=0pt,
below left=5pt and 0pt of #1.south east] %
\tikzstyle{factor caption} = [caption] %
\tikzstyle{every label} += [caption] %

\makeatother
\usetikzlibrary{backgrounds}
\usetikzlibrary{calc,quotes,arrows.meta}
\usetikzlibrary{shapes.misc,positioning,calc,graphs,arrows}

\pgfdeclarelayer{edgelayer}
\pgfdeclarelayer{nodelayer}
\pgfsetlayers{edgelayer,nodelayer,main}

\definecolor{hexcolor0xbfbfbf}{rgb}{0.749,0.749,0.749}

\tikzset{>=latex}
\tikzstyle{none}   = [inner sep=0pt]
\tikzstyle{line}   = [ -, thick, shorten <=1pt, shorten >=1pt ]
\tikzstyle{arrow}  = [ ->, thick, shorten <=1pt, shorten >=1pt ]
\tikzstyle{ardash} = [ dashed, ->, thick, shorten <=1pt, shorten >=1pt ]

\tikzstyle{box} = [rectangle, minimum width=1.5cm, minimum height=1.5cm,text centered, draw=black, inner sep=7pt]
\tikzstyle{neuron} = [circle, minimum width=4mm, very thick, draw=blue!80!black]

\tikzstyle{empty}=[circle,opacity=0.0,text opacity=1.0,inner sep=0pt]
\tikzstyle{box}=[rectangle,fill=White,draw=Black]
\tikzstyle{filled}=[circle,thick,fill=hexcolor0xbfbfbf,draw=Black]
\tikzstyle{hollow}=[circle,thick,fill=White,draw=Black]
\tikzstyle{param}=[rectangle,fill=Black,draw=Black,inner sep=0pt,minimum width=4pt,minimum height=4pt]
\tikzstyle{paramhollow}=[rectangle,thick,fill=White,draw=Black,inner sep=0pt,minimum
width=4pt,minimum height=4pt]

\usepackage{pgfplots}                               %
\pgfplotsset{compat=newest}
\pgfplotsset{plot coordinates/math parser=false}
\usepgfplotslibrary{statistics}
\newlength\figureheight
\newlength\figurewidth
\newlength\figureheightsmall
\newlength\figurewidthsmall
\definecolor{POSTcolor}{rgb}{0.48, 0.20, 0.58} %
\definecolor{Qcolor}{rgb}{0.00, 0.53, 0.22} %

\makeatletter
\tikzset{
  prefix after node/.style={
    prefix after command={\pgfextra{#1}}
  },
  /semifill/ang/.store in=\semi@ang,
  /semifill/ang=0,
  semifill/.style={
    circle, draw,
    prefix after node={
      \typeout{aaa \semi@ang}
      \let\nodename\tikz@last@fig@name
      \fill[/semifill/.cd, /semifill/.search also={/tikz}, #1]
        let \p1 = (\nodename.north), \p2 = (\nodename.center) in
        let \n1 = {\y1 - \y2} in
        (\nodename.\semi@ang) arc [radius=\n1, start angle=\semi@ang, delta angle=180];
    },
  }
}
\makeatother

\usetikzlibrary{shapes.geometric, arrows, fit, automata, positioning, calc, arrows.meta}
\usetikzlibrary{backgrounds}
\usetikzlibrary{positioning,arrows.meta,shapes.geometric,backgrounds,fit,calc}
\tikzstyle{box} = [rectangle, rounded corners, minimum width=3cm, minimum height=1cm,text centered, draw=black, inner sep=7pt]
\tikzset{
  plate/.style={draw, shape=rectangle, rounded corners=0.5ex, thick,
    minimum width=3.1cm, text width=3.1cm, align=right, inner sep=10pt, inner ysep=10pt, 
    append after command={node[above left= 1pt of \tikzlastnode.south east] {#1}}}
}

\usepackage{dsfont}
\usepackage{adjustbox}
\usepackage{siunitx}
\usepackage{multirow}
\usepackage{arydshln}

\allowdisplaybreaks
\DeclareRobustCommand{\sd}[1]{\color{black!80!white}\scriptstyle #1}

\usepackage{algorithmic}
\usepackage{wrapfig}
\usepackage{microtype}
\usepackage{subfigure}

\DeclareMathOperator*{\argmin}{arg\,min}
\DeclareMathOperator*{\argmax}{arg\,max}

\newcommand{\lr}[1]{\left (#1\right)}

\let\P\undefined
\NewDocumentCommand{\P}{o}{\mathbb P{\IfValueT{#1}{\lr{#1}}}}

\newcommand{\y}{\mathbf y}

\newcommand{\p}{\mathbf p}

\let\emptyset\varnothing

\newtheorem{theorem}{Theorem}%

\newtheorem{lemma}{Lemma}

\newcommand{\St}{\mathcal{S}}
\newcommand{\Ac}{\mathcal{A}}

\newcommand{\BibTeX}{B\kern-.05em{\sc i\kern-.025em b}\kern-.08em\TeX}

\usepackage{latexsym}
\usepackage{amssymb}
\usepackage{amsmath}
\usepackage{amsthm}
\usepackage{microtype} 
\usepackage{booktabs}
\usepackage{multirow}
\usepackage{enumitem}
\usepackage{color}
\usepackage{algorithm}
\usepackage{lipsum}

\makeatletter
\newcommand\blfootnote[1]{%
    \begingroup
    \renewcommand{\thefootnote}{}%
    \renewcommand{\@makefnmark}{}%
    \footnote{#1}%
    \addtocounter{footnote}{-1}%
    \endgroup
}
\makeatother

\usepackage{hyperref}
\usepackage[capitalize,noabbrev]{cleveref}

\crefname{equation}{}{}
\usepackage{xr-hyper}
\usepackage{nicefrac}



\begin{document}

\paperid{8131}

\title{Deep Exploration with PAC-Bayes}

\author[A]{\fnms{Bahareh}~\snm{Tasdighi}\thanks{Corresponding author: \texttt{tasdighi@imada.sdu.dk}\\ Full paper including appendix at \url{https://arxiv.org/abs/2402.03055}}}
\author[A]{\fnms{Manuel}~\snm{Haussmann}}
\author[A]{\fnms{Nicklas}~\snm{Werge}}
\author[A]{\fnms{Yi-Shan}~\snm{Wu}}
\author[A]{\fnms{Melih}~\snm{Kandemir}}

\address[A]{Department of Mathematics and Computer Science\\ University of Southern Denmark}

\def\promise#1{{\color{green!40!black}#1}}

\begin{frontmatter}

\begin{abstract}
Reinforcement learning (RL) for continuous control under delayed rewards is an under-explored problem despite its significance in real-world applications. Many complex skills are based on intermediate ones as prerequisites. For instance, a humanoid locomotor must learn how to stand before it can learn to walk. To cope with delayed reward, an agent must perform deep exploration. However, existing deep exploration methods are designed for small discrete action spaces, and their generalization to state-of-the-art continuous control remains unproven. 
We address the deep exploration problem for the first time from a PAC-Bayesian perspective in the context of actor-critic learning. 
To do this, we quantify the error of the Bellman operator through a PAC-Bayes bound, where a bootstrapped ensemble of critic networks represents the posterior distribution, and their targets serve as a data-informed function-space prior. 
We derive an objective function from this bound and use it to train the critic ensemble. Each critic trains an individual soft actor network, implemented as a shared trunk and critic-specific heads.
The agent performs deep exploration by acting epsilon-softly on a randomly chosen actor head. Our proposed algorithm, named {\it PAC-Bayesian Actor-Critic (PBAC)}, is the only algorithm to consistently discover delayed rewards on continuous control tasks with varying difficulty.
\end{abstract}

\end{frontmatter}

\section{Introduction}\label{sec:intro}

Many autonomous agents of the future will be expected to solve broadly defined continuous control tasks. For example, humanoid robots are populating factories and warehouses. These versatile platforms will be trained to automate increasingly cognitively demanding processes with sparse or delayed rewards. Their physical capabilities are at present far from optimal because existing reinforcement learning practices do not support solving difficult continuous control tasks in such reward situations. The best practices of model-based reinforcement learning only generalize across a large set of tasks when rewards are dense \citep{Hafner_2025,scannell2025discrete,zhang2021made}. Their dependency on multi-stage latent planning on powerful world models inflates computational requirements and limits deployment in low-energy use cases, such as soft robotics.

Solving control tasks with delayed rewards without an environment model requires smart exploration heuristics. Common algorithms like Soft Actor-Critic \citep{haarnoja2018soft} rely on Boltzmann exploration, while TD3 \citep{fujimoto2018addressing} uses Gaussian perturbation. Such random and undirected exploration methods are prohibitively sample-inefficient in delayed reward scenarios. \emph{Deep exploration} \cite{osband2016deep}  suggests directing the search towards under-explored areas of the state-action space proportionally to the uncertainty of an agent’s predictions regarding these areas. Several approaches formulate deep exploration as approximate Bayesian inference of the action-value estimate of the next state, and use the resulting probabilistic representation for search \citep{fellows2024ben,osband2016deep,osband2018randomized,touati2020randomized}. Pseudo-count approaches estimate the visitation count of state-action pairs by distilling a randomly initialized predictor into a parametric auxiliary network \citep{burda2018exploration,nikulin2023anti,yang2024exploration}. These strategies do not work in continuous control scenarios, mainly due to the difficulty in quantifying the uncertainty of a high-capacity critic network necessary to model a highly non-linear and continuous state-action space.

\begin{figure}
\centering
\pgfdeclarelayer{background}
\pgfsetlayers{background,main}
\adjustbox{max width=0.99\columnwidth}{
\begin{tikzpicture}[
    concept node/.style={draw=teal!80!black, thick, fill=teal!20!white, rounded corners, minimum width=8em, minimum height=2em, align=center, font=\sffamily},
    method node/.style={draw=brown!80!black, thick, fill=brown!40!white, rounded corners, minimum width=8em, minimum height=4em, align=center, font=\sffamily},
    key concept/.style={draw=blue!60!black, thick, fill=blue!15!white, rounded corners, minimum width=10em, minimum height=3.5em, align=center, font=\sffamily\bfseries},
    model node/.style={draw=orange!90!black, ultra thick, fill=orange!20, rounded corners, minimum width=14em, minimum height=3em, font=\bfseries\sffamily, align=center},
    approach node/.style={draw=purple!40!black, thick, fill=purple!10!white, rounded corners, minimum width=11em, minimum height=3.5em, align=center, font=\sffamily},
    group/.style={draw=gray!50, rounded corners, thick, inner sep=1.0em},
    arrow/.style={-{Stealth[scale=1.2]}, thick, draw=gray!90},
    node distance=1.7em
]
    \node[concept node] (modelfree) {Model-free Learning};
    \node[concept node, left=4em of modelfree] (delayed) {Delayed Reward};
    \node[concept node, right=4em  of modelfree] (control) {Continuous Control};
 
    \node[group, fit=(delayed) (modelfree) (control), label={[anchor=north west,font=\bfseries\itshape]north west:{\color{white!30!black} Our Scope}}, yshift=0.3em, dashed] {};

    \node[approach node, below=3em of delayed] (exploration) {\textsc{Deep exploration}\\ \scriptsize{via uncertainty quantification}};
    \node[approach node, below=3em of modelfree] (selfreference) {\textsc{Self-referencing}\\\scriptsize{via bootstrapping}};
    \node[approach node, below=3em of control] (deeplearning) {\textsc{Non-linear dynamics}\\\scriptsize{via deep learning}};

    \node[group, fit=(exploration) (selfreference) (deeplearning), label={[anchor=north west,font=\bfseries\itshape]north west:{\color{white!30!black} Approaches}}, yshift=0.3em, dashed] {};

    \node[method node, below left=3.5em and -4em of selfreference] (critic) {\textsc{\bf Critic}\\ \footnotesize{PAC-Bayesian}\\\footnotesize{Policy Evaluation}};
    \node[method node, below right=3.5em and -4em of selfreference] (actor) {\textsc{\bf Actor}\\ \footnotesize{Posterior Sampling}};

    \begin{pgfonlayer}{background}
    \node[group, fit=(critic) (actor), label={[anchor=north,font=\bfseries]north:{\color{orange!82!black}{\Large\textsc{PAC-Bayesian Actor Critic (PBAC)}}}}, yshift=0.7em, fill=orange!10!white,draw=orange!80!black, minimum height=7em, minimum width=25em] (pbac) {};
    \end{pgfonlayer}
    
    \draw[arrow] (delayed) -- (exploration);
    \draw[arrow] (modelfree) -- (selfreference);
    \draw[arrow] (control) -- (deeplearning);

    \draw[arrow,out=270, in=180] ($(exploration.south) -(1,0)$) to ($(pbac.west) +(0,0.5)$);
    \draw[arrow] (selfreference) -- (pbac);
    \draw[arrow,out=270, in=0] ($(deeplearning.south) + (1,0)$) to ($(pbac.east) +(0,0.5)$);
    
\end{tikzpicture}}
\caption{From model-free continuous control from delayed rewards to our proposed approach. PAC-Bayesian analysis enables the training of probabilistic action-value functions with bootstrapping. Posterior sampling is a theoretically grounded approach to train an actor network using a probabilistic critic. Together, these components yield \emph{PAC-Bayesian Actor Critic (PBAC)}.}
\label{fig:concept}
\end{figure}
\vspace{0.5cm}
We claim that model-free continuous control from delayed rewards cannot be solved by handcrafted adaptations of existing approaches. It requires an original combination of subject areas not yet in strong connection yet. We suggest the deductive reasoning in Figure \ref{fig:concept} to identify principles of successful algorithm design for this challenging scenario. Handling delayed rewards requires deep exploration, a model-free approach necessitates bootstrapping, and capturing non-linear continuous dynamics needs deep learning. Deep exploration is only possible if action-value uncertainties are accurately quantified. 
Bootstrapping is a \emph{self-referencing} method to predict future rewards. Bayesian inference can quantify posterior probabilities relative to a prior distribution. Using these two relativist approaches---bootstrapping with Bayesian methods---together may have detrimental side effects. While bootstrapping errors on expected returns can be corrected by observed rewards, the same does not happen for bootstrapped uncertainty estimates, as these are not directly observable quantities.  

\begin{figure*}%
    \centering
    \includegraphics[width=0.65\linewidth]{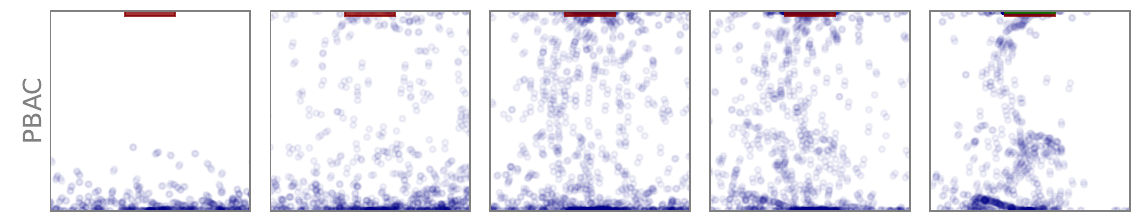}
    \includegraphics[width=0.65\linewidth]{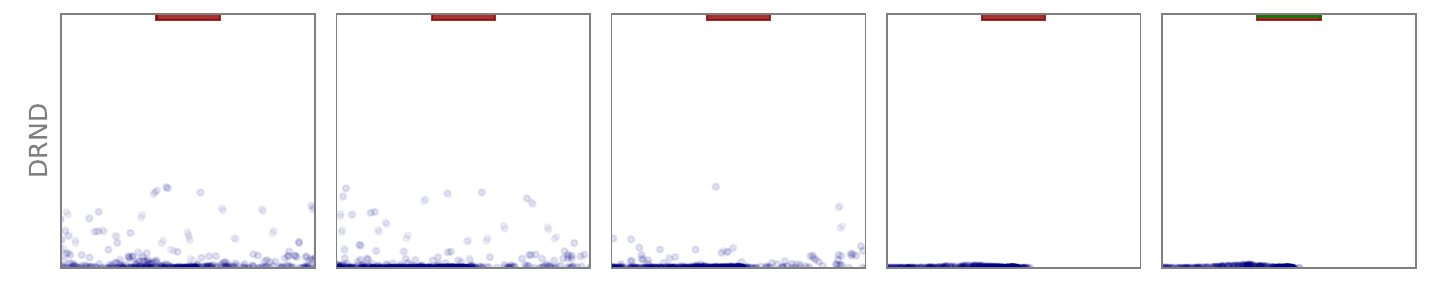}
    \includegraphics[width=0.65\linewidth]{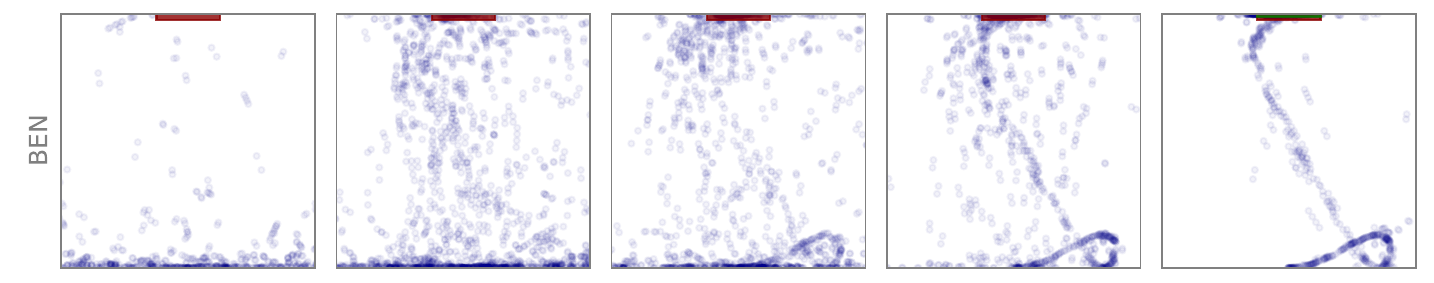}
    \caption{\emph{Deep exploration.} Cartpole state and angle dimension exploration patterns visualized across training.
    Random trajectories are displayed at \num{20000} gradient step intervals. 
    DRND~\citep{yang2024exploration} fails to explore sufficiently.
    Both BEN~\citep{fellows2024ben} and PBAC (ours) quickly identify the target region (red box). While BEN fully exploits, our model achieves the same performance (see \cref{tab:result_table}), while remaining flexible enough to explore.
    }
    \label{fig:teaser_exploration}
    \vspace{1.0em}
\end{figure*}
Our hypothesis is that Probably Approximately Correct (PAC)-Bayesian \citep{alquier2024user,mcallester1999pac} analysis effectively derives principled objective functions for training probabilistic action-value functions with bootstrapping. It provides frequentist uncertainty estimates of generalization performance based on an objective notion of uncertainty, unlike the relativist Bayesian approach.
PAC-Bayes admits data-informed priors for model capacity estimation \cite{ambroladze2006tighter}, creating practical bounds even for deep learning \cite{dziugaite2017computing,perez2021tighter,wu2024recursive}. We quantify the Bellman operator error using a Catoni-type bound \citep{catoni2007pac} for the first time in a policy evaluation context and in an original form. We use a bootstrapped ensemble of critic networks as an empirical posterior distribution and build a data-informed function-space prior from their target networks. To train this ensemble, we derive a simple and intuitive objective function. 

We adopt the PAC-Bayes trained policy network ensemble as a critic to perform actor training. Our algorithm design builds on posterior sampling, a theoretically justified methods for balancing exploration and exploitation using probabilistic action-value estimates \cite{fellows2024ben,osband2013more,osband2019deep}. As the actor, we employ a neural network with a shared trunk and multiple squashed Gaussian heads, each assigned to a separate critic. (See \cref{appsec:hyper} for details on this architecture.)

We conduct experiments on a diverse set of continuous control tasks, both established and novel, some of which features reward structures significantly more challenging than those found in prior art. Our \emph{PAC-Bayesian Actor-Critic (PBAC)} algorithm is the only model capable of solving most of these tasks, whereas both state-of-the-art and well-established methods fail in several. \cref{fig:teaser_exploration} visualizes PBAC's deep exploration, effective exploitation, and maintained flexibility for state space exploration and reward region adaptation.

\section{Background and prior art}\label{sec:background}

\subsection{Reinforcement learning}
Consider a measurable space $(\St, \sigma(\St))$ consisting of a set of states~$\St$ an agent may be in and its corresponding $\sigma$-algebra~$\sigma(\St)$ comprising all measurable subsets of $\St$ and an action space~$\Ac$ from which an agent selects an action to interact with its environment. 
A Markov Decision Process (MDP) \citep{puterman2014markov} is a tuple ${M=\langle\St, \Ac, r, P, P_0, \gamma\rangle}$, where 
$r: \St \times \Ac \rightarrow [0,R]$ is a bounded reward function, 
${P: \St \times \Ac  \times \sigma(\St) \rightarrow [0,1]}$ is a state transition kernel providing the probability of the next state $s' \in \St$ for the current state-action pair $(s,a)$,   
$P_0: \sigma(\St) \rightarrow [0,1]$ is an initial state distribution, and $\gamma \in [0,1)$ is a discount factor. 
The MDP is assumed to be communicating with a unique stationary policy. 
For a policy $\pi: \mcS \to \mcA$, RL aims to maximize the expected sum of discounted rewards
\begin{equation*}
\pi^* \in \argmax_{\pi \in \Pi} \mdE_{\tau_{\pi}}\Big[\sum_{t=0}^\infty \gamma^t r(s_t) \Big],
\end{equation*}
where the expectation is taken with respect to a trajectory ${\tau_{\pi} := (s_0, a_0, s_1, a_1, s_2, a_2, \ldots )}$ of states and actions that occur when a policy $\pi$ selected from a feasible set $\Pi$ is employed.
The exact Bellman operator for a policy $\pi$ is 
\begin{equation*}
    (T_\pi Q)(s,a) = r(s,a) + \gamma\Ep{s' \sim P(\cdot|s,a)}{Q(s',\pi(s'))},
\end{equation*}
where $Q: \St \times \Ac \rightarrow \mdR$.\footnote{Note the parentheses around $T_\pi Q$. $Q$ is an argument to the operator $T_\pi$.} 
The operator's unique fixed point is the true action-value function~$Q_\pi$ which maps a state-action pair $(s,a)$ to the discounted reward under $\pi$. Formally ${(T_\pi Q)(s,a) = Q(s,a)}$ if and only if ${Q(s,a)=Q_\pi(s,a), \forall (s,a)}$. Any other $Q$ incurs a \emph{Bellman error} ${(T_\pi Q)(s,a) - Q(s,a)}> 0$ for some $(s,a)$. We define the squared ${P_\pi\text{-weighted}}$ ${L_2\text{-norm}}$ as ${\lVert f \rVert_{P_\pi}^2=\int_{s \in \St} \lvert f(s) \rvert^2 d P_\pi(s)}$. We approximate the true action-value function $Q_\pi$ by one-step Temporal Difference (TD) learning by minimizing
\begin{align}
    L(Q) = ||T_\pi Q - Q ||_{P_\pi}^2 \label{eq:td0}
\end{align}
with respect to $Q$ given a data set $\mathcal{D}$. 
As the transition kernel is typically unknown, $L(Q)$ is estimated using Monte Carlo samples
\begin{equation*}
    \widetilde{L}(Q)\!=\!\Ep{s\sim P_\pi,s'\sim P(\cdot |s,\pi(s))}{((\widetilde{T}_\pi Q)(s,\pi(s),s')\!-\!Q(s,\pi(s)))^2}\!\!,
\end{equation*}
with the \emph{sample Bellman operator} \citep{fellows2021bayesian}, defined as 
\begin{equation*}
    (\widetilde{T}_\pi Q)(s,a,s') = r(s,a) + \gamma Q(s',\pi(s')).
\end{equation*}
An approximator $Q$ is inferred with respect to tuples~${(s,a,s')}$ stored in a replay buffer~$\mathcal{D}$ by minimizing an empirical estimate of $\widetilde{L}(Q)$:
\begin{align*}
    \widehat{L}_{\mathcal{D}}(Q) &:= \frac{1}{|\mathcal{D}|} \sum_{(s,a,s') \in \mathcal{D}} \left((\widetilde{T}_\pi Q)(s,a,s') -Q(s,a)\right)^2.
\end{align*}
The approximated $Q$ function is used as a training objective in a subsequent policy optimization step. 
Temporal difference learning identifies the globally optimal policy when used for discrete MDPs. In the continuous case, it finds a local optimum \citep{bertsekas1996neuro} and quite often in a sample-inefficient way in real-world applications.

\subsection{Deep exploration}
The classical optimal control paradigm prescribes a rigorous design of cost functions, i.e., negative rewards, that provide a dense signal about model performance at each stage of model fitting \citep{kirk2004optimal,stengel1994optimal}. This approach is akin to detailed loss designs, a common practice in deep learning. However, next-generation AI problems demand agents with high cognitive capabilities that autonomously learn smart strategies to reach distant rewards without being trapped by intermediate small rewards or prematurely abandoning the search. These approaches perform \emph{deep exploration} \citep{osband2016deep}, which means case-specific search based on uncertainty regarding the knowledge of a particular environmental situation, e.g., the value of a state-action pair.
Deep exploration approaches can be classified into four categories based on their assumed knowledge of environment dynamics for making this estimation:

\textbf{(i) Model-based approaches} 
learn to infer the state transition kernel by a function approximator ${P_\phi(\cdot|s,a) \approx P(\cdot|s,a)}$ and use its uncertainty to generate {\it intrinsic rewards} \citep{singh2004intrinsic} to perform curiosity-driven exploration, which has remarkable roots in biological intelligence \citep{modirshanechi2023curiosity,schmidhuber2010formal}. VIME \citep{houthooft2016vime} incentivizes exploration of states that bring more information gain on the estimated kernel, while BYOL-Explore \citep{guo2022byol} and BYOL-Hindsight \citep{jarrett2022curiosity} choose states with higher prediction error. Approaches such as UBE \citep{odonoghue2018uncertainty}, Exact UBE \citep{luis2023model}, and QU-SAC \citep{luis2023value} choose the target quantity as the sample Bellman operator $\widetilde{T}_\pi Q$. 

\textbf{(ii) Pseudo-count approaches}
learn to predict the relative visitation frequencies of state-action pairs without needing to predict state transitions \citep{bellemare2016unifying}. These frequencies are then used to generate intrinsic rewards  directing the exploration towards less frequent state-action pairs. The commonplace approach, called {\it Random Network Distillation (RND)} \citep{burda2018exploration}, distills a fixed network with randomly initialized weights into another one with learned parameters. The distillation serves as a pseudo-count to generate intrinsic rewards. State-of-the-art variants include SAC-RND \citep{nikulin2023anti} and SAC-DRND \citep{yang2024exploration}.

\textbf{(iii) Randomized value iteration approaches}
estimate the uncertainty around the approximate $Q$ function by modeling it as a random variable and performing posterior inference on observed Bellman errors. Proposed inference methods include Bayesian linear regression \citep{osband2016generalization}, mean-field  \citep{lipton2018bbq} or structured \citep{touati2020randomized} variational inference on Bayesian neural networks, and Bayesian deep bootstrap ensembles with various priors \citep{osband2016deep,osband2018randomized}. \citet{fellows2021bayesian} demonstrated that these approaches infer a posterior on the Bellman target $T_\pi Q$ rather than $Q$, proposing a solution via separate $Q$ networks on the posterior predictive mean of $T_\pi Q$. Their follow-up, \emph{Bayesian Exploration Networks (BEN)} \citep{fellows2024ben}, applies structured variational inference with Bayesian heteroscedastic neural  on the Bellman target.
Other ensemble-based approaches aim to estimate epistemic uncertainty \citep{bai2021principled} or account for environmental noise \citep{liu2024ovd} without a Bayesian interpretation.

\textbf{(iv) Policy randomization approaches}
propose different perturbations on the learned models to randomize the employed policy without aiming to quantify the uncertainty of a quantity. Proposed solutions include perturbing the $Q$-network weights for discrete action spaces \citep{fortunato2018noisy}, actor-network weights \citep{plappert2018parameter} for continuous control, and transforming the policy distribution by normalizing flows \citep{mazoure19leveraging}.

In this work, we follow the randomized value iteration approach due to its balance between computational feasibility and theoretical soundness.  Bayesian model averaging provides the optimal decision rule for an approximate $Q$ function with  partial observation on an MDP with reliable uncertainty estimates \citep{CoxHink74}. 

Bootstrapping is a self-referential process: predictions are used to create pseudo-labels that are in turn used to refine predictions. Bayesian inference is a relativist interpretation of uncertainty: the reliability of the posterior distribution is with reference to the reliability of the prior. Our central hypothesis is that using a relativist uncertainty quantification approach to perform a self-referential learning scheme will accumulate miscalibrated uncertainties. Acting based on such uncertainties will limit the performance of an online learning agent. Both BEN and BootDQN-P perform bootstrapping using a Bayesian account of the uncertainty around the Bellman target. A theoretical explanation of why a Bayesian Bellman operator may lead to training artifacts can be found in \cite{fellows2021bayesian}.
Because BootDQN-P suffers from this weakness, it is overshadowed by distributional approaches in discrete control \citep{dabney2018implicit}. We are the first to rigorously test it in continuous control, and its performance is not competitive. BEN mitigates this problem only partially through a long array of complicated downstream approximations during posterior computation. The effects of the resulting approximation errors on the quality of the uncertainty estimates are unknown. As we demonstrate in our experiments, its empirical performance is suboptimal, which highlights how deeply rooted the problem with using a Bayesian Bellman operator is. Our PBAC is built on a frequentist approach to uncertainty quantification, which boosts deep exploration performance. 
The use of PAC-Bayes in reinforcement learning is limited to theoretical analysis \citep{fard2012pac} and preliminary work aimed at improving training robustness \citep{tasdighi2023pac}.

\subsection{PAC-Bayesian analysis and learning} 

Given input $\mcZ$, output $\mathcal{Y}$, and hypothesis spaces ${\mathcal{H}=\{h: \mcZ \rightarrow \mathcal{Y} \}}$, and loss function ${\ell: \mcY \times \mathcal{Y} \rightarrow [0,B]}$, we define the empirical risk $\widehat{L}$ on observations ${\mcO=\{(z_i,y_i): (z_i,y_i) \sim P_D, i \in [n]\}}$\footnote{Throughout the paper $[m]$ denotes the set $\{1, \ldots, m\}$ for $m \in \mdN$.} containing $n$ samples from a distribution~$P_D$, 
and the true risk $L$ as
\begin{align*}
    \widehat{L}_\mcO(h) = \frac{1}{n} \sum_{i=1}^n \ell(h(z_i), y_i), \quad
    L(h) &= \Ep{(z,y) \sim P_D}{\ell(h(z),y)}.
\end{align*}
PAC-Bayesian analysis \citep{shawe1997pac, alquier2024user} provides an upper bound $C$ of the form $d(\Ep{h \sim \rho}{L(h)}, \mdE_{h \sim \rho}[\widehat{L}_\mcO(h)])  \leq C( \rho, \rho_0, \delta)$ that holds with probability at least $1-\delta$ for any error tolerance $\delta \in (0,1]$, any prior probability measure~$\rho_0$, and any posterior probability measure~$\rho$ chosen from the set of all measures $\mathcal{P}$ feasible on the hypothesis space $\mathcal{H}$ for a given convex distance metric $d(\cdot,\cdot)$. The choice of the posterior measure $\rho$ can depend on the data~$\mcO$, while the prior measure $\rho_0$ cannot, hence their names. 
Unlike in Bayesian inference, the posterior and the prior are not necessarily related to each other via a likelihood function. The prior serves as a reference with respect to which a generalization statement is made.  
Of the various well-known PAC-Bayes bounds (such as \citet{germain2009pac,mcallester1999pac, seeger2002pac}), we build on \citet{catoni2007pac}'s bound in this work.
\begin{theorem}[\citet{catoni2007pac}]\label{thm:catoni_generic}
For distributions $\rho, \rho_0$ measurable on hypothesis space $\mcH$, $\nu > 0$, and error tolerance $\delta \in (0,1)$, the following inequality holds with probability greater than $1 - \delta$:
\begin{align*}
    \Ep{h\sim \rho}{L(h)} &\leq \Ep{\rho}{\widehat L(h)} + \frac{\nu C^2}{8n} + \frac{\KL{\rho}{\rho_0} - \log\delta}{n},
\end{align*}
where $C$ is a upper bound on $L$ and ${\KL{\rho}{\rho_0} = \Ep{h \sim \rho}{\log \rho(h) - \log \rho_0(h)}}$ is the Kullback-Leibler divergence between $\rho$ and $\rho_0$.
\end{theorem}

Since a PAC-Bayes bound holds for any posterior, a parametric family of posteriors can be chosen and its parameters can be fit to data by minimizing the right-hand side of the bound. This approach, called \emph{PAC-Bayesian Learning}, has demonstrated remarkable success in image classification \citep{dziugaite2017computing, wu2024recursive} and regression \citep{reeb2018learning}.  Only preliminary results such as \citet{tasdighi2023pac} exist that use this approach in deep reinforcement learning. The tightness of the bound is determined by the choice of $d$, as well as additional assumptions and algebraic manipulations on the second term on the right-hand side of the inequality. Remarkably, this term does not depend on $\rho$ and hence does not play any role in a PAC-Bayesian learning algorithm. Our solution relies on this simple but often overlooked observation.

\section{Deep exploration with PAC-Bayes}\label{sec:deepexp_pac_ac}

\subsection{A new PAC-Bayesian bound for policy evaluation}
We build an actor-critic algorithm able to perform deep exploration in continuous control setups with delayed rewards. We consider a model-free and off-policy approximate temporal difference training setup. 
The only existing prior work at the time of writing that addresses this setup with a PAC-Bayesian analysis is by \citet{fard2012pac}. 
Their work uses PAC-Bayes to directly bound the value of a policy from a single chain of observations. 
The Markovian dependencies of these observations necessitate building on a Bernstein-like concentration inequality that works only for extremely long episodes, prohibiting its applicability to PAC-Bayesian learning \citep{tasdighi2023pac}. We overcome this limitation by instead developing a PAC-Bayes bound on the Bellman operator error. 
See \cref{appsec:proofs} for proofs of all results.

We assume the existence of a replay buffer $\mathcal{D}$ containing samples obtained from real environment interactions.
Following \citet{fard2012pac}, we assume for convenience that the evaluated policy $\pi$ always induces a stationary state transition kernel, that is, ${P_\pi(s') = \int P(s' | s, \pi(s)) P_\pi(s) ds}$ for every $s \in \mathcal{S}$ in all training-time environment interactions, where $P_\pi$ denotes the state visitation distribution for policy $\pi$. 
This is essentially a hidden assumption made by typical off-policy deep reinforcement learning algorithms that sample uniformly from their replay buffers. 
The difference between the sample Bellman operator and the Bellman operator, $(\widetilde{T}_\pi Q) (s,a,s') - (T_\pi Q) (s,\pi(s))$, tends to accumulate a positive value when $Q$ is concurrently used as a training objective for policy improvement, leading to the well-known  \emph{overestimation bias} \citep{thrun1993issues}. The problem is tackled by approaches such as double $Q$-learning~\citep{hasselt2010double} and min-clipping \citep{fujimoto2018addressing}. 
We define a new random variable as 
\begin{equation*}
    X(s,a) \overset{d}{=} Q(s,a) + (\widetilde{T}_\pi Q)(s,a,s') - (T_\pi Q)(s,a), %
\end{equation*}
for all $(s,a)\in \St \times A$, where~$\overset{d}{=}$ denotes equality in distribution. It is a hypothetical variable that predicts the value of the observed Bellman target caused by the randomness of $s'$ given $(s,a)$.  It is related to the action-value function approximator as 
\begin{align*}
    Q(s,a) &=   Q(s,a) + \Ep{s' \sim P(\cdot | s,a)}{(\widetilde{T}_\pi Q)(s,a,s')}\!-\! (T_\pi Q)(s,a)\\
         &= \mdE_{s' \sim P(\cdot | s,a)}\Big[ Q(s,a) + (\widetilde{T}_\pi Q)(s,a,s') 
         - T_\pi Q(s,a)\Big] \\
         &= \E{X(s,a)}.
\end{align*}
Let $\rho$ denote the distribution of $X$. The one-step TD loss is
\begin{align*}
 L(Q) &= ||T_\pi Q - Q ||_{P_\pi}^2= ||T_\pi Q - \E{X} ||_{P_\pi}^2\\
 &=\mdE_{s \sim P_\pi}\Big[ \Big (\mdE_{s'\sim P(\cdot | s,\pi(s))}\big[ r(s,\pi(s)) \\
&\quad+\gamma \E{X(s',\pi(s'))}\big] - \E{X(s,\pi(s))}\Big  )^2\Big],
\end{align*}
with the corresponding stochastic variant
\begin{align*}
\widetilde{L}(\rho) =  \E { \left((\widetilde{T}_\pi X)\left(s,\pi(s),s'\right) - X(s,\pi(s))\right)^2 }
\end{align*}
where the expectation is with respect to $s \sim P_\pi$, $s' \sim P(\cdot|s,\pi(s))$, and $X \sim \rho$. A realization of $X$ shares the same domain and range as $Q$, hence, it can be passed through the sample Bellman operator $\widetilde{T}$. 

This optimization problem no longer fits a deterministic map $Q$ but instead infers a distribution $\rho$ describing the action-value function perturbed by the Bellman operator error. The loss represents the true risk of a Gibbs predictor $X\sim \rho$.  The corresponding empirical risk is
\begin{align*}
\widehat{L}_\mathcal{D}(\rho)=  \frac1{|D|} \sum_{\mathclap{(s,a,s') \in \mathcal{D}}} \Ep{X \sim\rho} {\left( (\widetilde{T}_\pi X)(s,a,s') - X(s,\pi(s)) \right)^2}.   
\end{align*}
Our learning problem selects $\rho$ from a feasible set $\mathcal{P}$ to achieve the tightest PAC-Bayes bound on $\widetilde{L}(\rho)$  with respect to data $\mathcal{D}$ and prior $\rho_0 \in \mathcal{P}$ over the distribution of $X$. This bound provides generalization guarantees on $X$'s prediction of noisy Bellman operator outputs. Our primary objective is to minimize the value estimation error $\lVert Q_\pi - Q\rVert_{P_\pi}=\lVert Q_\pi - \E{X}\rVert_{P_\pi}$, upper bounded via Jensen's inequality by $\mdE_{X \sim \rho} \lVert Q_\pi - X \rVert_{P_\pi}$.  \citet{fellows2021bayesian} addressed the mismatch between target and inferred quantities in probabilistic reinforcement learning. We propose an alternative solution by leveraging the contraction property of Bellman operators and the following result.
\begin{lemma} \label{prop:mse2value}
For any $\rho$ defined on $X$, we have 
\begin{align*}
\widetilde{L}(\rho)= \mdE_{X \sim\rho}||T_\pi X- X||_{P_\pi}^2+ \gamma^2 \mdE_{X \sim\rho} \varp{s \sim P_\pi}{ X(s,\pi(s))}.
\end{align*}
\end{lemma}
Applying the contraction property of Bellman operators to ${||T_\pi X- X||_{P_\pi}^2}$ (see Lemma \ref{lem:contraction}) and adding all quantities into \cref{thm:catoni_generic}'s generic  bound, we arrive at our main result.
\begin{theorem} \label{prop:main_result}
For any posterior and prior measures ${\rho, \rho_0 \in \mathcal{P}}$, error tolerance $\delta \in (0,1]$, and policy $\pi$, the following inequality holds with probability at least $1-\delta$:
\begin{align}
&\mdE_{X \sim \rho} ||Q_\pi - X||_{P_\pi}^2 \leq  \Big ( \widehat{L}_\mathcal{D}(\rho)  + \frac{\nu \bar \lambda B^2}{8n} + \frac{\KL{\rho}{\rho_0} - \log\delta}{\nu}\nonumber\\
&\qquad - \gamma^2 \mdE_{X \sim \rho} \varp{s \sim P_\pi}{X(s,\pi(s))}  \Big )\big/(1 - \gamma)^2, \label{eq:main-result}
\end{align}
where $\bar \lambda \deq  \frac{1 + \max_{t\in [n]}\lambda_t}{1 - \max_{t \in [n]}\lambda_t}$
with ${\lambda_t \in [0,1)}$ as the operator norm of the transition kernel of the time-inhomogeneous Bellman error Markov chain at time $t$, $B=R^2/(1 - \gamma)^2$ and $\nu > 0$ an arbitrary constant.
\end{theorem}
Note that  $\max_{t\in[n]}{\lambda_t} \to 1$ implies a chain with absolute past correlation. We use $\nu = 1$  throughout.
Since $\bar \lambda$ and $B$ are independent of the posterior, we exclude them when deriving the training objective, with further discussion in  \cref{appsec:proofs}.
Our main focus in this work is to derive a valid bound from learning-theoretic results to discover algorithmic principles. 
As such, we leave the task of tightening this bound for additional performance guarantees to future work. 
As a step in that direction, we derive an alternative bound following \citet{fard2012pac}, which is tighter yet also requires theoretical assumptions rarely satisfied in our use cases. 
Fitting $\nu$ by following \citet[][Theorem 2.4 and Corollary 2.5]{alquier2024user} would provide further tightness guarantees.

\subsection{A practical algorithm for this bound}
It is not straightforward how to use the theoretical bound in \eqref{eq:main-result} to train the critic networks of an actor-critic design, as the bound does not prescribe how $\rho$ should be distributed or how the KL divergence between distributions on non-linear function spaces can be computed in practice. %
Below, we provide a recipe to develop  an effective deep reinforcement learning algorithm from our new PAC-Bayes bound.
A pseudocode of the resulting algorithm can be found in \cref{appsec:pseudocode}.

\subsubsection{Non-parametric posterior design}
Given the highly nonlinear and complex structure of the task, we require $\rho$ to be a flexible and powerful posterior distribution.
We learn an ensemble  
of $K$ critics ${\{X_k: \St \times \Ac \rightarrow \mdR\;|\;k \in [K]\}}$ parameterized by weights $\theta_k$
and model ${\rho(A) := \frac1K\sum_{k=1}^K\delta_{X_k}(A)}$ for ${A \in \sigma(\mcS)}$ and Dirac measure $\delta$.
We extend this with a set of target copies, $\bar X_k$, updated by Polyak averaging ${\bar \theta_k \gets (1 - \tau)\bar \theta_k + \tau\theta_k}$ for $\tau \in [0,1]$.
For a single observation, the empirical risk term is ${\frac1K\sum_{k=1}^K\left(r + \gamma \bar X_k(s',\pi(s')) - X_k(s,a)\right)^2}$
where we use~$\bar{X}_k$ to ensure robust training \citep{lillicrap2015continuous}.
To further decorrelate the ensemble members, we use bootstrap ensembles\footnote{Despite the overlapping naming, there is no relation between a \emph{bootstrap} ensemble, where random subsets of a data set are used to train an ensemble, and bootstrapping in reinforcement learning, i.e., the technique of iteratively updating value estimates based on prior estimates.}, which have been shown to provide significant practical benefits in prior work on deep exploration \citep{osband2016deep,osband2018randomized}.
We pass each minibatch through a random binary mask that filters some portion of the data for each ensemble element.

\subsubsection{Computing the bound}
\textbf{From weight space to function space.}
Deep RL uses action-value functions for tasks such as exploration, bootstrapping, and actor training. 
Design choices such as how much to explore and how conservatively to evaluate the Bellman target tend to be simpler to specify in function space compared to weight space.
The data processing inequality \citep[Theorem 6.2]{polyanskiy2014lecture} additionally provides a tighter generalization bound in function space. 
While function-space priors for regularization have garnered attention in the Bayesian deep learning literature \citep{tran2022all,rudner2023function,qiu2023},
we are not aware of any earlier application of function-space priors in the context of PAC-Bayesian learning, especially in reinforcement learning.
The KL divergence between two measures $\rho, \rho_0$ defined on $\mathcal{H}$ is (via \citet{rudner2022tractable}):
\begin{align*}
    &\KL{\rho}{\rho_0} = \\
    &\quad \sup_{\mathcal{D} \in \mathcal{B}} \Ep{s \in \mathcal{D}}{\Ep{X \sim \rho}{\log f_\rho(X|s,\pi(s)) 
    - \log f_{\rho_0}(X|s,\pi(s))}\big.},
\end{align*}
where $f_\rho$ and $f_{\rho_0}$ are the probability density functions of the two measures evaluated at $X$ for a given state-action pair and $\mathcal{B}$ is the space of all possible data sets $\mathcal{D}$. 
This $\sup$ calculation is intractable in practice.
We simplify \citet{rudner2022tractable}'s approximation to
\begin{equation}\label{eq:kl_approx}
    \sum_{j=1}^J\sum_{i=1}^{|B_j|} \log f_\rho(s_i,\pi(s_i)) - \log f_{\rho_0}(s_i,\pi(s_i))
\end{equation}
as the sum of a set of scalars upper bounds their supremum and $B_j$ are batches sampled from the available data.

Since $\rho$ is implicitly defined as an ensemble mixture, it lacks an analytical density function. We approximate it in the KL divergence by a normal density, i.e., ${f_\rho(s,\pi(s)) \deq \Norm(\mu_\pi(s),\sigma_\pi^2(s))}$, where we define ${\mu_\pi(s) = \frac1K \sum_{k=1}^K X_k(s,\pi(s))}$ and ${\sigma_\pi^2 = \frac{1}{K-1}\sum_{k=1}^K(X_k(s,\pi(s)) - \mu_\pi(s))^2}$.

Building on PAC-Bayes research outside reinforcement learning \citep{ambroladze2006tighter,dziugaite2018data}, we implement a data-informed prior $\rho_0$ for a tighter PAC-Bayes-based policy evaluation bound. 
Because the latest action-value information resides in the targets $\bar{X}_k$, we use them for a maximum likelihood normal density estimate. 
The target networks evaluate the next state-action values, and we set ${f_{\rho_0}(s',\pi(s')) = \mathcal{N}(\bar{\mu}_\pi(s'), \sigma_0^2)}$ where ${\bar{\mu}_\pi(s')=\frac{1}{K} \sum_{k=1}^K \bar{X}_k(s',\pi(s'))}$. Projecting backwards one time step yields ${f_{\rho_0}(s,\pi(s)) := \mathcal{N}(r+\gamma \bar{\mu}_\pi(s'), \gamma^2 \bar{\sigma}_0^2)}$ where $r$ is the reward for state $s$. 
This approach maintains the PAC-Bayes bound's validity because $\rho_0$ uses critic targets from the previous gradient step, while the bound applies to the current minibatch. 

We approximate the final KL divergence between $\rho$ and $\rho_0$ as
\begin{align*}
    &\KL{\rho(s,\pi(s))}{\rho_0(s,\pi(s)) \big.}\\
    &\approx \frac{1}{2K} \sum_{k=1}^{K} \Big(\frac{(r+\gamma \bar{\mu}_\pi(s')-X_k)^2}{\gamma^2 \sigma_0^2} - \log \sigma_\pi^2(s) \!\Big ) \!+\! \text{const}
\end{align*}
to ensure that the computation remains closer to the true posterior.
See \cref{appsec:proofs} for the full details of the derivation.

\subsubsection{The critic training objective}
Our complete training objective is given as 
\begin{align}
\begin{split}
  &\mcL(\{X_k:k\in[K]\})  
  := \\
  &\underbrace{\frac{1}{n K} \sum_{i=1}^{n} \sum_{k=1}^K b_{ik} \Big(r_i + \gamma \bar{X}_k(s'_i, \pi(s'_i))-X_k(s_i,\pi(s_i)) \Big)^2}_{\text{Diversity}} \\
  &+\underbrace{\frac{1}{ n K}  \sum_{i=1}^{n}  \sum_{k=1}^{K}  \frac{ b_{ik} \big (r_i+\gamma \bar{\mu}_\pi(s'_i)-X_k(s_i,\pi(s_i)) \big)^2}{2\gamma^2 \sigma_0^2} }_{\text{Coherence}}\\
  &- \underbrace{\frac{2\gamma^2+1}{2n} \sum_{i=1}^{n} \log \sigma_\pi^2(s_i)}_{\text{Propagation}},
  \end{split}
    \label{eq:critic_loss} 
\end{align}
for binary bootstrap masks ${b_{ik} \sim \mathrm{Bernoulli}(1-\kappa)}$, ${\forall [n] \times [K]}$ with a bootstrapping rate~${\kappa \in (0,1)}$.\footnote{We calculate $\bar{\mu}_\pi$ and $\sigma_\pi^2$ after applying the bootstrap masks.} 
Here, we further reduce the influence of the expected variance term in \eqref{eq:main-result} by taking its logarithm.\footnote{The bound remains valid as $\log(x) < x, \forall x \in \mdR^+$.}
The loss is computed by one forward pass through each ensemble member (run in parallel) and their targets per data point with additional constant-time operations. Hence, its computation time matches that of any other actor-critic method. 
The loss contains three interpretable terms. The first induces \emph{diversity} into the ensemble by training each member on individual targets. 
This way, the model quantifies the uncertainty of its predictions, as observations assigned to similar action values can only be those with which all ensemble elements are familiar. 
The second term ensures \emph{coherence} among ensemble elements as it trains them with the same target. 
The third term is a regularizer that repels ensemble elements away from each other in the absence of counter-evidence. This way, uncertainty \emph{propagation} is maintained, as the model cannot find a solution to collapse the ensemble elements into a single solution. 
As the variance grows, the PBAC loss approaches to BootDQN-P \citep{osband2018randomized}. As it shrinks, it converges to the deterministic policy gradient.
See \cref{appsec:lossablation} for a detailed ablation of each term's importance.
Our PBAC critic training loss effectively unifies the advantages of variational Bayesian inference~\citep{blei2017variational} and Bayesian deep ensembles~\citep{lakshminarayanan2017simple} into a single theoretically justified framework. 
Unlike variational Bayes, it operates with density-free posteriors, and unlike Bayesian deep ensembles, it enables a principled application use of prior regularization through the KL-divergence.

\subsection{Actor training and behavior policy} 

We use an actor network with a shared trunk $g(s)$ and individual heads $\pi_k(s):=(h_k \circ g) (s)$ for each ensemble element. 
Each head $h_k$ predicts the mean and variance parameters of a Gaussian distribution subsequently mapped to the $(-1,1)$ interval via a $\tanh$ transformation. Its objective includes an entropy maximization regularization \citep{haarnoja2018soft}.
Since the optimal decision rule under uncertainty is the Bayes predictor, we train the actor on the average value of each state with respect to the learned posterior 
\begin{equation*}
\argmax_\pi \mdE_{s \sim P_\pi} \mdE_{X \sim \rho}  X(s,\pi(s))
\end{equation*}
and implement its empirical approximation as 
\begin{equation*}
\argmax_{g, h_1, \ldots, h_K}  \frac{1}{n K} \sum_{i=1}^n \sum_{k=1}^K X_k(s, \pi_k(s)).
\end{equation*}
We perform posterior sampling for exploration, relying on its demonstrated theoretical benefits \citep{grande2014sample,osband2013more,strens2000bayesian} and practical success in discrete control \citep{fellows2021bayesian,osband2016deep, osband2018randomized}. That is, our behavior policy is randomly chosen among the available $K$ options and fixed for a number of time steps as practiced in prior work \citep{touati2020randomized}. In critic training, the Bellman targets are computed with the active behavior policy. 
See \cref{appsec:hyper} and \ref{appsec:pseudocode} for our hyperparameters and further algorithmic details.

\section{Experiments}\label{sec:experiments}
\begin{table*}
    \centering
    \caption{\emph{DMC and MuJoCo.} InterQuartile Mean (IQM) scores over ten seeds with [lower, upper] quantiles. 
    The three blocks report the DMC, and dense and delayed MuJoCo environments.
    For each task, we mark the IQM values bold if they do not differ significantly ($p < 0.05$) from the highest IQM.
    }
    \vspace{1.5em}
    \label{tab:result_table}
    \adjustbox{max width=0.98\textwidth}{
    \begin{tabular}{lcccccccc}
    \toprule
    & \multicolumn{4}{c}{(a) Final episode reward $(\uparrow)$} & \multicolumn{4}{c}{(b) Area under learning curve $(\uparrow)$} \\
    \cmidrule(lr){2-5} \cmidrule(lr){6-9}
    \textsc{Environment} & BEN & BootDQN-P & DRND & PBAC (ours) & BEN & BootDQN-P & DRND & PBAC (ours) \\
    \cmidrule(lr){1-5}\cmidrule(lr){6-9}
        ballincup   
         &  $\mathbf{977 \sd{[971,980]}}$ & $\mathbf{973 \sd{[965,978]}}$ & $\mathbf{973 \sd{[970,977]}}$ & $\mathbf{978 \sd{[974,981]}}$
         & $\mathbf{954 \sd{[936,963]}}$ & $801 \sd{[778,861]}$ & $863 \sd{[773,951]}$ & $\mathbf{945 \sd{[837,968]}}$
         \\
        cartpole  
         & $\mathbf{796 \sd{[186,836]}}$ & $\mathbf{601 \sd{[379,770]}}$ & $0.0 \sd{[0.0,0.0]}$ & $\mathbf{785 \sd{[584,819]}}$
         & $\mathbf{423 \sd{[179,513]}}$ & $\mathbf{527 \sd{[412,609]}}$ & $0.0 \sd{[0.0,0.0]}$ & $\mathbf{500 \sd{[410,561]}}$
         \\
        reacher 
         &  $903 \sd{[799,975]}$ & $\mathbf{959 \sd{[911,965]}}$ & $846 \sd{[762,875]}$ & $\mathbf{912 \sd{[861,966]}}$
         &  $\mathbf{708 \sd{[628,734]}}$ & $\mathbf{730 \sd{[698,764]}}$ & $624 \sd{[494,695]}$ & $\mathbf{753 \sd{[657,806]}}$
         \\
    \cmidrule(lr){1-5}\cmidrule(lr){6-9}
        ant 
         & $5579 \sd{[5114,5685]}$ & $298 \sd{[114,430]}$ & $3080 \sd{[1172,4243]}$ & $\mathbf{6376 \sd{[6215,6472]}}$
         & $3103 \sd{[2719,3387]}$ & $511 \sd{[486,547]}$ & $1799 \sd{[596,2613]}$ & $\mathbf{4719 \sd{[4215,4827]}}$ 
         \\
        hopper  
         & $1051 \sd{[567,1443]}$ & $\mathbf{1988 \sd{[1260,2494]}}$ & $\mathbf{1864 \sd{[1103,2914]}}$ & $\mathbf{1410 \sd{[1238,1606]}}$
         &  $839 \sd{[733,986]}$ & $974 \sd{[949,1005]}$ & $952 \sd{[736,1365]}$ & $\mathbf{1565 \sd{[1415,1692]}}$
         \\
        humanoid  
         & $1782 \sd{[1443,2008]}$ & $311 \sd{[243,372]}$ & $\mathbf{5144 \sd{[4743,5879]}}$ & $2237 \sd{[1394,2487]}$
         & $1227 \sd{[1118,1295]}$ & $258 \sd{[240,265]}$ & $\mathbf{2677 \sd{[2416,3026]}}$ & $1352 \sd{[1197,1441]}$ 
         \\
    \cmidrule(lr){1-5}\cmidrule(lr){6-9}
        ant (delayed)  
         & $\mathbf{4580 \sd{[3369,4980]}}$ & $-432 \sd{[-542,-339]}$ & $10.0 \sd{[8.3,12.7]}$ & $\mathbf{5138 \sd{[4770,5457]}}$
         & $2705 \sd{[1804,3126]}$ & $-307 \sd{[-319,-288]}$ & $-0.9 \sd{[-3.8,1.6]}$ & $\mathbf{3732 \sd{[3519,4000]}}$
         \\
        ant (very delayed) 
         & $2051 \sd{[-1,3661]}$ & $-169 \sd{[-291,-136]}$ & $-2.8 \sd{[-2.8,-2.4]}$ & $\mathbf{3905 \sd{[3398,4864]}}$
         &  $875 \sd{[-16,1494]}$ & $-16 \sd{[-76,53]}$ & $-7.6 \sd{[-9.0,-6.0]}$ & $\mathbf{2297 \sd{[1783,2973]}}$
         \\
        hopper (delayed)  
         & $\mathbf{813 \sd{[689,928]}}$ & $\mathbf{1067 \sd{[659,1340]}}$ & $440 \sd{[295,525]}$ & $\mathbf{775 \sd{[722,910]}}$
         & $611 \sd{[554,734]}$ & $275 \sd{[170,394]}$ & $309 \sd{[267,368]}$ & $\mathbf{803 \sd{[646,970]}}$
         \\
        hopper (very delayed)  
         & $\mathbf{190 \sd{[6,643]}}$ & $\mathbf{388 \sd{[-0,1579]}}$ & $82 \sd{[-0,347]}$ & $\mathbf{516 \sd{[419,580]}}$
         & $\mathbf{367 \sd{[72,537]}}$ & $62 \sd{[-0,248]}$ & $55 \sd{[-0,250]}$ & $\mathbf{344 \sd{[270,374]}}$
         \\
        humanoid (delayed)
         & $6.0 \sd{[5.9,6.1]}$ & $-3.3 \sd{[-13.6,4.6]}$ & $5.9 \sd{[5.8,5.9]}$ & $\mathbf{281 \sd{[208,551]}}$
         & $0.6 \sd{[0.2,1.2]}$ & $-0.9 \sd{[-3.5,1.9]}$ & $4.9 \sd{[4.9,5.0]}$ & $\mathbf{117 \sd{[98,427]}}$
         \\
    \bottomrule
    \end{tabular}
    }
\end{table*}

We conduct experiments to test whether model-free deep exploration algorithms can be designed using PAC-Bayesian bounds.
Focusing on environments with nonlinear dynamics and continuous state-action spaces, 
we use established sparse benchmarks from the DMC suite~\citep{tassa2018deepmind},
robotic tasks from Meta-World environments with sparse rewards \citep{yu2020meta}, 
and introduce new delayed reward benchmarks based on various MuJoCo environments~\citep{todorov2012mujoco}. 
We evaluate whether PBAC improves performance in these setups with difficult reward characteristics while remaining competitive on their dense-reward counterparts. 
In \cref{appsec:prior_sparsity}, we discuss prior work on such environments.

We build an ensemble of ten Q-functions and use a replay ratio of five to improve sample efficiency. This approach has been shown to achieve the same level of performance as state-of-the-art methods, effectively addressing sample efficiency, as demonstrated by \citet{nauman2024overestimation}. 
More details on the reward structure of each environment and the remaining design choices are provided in \cref{appsec:exp_details}.
We provide a public implementation at \url{https://github.com/adinlab/PAC-Bayes-ActorCritic}.

\textbf{Baselines.}
We compare our method with other state-of-the-art exploration approaches on three standard MuJoCo setups, their delayed versions, three sparse DMC tasks, and seven Meta-World environments. Following the results of \citet{fu2024furl}, we excluded three environments they observed to be (almost) unlearnable for their vision-based RL approach.  
We choose SAC-DRND~\citep{yang2024exploration} as the best representative of the RND family integrated with SAC and ${\text{BootDQN-P}}$~\citep{osband2018randomized} as a close SOTA Bayesian model-free method to our model. BootDQN-P was originally introduced for discrete action spaces. We discuss our adaptation to continuous action spaces in \cref{appsec:baselines}.
Finally, \text{BEN}~\citep{fellows2024ben} serves as the best representative approach that can learn the Bayes-optimal policies. 
See \cref{appsec:exp_details} for further details.

\textbf{Delayed rewards.}
Our chosen MuJoCo environments rely on a reward signal $r$ composed of three terms, 
\begin{equation*}
    r \deq \underbrace{\frac{dx_t}{dt}}_{\mathclap{\text{forward reward}}} - \underbrace{w_a||a_t||^2}_\text{action cost} + \underbrace{H}_{\mathclap{\text{health reward}}}.
\end{equation*}
The \emph{forward reward} is given by the agent's speed towards its destination, while the \emph{action cost} ensures that the agent solves the task as efficiently as possible. The \emph{health reward} $H$ encourages the agent to maintain its balance.
To model delayed rewards, we explore the following two variations of this default reward:

(i) A \emph{delayed} version in which the health reward $H$ is always set to zero, removing any incentive for the agent to maintain stability. This can cause the agent to prioritize speed (minimizing positional delay), potentially leading to inefficient movement patterns such as falling, rolling, or chaotic behavior. Without incentives for stability, the agent may fail to learn proper locomotion, resulting in suboptimal performance, more difficult training, and reduced overall effectiveness.

(ii) A \emph{very delayed} version in which the forward reward proportional to velocity is granted only if the agent reaches a distance threshold $c$, i.e., the reward is given as
$r = \nicefrac{dx_t}{dt}\mathbf{1}_{x_t > c} - w_a||a_t||^2$.
The delay degree, which determines the task difficulty, increases proportionally to~$c$. This delay applies in addition to the zero health reward from above. Although this incentivizes reaching the target, the lack of intermediate progress rewards may lead the agent to struggle with exploration, resulting in slower learning and difficulty discovering effective policies as the agent rarely receives informative signals.

In these settings, the agent does not receive any positive reward and incurs a penalty for using its actuators for failed explorations. Hence, they are well-suited for testing how directed the exploration scheme of a learning algorithm is. We use them to assess whether an exploration strategy can consistently promote effective exploration when the agent encounters a lack of informative feedback during learning. 
We set the hyperparameters $w_a$ and $H$ to the original MuJoCo default values. See \cref{appsec:mujocodelay} for these parameters and further discussion. 
For completeness, we additionally report the dense counterparts.

We summarize our results in \cref{tab:result_table}, reporting the interquartile mean (IQM) \citep{agarwal2021deep} together with the corresponding interquartile range over ten seeds. Reported are the reward on the final episode, and the area under the learning curve (AULC) for the whole training period, which quantifies the convergence speed, i.e., how fast an agent learns its task. See \cref{fig:delayed_curves} for the reward curves of a subset of environments. We provide the remaining curves in \cref{fig:all_curves}.
PBAC is competitive in most environmental settings and excels especially in the learning speed (subtable (b)) in delayed environments, while still outperforming the baselines in two of the three dense reward environments. 
Interquartile means that are not statistically significant at a level $p < 0.05$ compared to the highest IQM in a paired t-test are marked in bold (see \cref{appsec:eval_methodology} for details). 

\begin{figure*}
    \centering
    \includegraphics[width=0.30\textwidth]{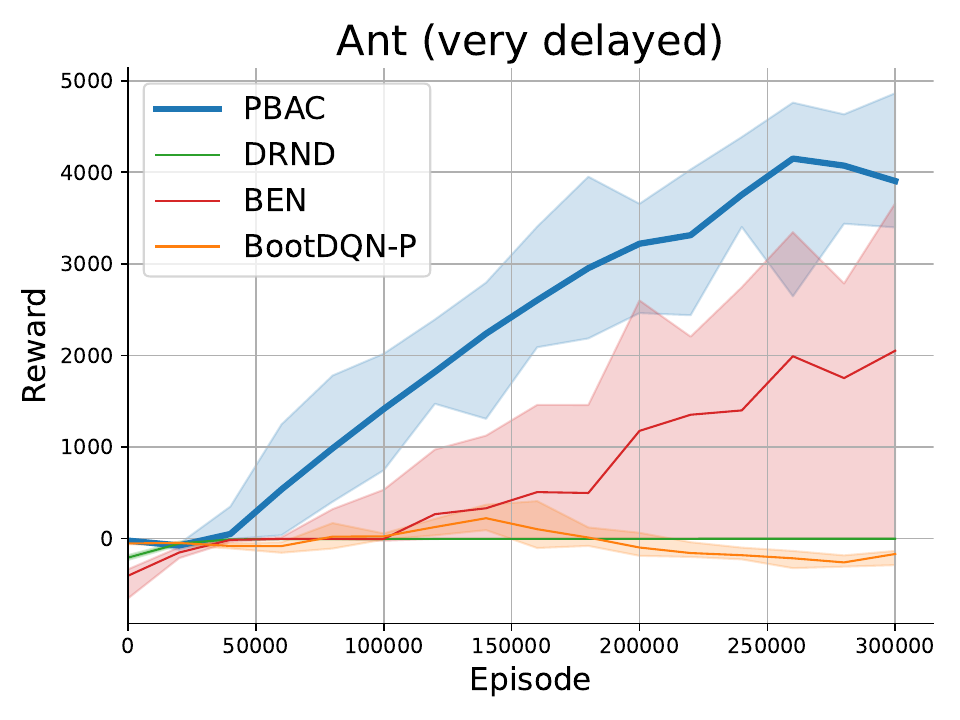}
    \includegraphics[width=0.30\textwidth]{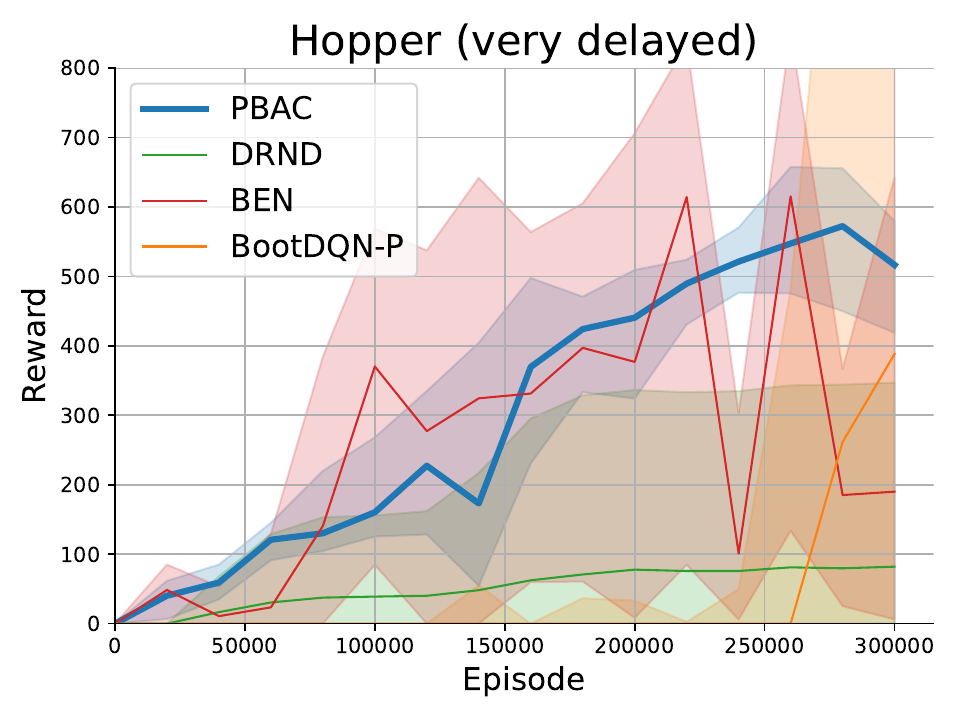}
    \includegraphics[width=0.30\textwidth]{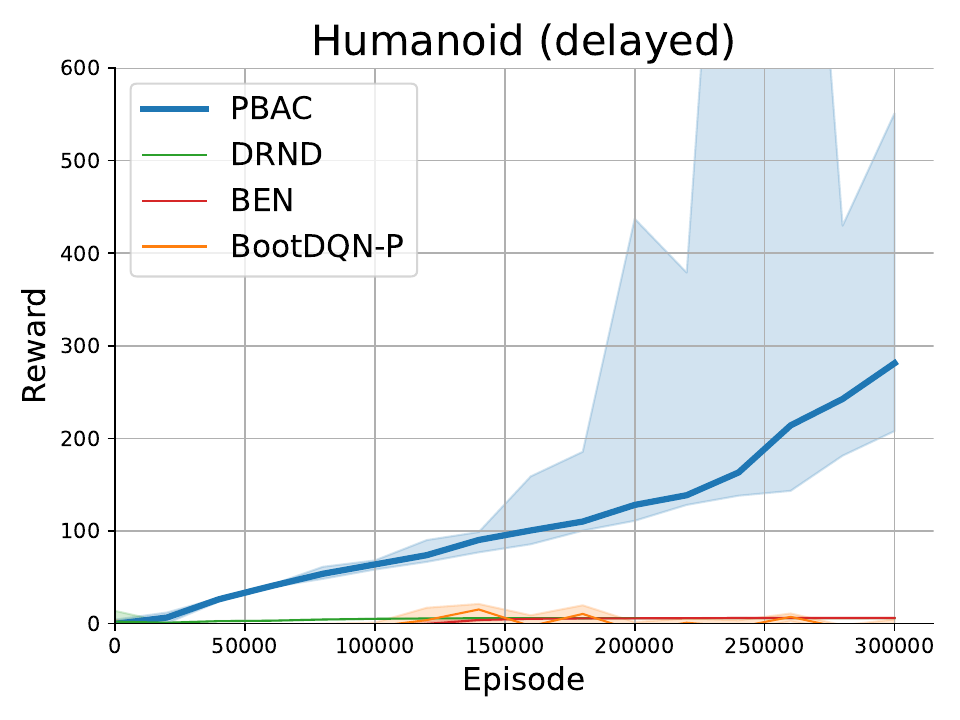}
    \caption{Reward curves for three delayed reward MuJoCo environments corresponding to the results in \cref{tab:result_table}. Visualized are the interquartile mean together with the interquartile range over ten seeds for each method. The remaining reward curves are shown in \cref{fig:all_curves}.} \label{fig:delayed_curves}
    \vspace{1.0em}
\end{figure*}

\textbf{Sparse rewards.}
DMC and Meta-World provide a set of sparse reward scenarios in which the agent only receives reward upon successful completion of its task.
We evaluated all models on Meta-World's random-goal MT benchmark, where the hidden goal position changes with each trajectory. 
We report the final average success rate and the area under the success rate curve. Methods that do not perform significantly worse ($p < 0.05$) than the best performing average are marked in bold.
Our results in \cref{tab:success_result_table} show that PBAC outperforms the other baselines in \emph{average area under success curve} and performs similarly to DRND in the \emph{final episode success rate}, demonstrating its ability to generalize effectively across different goal positions. High average success rates are relevant, e.g., in online learning settings where an agent must learn a task not only successfully but also as quickly as possible.
Note especially the improvements in performance of PBAC compared to the Bayesian baseline, i.e., BEN.
More detailed information is provided in \cref{app:env} as well as on \emph{button} and \emph{door}, which none of the baselines solve.
Three DMC environments are reported in \cref{tab:result_table}.
See \cref{fig:metaworld_curves} for the full reward curves.

\begin{table}
    \centering
    \caption{\emph{Sparse Meta-World.} Average success rate percentage over five seeds on sparse Meta-World environments with hidden goals. The highest mean and those that do not differ significantly from it ($p < 0.05$) are bold. 
    }
    \label{tab:success_result_table}
    \vspace{2.0em}
    \adjustbox{max width=0.95\columnwidth}{
    \begin{tabular}{lccccc}
    \toprule
    &\multicolumn{5}{c}{(a) average of the final episode success rate ($\uparrow$)}\\
    \cmidrule(lr){2-6}
    \textsc{Environment} & BEN & BootDQN-P & DRND  & PBAC (ours)\\
    \midrule
        window close 
 &  $ 0.20 \sd{\pm 0.40}$ & $ 0.00 \sd{\pm 0.00}$ &$ \mathbf{0.98\sd{\pm 0.04}}$  &  $   0.88 \sd{\pm 0.16}$ 
        \\ 
        window open
& $0.00 \sd{\pm 0.00}$ & $0.04 \sd{\pm 0.08}$ & $\mathbf{0.80 \sd{\pm 0.40}}$ &  $\mathbf{0.56 \sd{\pm 0.35}}$
\\
drawer close
&  $0.96 \sd{\pm 0.08}$ & $0.62 \sd{\pm 0.47}$ & $\mathbf{1.00} \sd{\pm 0.00}$ &  $0.92 \sd{\pm 0.12}$
\\
drawer open
& $0.00 \sd{\pm 0.00}$ & $0.00 \sd{\pm 0.00}$ & $\mathbf{1.00} \sd{\pm 0.00}$ &  $0.78 \sd{\pm 0.26}$
\\
reach
& $0.60 \sd{\pm 0.49}$ & $0.00 \sd{\pm 0.00}$ & $0.80 \sd{\pm 0.40}$ & $\mathbf{1.00} \sd{\pm 0.00}$
\\
button press topdown
& $0.00 \sd{\pm 0.00}$  & $0.00 \sd{\pm 0.00}$ &  $0.20 \sd{\pm 0.40}$ & $\mathbf{0.96}\sd{\pm 0.05}$
\\
door open
& $0.00 \sd{\pm 0.00}$  & $0.00 \sd{\pm 0.00}$ & $0.60 \sd{\pm 0.49}$ & $\mathbf{0.98}\sd{\pm 0.04}$
        \\ 
    \midrule
    &\multicolumn{5}{c}{(b) average of the area under success rate curve ($\uparrow)$} \\
    \cmidrule(lr){2-6}
window close 
    & $ 0.18\sd{\pm 0.38}$  & $ 0.06\sd{\pm 0.23}$ &  $\mathbf{0.92}\sd{\pm 0.26}$ &  $\mathbf{0.89}\sd{\pm 0.26}$
    \\ 
    window open  
     &  $0.00\sd{\pm 0.00}$ & $0.15\sd{\pm 0.36}$ & $\mathbf{0.76}\sd{\pm 0.43}$ &  $\mathbf{0.73}\sd{\pm 0.35}$
    \\ 
    drawer close  
     & $0.94\sd{\pm 0.18}$ & $0.61\sd{\pm 0.48}$ & $\mathbf{0.97}\sd{\pm 0.15}$ & $\mathbf{0.88\sd{\pm 0.25}}$  
    \\ 
   drawer open  
     & $0.00\sd{\pm 0.00}$ & $0.00\sd{\pm 0.00}$ & $\mathbf{0.81}\sd{\pm 0.39}$ & $\mathbf{0.80\sd{\pm 0.34}}$  
    \\ 
   reach  
    & $0.56\sd{\pm 0.49}$  & $0.00\sd{\pm 0.00}$ & $\mathbf{0.78}\sd{\pm 0.42}$ & $\mathbf{0.98}\sd{\pm 0.14}$ 
    \\ 
     button press topdown  
     &  $0.00\sd{\pm 0.00}$ & $0.00\sd{\pm 0.00}$ & $0.19\sd{\pm 0.39}$ &  $\mathbf{0.68}\sd{\pm 0.38}$ 
    \\ 
    door open  
     &  $0.00\sd{\pm 0.00}$ & $0.00\sd{\pm 0.00}$  & $0.50\sd{\pm 0.49}$ &  $\mathbf{0.69}\sd{\pm 0.43}$ 
    \\
    \bottomrule
    \end{tabular}
    }
\end{table}

\textbf{Exploration patterns.}
\cref{fig:teaser_exploration} shows how PBAC explores the first two dimensions of the state space (position and angle) in the cartpole environment throughout its training process. We describe this visualization in greater detail and include the remaining methods as well as a second environment (delayed ant) in \cref{appsec:ssvis}.

\textbf{Ablation.}
We provide results on the influence of the three main hyperparameters: bootstrap rate $\kappa$, posterior sampling rate, and prior variance $\sigma_0$ in \cref{appsec:hyper}.
An ablation on the importance of the individual terms in the objective \cref{eq:critic_loss} is available in \cref{appsec:lossablation}.

\section{Conclusion}
We introduced a method to perform deep exploration in delayed reward environmental settings for the first time with a principled PAC-Bayesian approach. Comparing it to state-of-the-art baselines, we demonstrated its superior performance on a wide range of continuous control benchmarks with various delayed and sparse reward patterns. 

PBAC's performance is sensitive to hyperparameters, especially in dense reward environments. This is expected, as it is designed to perform deep exploration in sparse-reward settings.
While performing competitively in two dense environments, \emph{ant} and \emph{hopper}, it struggles in the dense \emph{humanoid}, as do several of the baselines. This is due to the significant amount of reward an agent receives by only remaining alive ($H=5$) compared to $H=1$ in the other two dense environments. It is not surprising that in this situation, random exploration is sufficient and more robust. However, as soon as the reward becomes delayed, PBAC remains the only method to learn the task.

Our current work does not provide any convergence guarantees on PBAC's behavior. This theoretical problem requires, and deserves, dedicated investigation, which we consciously keep outside our focus. 
Global convergence cannot be guaranteed for standard $Q$-learning with continuous state spaces, let alone continuous action spaces. 
The convergence results for SAC \citet{haarnoja2018soft} are too restrictive, as they assume error-free Bellman backups. If such assumptions are deemed acceptable, the convergence of PBAC follows analogously to SAC's derivation.
As PBAC builds on a sum of two Bellman backups and a regularizer, it will inherit similar properties to the guarantees of stochastic iterative $Q$-learning variants.

\begin{ack}
    This work was funded by the Novo Nordisk Foundation (NNF21OC0070621) and
the Carlsberg Foundation (CF21-0250).
\end{ack}

\bibliography{refs-ecai}

\begin{thebibliography}{83}
\providecommand{\natexlab}[1]{#1}
\providecommand{\url}[1]{\texttt{#1}}
\expandafter\ifx\csname urlstyle\endcsname\relax
  \providecommand{\doi}[1]{doi: #1}\else
  \providecommand{\doi}{doi: \begingroup \urlstyle{rm}\Url}\fi

\bibitem[Agarwal et~al.(2021)Agarwal, Schwarzer, Castro, Courville, and
  Bellemare]{agarwal2021deep}
R.~Agarwal, M.~Schwarzer, P.~Castro, A.~Courville, and M.~Bellemare.
\newblock Deep reinforcement learning at the edge of the statistical precipice.
\newblock \emph{Advances in Neural Information Processing Systems}, 2021.

\bibitem[Alquier(2024)]{alquier2024user}
P.~Alquier.
\newblock User-friendly introduction to {PAC-B}ayes bounds.
\newblock \emph{Foundations and Trends in Machine Learning}, 2024.

\bibitem[Ambroladze et~al.(2006)Ambroladze, Parrado-Hernandez, and
  Shawe-Taylor]{ambroladze2006tighter}
A.~Ambroladze, E.~Parrado-Hernandez, and J.~Shawe-Taylor.
\newblock Tighter {PAC-B}ayes bounds.
\newblock In \emph{Advances in Neural Information Processing Systems}, 2006.

\bibitem[Amin et~al.(2021)Amin, Gomrokchi, Aboutalebi, Satija, and
  Precup]{amin2021locally}
S.~Amin, M.~Gomrokchi, H.~Aboutalebi, H.~Satija, and D.~Precup.
\newblock Locally persistent exploration in continuous control tasks with
  sparse rewards.
\newblock In \emph{Proceedings of the International Conference on Machine
  Learning}, 2021.

\bibitem[Bai et~al.(2021)Bai, Wang, Han, Hao, Garg, Liu, and
  Wang]{bai2021principled}
C.~Bai, L.~Wang, L.~Han, J.~Hao, A.~Garg, P.~Liu, and Z.~Wang.
\newblock Principled exploration via optimistic bootstrapping and backward
  induction.
\newblock In \emph{International Conference on Machine Learning}, pages
  577--587. PMLR, 2021.

\bibitem[Ball et~al.(2023)Ball, Smith, Kostrikov, and
  Levine]{ball2023efficient}
P.~Ball, L.~Smith, I.~Kostrikov, and S.~Levine.
\newblock Efficient online reinforcement learning with offline data.
\newblock In \emph{Proceedings of the International Conference on Machine
  Learning}, 2023.

\bibitem[Bellemare et~al.(2016)Bellemare, Srinivasan, Ostrovski, Schaul,
  Saxton, and Munos]{bellemare2016unifying}
M.~Bellemare, S.~Srinivasan, G.~Ostrovski, T.~Schaul, D.~Saxton, and R.~Munos.
\newblock Unifying count-based exploration and intrinsic motivation.
\newblock In \emph{Advances in Neural Information Processing Systems}, 2016.

\bibitem[Bertsekas and Tsitsiklis(1996)]{bertsekas1996neuro}
D.~Bertsekas and J.~Tsitsiklis.
\newblock \emph{{Neuro}-dynamic programming}.
\newblock Athena Scientific, 1996.

\bibitem[Blei et~al.(2017)Blei, Kucukelbir, and McAuliffe]{blei2017variational}
D.~Blei, A.~Kucukelbir, and J.~McAuliffe.
\newblock Variational inference: A review for statisticians.
\newblock \emph{Journal of the American Statistical Association}, 2017.

\bibitem[Boucheron et~al.(2013)Boucheron, Lugosi, and
  Massart]{boucheron2013concenetration}
S.~Boucheron, G.~Lugosi, and P.~Massart.
\newblock \emph{Concentration Inequalities: A Nonasymptotic Theory of
  Independence}.
\newblock Oxford University Press, 02 2013.

\bibitem[Brockman(2016)]{brockman2016openai}
G.~Brockman.
\newblock Openai gym.
\newblock \emph{arXiv preprint arXiv:1606.01540}, 2016.

\bibitem[Burda et~al.(2019)Burda, Edwards, Storkey, and
  Klimov]{burda2018exploration}
Y.~Burda, H.~Edwards, A.~Storkey, and O.~Klimov.
\newblock Exploration by random network distillation.
\newblock In \emph{International Conference on Learning Representations}, 2019.

\bibitem[Catoni(2007)]{catoni2007pac}
O.~Catoni.
\newblock {PAC-B}ayesian supervised classification: The thermodynamics of
  statistical learning.
\newblock \emph{Lecture Notes-Monograph Series}, 2007.

\bibitem[Chentanez et~al.(2004)Chentanez, Barto, and Singh]{singh2004intrinsic}
N.~Chentanez, A.~Barto, and S.~Singh.
\newblock Intrinsically motivated reinforcement learning.
\newblock In \emph{Advances in Neural Information Processing Systems}, 2004.

\bibitem[Chernoff(1952)]{chernoff52ameasure}
H.~Chernoff.
\newblock {A Measure of Asymptotic Efficiency for Tests of a Hypothesis Based
  on the sum of Observations}.
\newblock \emph{The Annals of Mathematical Statistics}, 1952.

\bibitem[Cox and Hinkley(1974)]{CoxHink74}
D.~Cox and D.~Hinkley.
\newblock \emph{Theoretical Statistics}.
\newblock Chapman \& Hall, 1974.

\bibitem[Curi et~al.(2020)Curi, Berkenkamp, and Krause]{curi2020efficient}
S.~Curi, F.~Berkenkamp, and A.~Krause.
\newblock Efficient model-based reinforcement learning through optimistic
  policy search and planning.
\newblock In \emph{Advances in Neural Information Processing Systems}, 2020.

\bibitem[Dabney et~al.(2018)Dabney, Ostrovski, Silver, and
  Munos]{dabney2018implicit}
W.~Dabney, G.~Ostrovski, D.~Silver, and R.~Munos.
\newblock Implicit quantile networks for distributional reinforcement learning.
\newblock In \emph{International conference on machine learning}, pages
  1096--1105. PMLR, 2018.

\bibitem[Donsker and Varadhan(1976)]{donsker1976asymptotic}
M.~D. Donsker and S.~S. Varadhan.
\newblock Asymptotic evaluation of certain markov process expectations for
  large time—iii.
\newblock \emph{Communications on pure and applied Mathematics}, 29\penalty0
  (4):\penalty0 389--461, 1976.

\bibitem[Dziugaite and Roy(2017)]{dziugaite2017computing}
G.~Dziugaite and D.~Roy.
\newblock Computing non-vacuous generalization bounds for deep (stochastic)
  neural networks with many more parameters than training data.
\newblock In \emph{Proceedings of the Conference on Uncertainty in Artificial
  Intelligence}, 2017.

\bibitem[Dziugaite and Roy(2018)]{dziugaite2018data}
G.~Dziugaite and D.~Roy.
\newblock Data-dependent {PAC-B}ayes priors via differential privacy.
\newblock In \emph{Advances in Neural Information Processing Systems}, 2018.

\bibitem[Fan et~al.(2021)Fan, Jiang, and Sun]{fan2021hoeffding}
J.~Fan, B.~Jiang, and Q.~Sun.
\newblock Hoeffding's inequality for general markov chains and its applications
  to statistical learning.
\newblock \emph{Journal of Machine Learning Research}, 2021.

\bibitem[Fard et~al.(2012)Fard, Pineau, and Szepesv{\'a}ri]{fard2012pac}
M.~Fard, J.~Pineau, and C.~Szepesv{\'a}ri.
\newblock {PAC-Bayesian} policy evaluation for reinforcement learning.
\newblock In \emph{Proceedings on the International Conference on Artificial
  Intelligence and Statistics}, 2012.

\bibitem[Fellows et~al.(2021)Fellows, Hartikainen, and
  Whiteson]{fellows2021bayesian}
M.~Fellows, K.~Hartikainen, and S.~Whiteson.
\newblock {Bayesian B}ellman operators.
\newblock In \emph{Advances in Neural Information Processing Systems}, 2021.

\bibitem[Fellows et~al.(2024)Fellows, Kaplowitz, Schroeder~de Witt, and
  Whiteson]{fellows2024ben}
M.~Fellows, B.~Kaplowitz, C.~Schroeder~de Witt, and S.~Whiteson.
\newblock {B}ayesian {E}xploration networks.
\newblock In \emph{Proceedings of the International Conference on Machine
  Learning}, 2024.

\bibitem[Fortunato et~al.(2018)Fortunato, Azar, Piot, Menick, Osband, Graves,
  Mnih, Munos, Hassabis, Pietquin, Blundell, and Legg]{fortunato2018noisy}
M.~Fortunato, M.~Azar, B.~Piot, J.~Menick, I.~Osband, A.~Graves, V.~Mnih,
  R.~Munos, D.~Hassabis, O.~Pietquin, C.~Blundell, and S.~Legg.
\newblock Noisy networks for exploration.
\newblock In \emph{International Conference on Learning Representations}, 2018.

\bibitem[Fu et~al.(2024)Fu, Zhang, Wu, Xu, and Boulet]{fu2024furl}
Y.~Fu, H.~Zhang, D.~Wu, W.~Xu, and B.~Boulet.
\newblock Fu{RL}: Visual-language models as fuzzy rewards for reinforcement
  learning.
\newblock In \emph{Forty-first International Conference on Machine Learning},
  2024.

\bibitem[Fujimoto et~al.(2018)Fujimoto, Hoof, and
  Meger]{fujimoto2018addressing}
S.~Fujimoto, H.~Hoof, and D.~Meger.
\newblock Addressing function approximation error in actor-critic methods.
\newblock In \emph{Proceedings of the International Conference on Machine
  Learning}, 2018.

\bibitem[Germain et~al.(2009)Germain, Lacasse, Laviolette, and
  Marchand]{germain2009pac}
P.~Germain, A.~Lacasse, F.~Laviolette, and M.~Marchand.
\newblock {PAC-Bayesian} learning of linear classifiers.
\newblock In \emph{Proceedings of the International Conference on Machine
  Learning}, 2009.

\bibitem[Grande et~al.(2014)Grande, Walsh, and How]{grande2014sample}
R.~Grande, T.~Walsh, and J.~How.
\newblock Sample efficient reinforcement learning with gaussian processes.
\newblock In \emph{Proceedings of the International Conference on Machine
  Learning}, 2014.

\bibitem[Guo et~al.(2022)Guo, Thakoor, P{\^\i}slar, Avila~Pires, Altch{\'e},
  Tallec, Saade, Calandriello, Grill, Tang, et~al.]{guo2022byol}
Z.~Guo, S.~Thakoor, M.~P{\^\i}slar, B.~Avila~Pires, F.~Altch{\'e}, C.~Tallec,
  A.~Saade, D.~Calandriello, J.~Grill, Y.~Tang, et~al.
\newblock {B}yol-explore: Exploration by bootstrapped prediction.
\newblock \emph{Advances in Neural Information Processing Systems}, 2022.

\bibitem[Haarnoja et~al.(2018{\natexlab{a}})Haarnoja, Zhou, Abbeel, and
  Levine]{haarnoja2018soft}
T.~Haarnoja, A.~Zhou, P.~Abbeel, and S.~Levine.
\newblock {Soft Actor-Critic: O}ff-policy maximum entropy deep reinforcement
  learning with a stochastic actor.
\newblock In \emph{Proceedings of the International Conference on Machine
  Learning}, 2018{\natexlab{a}}.

\bibitem[Haarnoja et~al.(2018{\natexlab{b}})Haarnoja, Zhou, Hartikainen,
  Tucker, Ha, Tan, Kumar, Zhu, Gupta, Abbeel, et~al.]{haarnoja2018soft-app}
T.~Haarnoja, A.~Zhou, K.~Hartikainen, G.~Tucker, S.~Ha, J.~Tan, V.~Kumar,
  H.~Zhu, A.~Gupta, P.~Abbeel, et~al.
\newblock Soft actor-critic algorithms and applications.
\newblock \emph{arXiv preprint arXiv:1812.05905}, 2018{\natexlab{b}}.

\bibitem[Hafner et~al.(2025)Hafner, Pasukonis, Ba, and Lillicrap]{Hafner_2025}
D.~Hafner, J.~Pasukonis, J.~Ba, and T.~Lillicrap.
\newblock Mastering diverse control tasks through world models.
\newblock \emph{Nature}, 2025.

\bibitem[Houthooft et~al.(2016)Houthooft, Chen, Duan, Schulman, De~Turck, and
  Abbeel]{houthooft2016vime}
R.~Houthooft, X.~Chen, Y.~Duan, J.~Schulman, F.~De~Turck, and P.~Abbeel.
\newblock Vime: Variational information maximizing exploration.
\newblock \emph{Advances in Neural Information Processing Systems}, 2016.

\bibitem[Jarrett et~al.(2023)Jarrett, Tallec, Altch{\'e}, Mesnard, Munos, and
  Valko]{jarrett2022curiosity}
D.~Jarrett, C.~Tallec, F.~Altch{\'e}, T.~Mesnard, R.~Munos, and M.~Valko.
\newblock Curiosity in hindsight: Intrinsic exploration in stochastic
  environments.
\newblock \emph{Proceedings of the International Conference on Machine
  Learning}, 2023.

\bibitem[Kingma and Ba(2015)]{kingma2014adam}
D.~Kingma and J.~Ba.
\newblock {{Adam}: A method for stochastic optimization}.
\newblock In \emph{International Conference on Learning Representations}, 2015.

\bibitem[Kirk(2004)]{kirk2004optimal}
D.~Kirk.
\newblock \emph{Optimal Control Theory: An Introduction}.
\newblock Courier Corporation, 2004.

\bibitem[Lakshminarayanan et~al.(2017)Lakshminarayanan, Pritzel, and
  Blundell]{lakshminarayanan2017simple}
B.~Lakshminarayanan, A.~Pritzel, and C.~Blundell.
\newblock Simple and scalable predictive uncertainty estimation using deep
  ensembles.
\newblock \emph{Advances in Neural Information Processing Systems}, 2017.

\bibitem[Lillicrap et~al.(2016)Lillicrap, Hunt, Pritzel, Heess, Erez, Tassa,
  Silver, and Wierstra]{lillicrap2015continuous}
T.~Lillicrap, J.~Hunt, A.~Pritzel, N.~Heess, T.~Erez, Y.~Tassa, D.~Silver, and
  D.~Wierstra.
\newblock Continuous control with deep reinforcement learning.
\newblock In \emph{International Conference on Learning Representations}, 2016.

\bibitem[Lipton et~al.(2018)Lipton, Li, Gao, Li, Faisal, and
  Deng]{lipton2018bbq}
Z.~Lipton, X.~Li, J.~Gao, L.~Li, A.~Faisal, and L.~Deng.
\newblock {BBQ-N}etworks: Efficient exploration in deep reinforcement learning
  for task-oriented dialogue systems.
\newblock In \emph{Proceedings of the AAAI Conference on Artificial
  Intelligence}, 2018.

\bibitem[Liu et~al.(2024)Liu, Wang, Zheng, Hao, Bai, Ye, Wang, Piao, and
  Sun]{liu2024ovd}
J.~Liu, Z.~Wang, Y.~Zheng, J.~Hao, C.~Bai, J.~Ye, Z.~Wang, H.~Piao, and Y.~Sun.
\newblock Ovd-explorer: Optimism should not be the sole pursuit of exploration
  in noisy environments.
\newblock In \emph{Proceedings of the AAAI Conference on Artificial
  Intelligence}, volume~38, pages 13954--13962, 2024.

\bibitem[Luis et~al.(2023{\natexlab{a}})Luis, Bottero, Vinogradska, Berkenkamp,
  and Peters]{luis2023model}
C.~Luis, A.~Bottero, J.~Vinogradska, F.~Berkenkamp, and J.~Peters.
\newblock Model-based uncertainty in value functions.
\newblock In \emph{Proceedings on the International Conference on Artificial
  Intelligence and Statistics}, 2023{\natexlab{a}}.

\bibitem[Luis et~al.(2023{\natexlab{b}})Luis, Bottero, Vinogradska, Berkenkamp,
  and Peters]{luis2023value}
C.~Luis, A.~G. Bottero, J.~Vinogradska, F.~Berkenkamp, and J.~Peters.
\newblock Value-distributional model-based reinforcement learning.
\newblock \emph{arXiv preprint arXiv:2308.06590}, 2023{\natexlab{b}}.

\bibitem[Mazoure et~al.(2019)Mazoure, Doan, Durand, Hjelm, and
  Pineau]{mazoure19leveraging}
B.~Mazoure, T.~Doan, A.~Durand, R.~Hjelm, and J.~Pineau.
\newblock Leveraging exploration in off-policy algorithms via normalizing
  flows.
\newblock In \emph{Conference on Robot Learning (CoRL)}, 2019.

\bibitem[McAllester(1999)]{mcallester1999pac}
D.~McAllester.
\newblock {PAC-B}ayesian model averaging.
\newblock In \emph{Proceedings of the Conference on Learning Theory}, 1999.

\bibitem[Modirshanechi et~al.(2023)Modirshanechi, Kondrakiewicz, Gerstner, and
  Haesler]{modirshanechi2023curiosity}
A.~Modirshanechi, K.~Kondrakiewicz, W.~Gerstner, and S.~Haesler.
\newblock Curiosity-driven exploration: foundations in neuroscience and
  computational modeling.
\newblock \emph{Trends in Neurosciences}, 2023.

\bibitem[Nauman et~al.(2024)Nauman, Bortkiewicz, Mi{\l}o\'{s}, Trzcinski,
  Ostaszewski, and Cygan]{nauman2024overestimation}
M.~Nauman, M.~Bortkiewicz, P.~Mi{\l}o\'{s}, T.~Trzcinski, M.~Ostaszewski, and
  M.~Cygan.
\newblock Overestimation, overfitting, and plasticity in actor-critic: the
  bitter lesson of reinforcement learning.
\newblock In \emph{Proceedings of the International Conference on Machine
  Learning}, 2024.

\bibitem[Nikulin et~al.(2023)Nikulin, Kurenkov, Tarasov, and
  Kolesnikov]{nikulin2023anti}
A.~Nikulin, V.~Kurenkov, D.~Tarasov, and S.~Kolesnikov.
\newblock Anti-exploration by random network distillation.
\newblock In \emph{Proceedings of the International Conference on Machine
  Learning}, 2023.

\bibitem[Osband et~al.(2013)Osband, Russo, and Van~Roy]{osband2013more}
I.~Osband, D.~Russo, and B.~Van~Roy.
\newblock ({M}ore) efficient reinforcement learning via posterior sampling.
\newblock \emph{Advances in Neural Information Processing Systems}, 2013.

\bibitem[Osband et~al.(2016{\natexlab{a}})Osband, Blundell, Pritzel, and
  Van~Roy]{osband2016deep}
I.~Osband, C.~Blundell, A.~Pritzel, and B.~Van~Roy.
\newblock Deep exploration via bootstrapped {DQN}.
\newblock \emph{Advances in Neural Information Processing Systems},
  2016{\natexlab{a}}.

\bibitem[Osband et~al.(2016{\natexlab{b}})Osband, Van~Roy, and
  Wen]{osband2016generalization}
I.~Osband, B.~Van~Roy, and Z.~Wen.
\newblock Generalization and exploration via randomized value functions.
\newblock In \emph{Proceedings of the International Conference on Machine
  Learning}, 2016{\natexlab{b}}.

\bibitem[Osband et~al.(2018)Osband, Aslanides, and
  Cassirer]{osband2018randomized}
I.~Osband, J.~Aslanides, and A.~Cassirer.
\newblock Randomized prior functions for deep reinforcement learning.
\newblock In \emph{Advances in Neural Information Processing Systems}, 2018.

\bibitem[Osband et~al.(2019)Osband, van Roy, Russo, and Wen]{osband2019deep}
I.~Osband, B.~van Roy, D.~Russo, and Z.~Wen.
\newblock Deep exploration via randomized value functions.
\newblock \emph{Journal of Machine Learning Research}, 2019.

\bibitem[O’Donoghue et~al.(2018)O’Donoghue, Osband, Munos, and
  Mnih]{odonoghue2018uncertainty}
B.~O’Donoghue, I.~Osband, R.~Munos, and V.~Mnih.
\newblock The uncertainty bellman equation and exploration.
\newblock In \emph{Proceedings of the International Conference on Machine
  Learning}, 2018.

\bibitem[Paszke et~al.(2019)Paszke, Gross, Massa, Lerer, Bradbury, Chanan,
  Killeen, Lin, Gimelshein, Antiga, Desmaison, Kopf, Yang, DeVito, Raison,
  Tejani, Chilamkurthy, Steiner, Fang, Bai, and Chintala]{paszke2019pytorch}
A.~Paszke, S.~Gross, F.~Massa, A.~Lerer, J.~Bradbury, G.~Chanan, T.~Killeen,
  Z.~Lin, N.~Gimelshein, L.~Antiga, A.~Desmaison, A.~Kopf, E.~Yang, Z.~DeVito,
  M.~Raison, A.~Tejani, S.~Chilamkurthy, B.~Steiner, L.~Fang, J.~Bai, and
  S.~Chintala.
\newblock {PyTorch: An Imperative Style, High-Performance Deep Learning
  Library}.
\newblock \emph{Advances in Neural Information Processing Systems}, 2019.

\bibitem[P{\'e}rez-Ortiz et~al.(2021)P{\'e}rez-Ortiz, Rivasplata, Shawe-Taylor,
  and Szepesv{\'a}ri]{perez2021tighter}
M.~P{\'e}rez-Ortiz, O.~Rivasplata, J.~Shawe-Taylor, and C.~Szepesv{\'a}ri.
\newblock Tighter risk certificates for neural networks.
\newblock \emph{Journal of Machine Learning Research}, 2021.

\bibitem[Plappert et~al.(2018)Plappert, Houthooft, Dhariwal, Sidor, Chen, Chen,
  Asfour, Abbeel, and Andrychowicz]{plappert2018parameter}
M.~Plappert, R.~Houthooft, P.~Dhariwal, S.~Sidor, R.~Chen, X.~Chen, A.~Asfour,
  P.~Abbeel, and M.~Andrychowicz.
\newblock Parameter space noise for exploration.
\newblock In \emph{International Conference on Learning Representations}, 2018.

\bibitem[Polyanskiy and Wu(2014)]{polyanskiy2014lecture}
Polyanskiy and Wu.
\newblock Lecture notes on information theory.
\newblock \emph{UIUC}, 2014.

\bibitem[Puterman(2014)]{puterman2014markov}
M.~Puterman.
\newblock \emph{Markov decision processes: discrete stochastic dynamic
  programming}.
\newblock John Wiley \& Sons, 2014.

\bibitem[Qiu et~al.(2023)Qiu, Rudner, Kapoor, and Wilson]{qiu2023}
S.~Qiu, T.~G.~J. Rudner, S.~Kapoor, and A.~G. Wilson.
\newblock Should we learn most likely functions or parameters?
\newblock In \emph{Advances in Neural Information Processing Systems}, 2023.

\bibitem[Reeb et~al.(2018)Reeb, Doerr, Gerwinn, and Rakitsch]{reeb2018learning}
D.~Reeb, A.~Doerr, S.~Gerwinn, and B.~Rakitsch.
\newblock Learning {Gaussian} processes by minimizing {PAC-Bayesian}
  generalization bounds.
\newblock \emph{Advances in Neural Information Processing Systems}, 2018.

\bibitem[Rudner et~al.(2022)Rudner, Chen, Teh, and Gal]{rudner2022tractable}
T.~Rudner, Z.~Chen, Y.~Teh, and Y.~Gal.
\newblock Tractable function-space variational inference in {B}ayesian neural
  networks.
\newblock \emph{Advances in Neural Information Processing Systems}, 2022.

\bibitem[Rudner et~al.(2023)Rudner, Kapoor, Qiu, and
  Wilson]{rudner2023function}
T.~G. Rudner, S.~Kapoor, S.~Qiu, and A.~G. Wilson.
\newblock Function-space regularization in neural networks: A probabilistic
  perspective.
\newblock In \emph{International Conference on Machine Learning}, 2023.

\bibitem[Scannell et~al.(2025)Scannell, Nakhaeinezhadfard, Kujanp{\"a}{\"a},
  Zhao, Luck, Solin, and Pajarinen]{scannell2025discrete}
A.~Scannell, M.~Nakhaeinezhadfard, K.~Kujanp{\"a}{\"a}, Y.~Zhao, K.~S. Luck,
  A.~Solin, and J.~Pajarinen.
\newblock Discrete codebook world models for continuous control.
\newblock In \emph{The Thirteenth International Conference on Learning
  Representations}, 2025.

\bibitem[Schmidhuber(2010)]{schmidhuber2010formal}
J.~Schmidhuber.
\newblock Formal theory of creativity, fun, and intrinsic motivation
  (1990–2010).
\newblock \emph{IEEE Transactions on Autonomous Mental Development}, 2010.

\bibitem[Seeger(2002)]{seeger2002pac}
M.~Seeger.
\newblock {PAC-Bayesian} generalisation error bounds for {G}aussian process
  classification.
\newblock \emph{Journal of Machine Learning Research}, 2002.

\bibitem[Shang et~al.(2016)Shang, Sohn, Almeida, and Lee]{shang16understanding}
W.~Shang, K.~Sohn, D.~Almeida, and H.~Lee.
\newblock Understanding and improving convolutional neural networks via
  concatenated rectified linear units.
\newblock In \emph{Proceedings of the International Conference on Machine
  Learning}, 2016.

\bibitem[Shawe-Taylor and Williamson(1997)]{shawe1997pac}
J.~Shawe-Taylor and R.~Williamson.
\newblock A {PAC} analysis of {Bayesian} estimator.
\newblock In \emph{Proceedings of the Conference on Learning Theory}, 1997.

\bibitem[Stengel(1994)]{stengel1994optimal}
R.~Stengel.
\newblock \emph{Optimal Control and Estimation}.
\newblock Courier Corporation, 1994.

\bibitem[Strens(2000)]{strens2000bayesian}
M.~Strens.
\newblock A {B}ayesian framework for reinforcement learning.
\newblock In \emph{Proceedings of the International Conference on Machine
  Learning}, 2000.

\bibitem[Tasdighi et~al.(2024)Tasdighi, Akg{\"u}l, Haussmann, Brink, and
  Kandemir]{tasdighi2023pac}
B.~Tasdighi, A.~Akg{\"u}l, M.~Haussmann, K.~K. Brink, and M.~Kandemir.
\newblock {PAC-B}ayesian soft actor-critic learning.
\newblock In \emph{Advances in Approximate Bayesian Inference (AABI)}, 2024.

\bibitem[Tassa et~al.(2018)Tassa, Doron, Muldal, Erez, Li, Casas, Budden,
  Abdolmaleki, Merel, Lefrancq, et~al.]{tassa2018deepmind}
Y.~Tassa, Y.~Doron, A.~Muldal, T.~Erez, Y.~Li, D.~d.~L. Casas, D.~Budden,
  A.~Abdolmaleki, J.~Merel, A.~Lefrancq, et~al.
\newblock Deepmind control suite.
\newblock \emph{arXiv preprint arXiv:1801.00690}, 2018.

\bibitem[Thrun and Schwartz(1993)]{thrun1993issues}
S.~Thrun and A.~Schwartz.
\newblock Issues in using function approximation for reinforcement learning.
\newblock In \emph{Proceedings of the Fourth Connectionist Models Summer
  School}, 1993.

\bibitem[Todorov et~al.(2012)Todorov, Erez, and Tassa]{todorov2012mujoco}
E.~Todorov, T.~Erez, and Y.~Tassa.
\newblock Mujoco: A physics engine for model-based control.
\newblock In \emph{IEEE/RSJ International Conference on Intelligent Robots and
  Systems (IROS)}, 2012.

\bibitem[Touati et~al.(2020)Touati, Satija, Romoff, Pineau, and
  Vincent]{touati2020randomized}
A.~Touati, H.~Satija, J.~Romoff, J.~Pineau, and P.~Vincent.
\newblock Randomized value functions via multiplicative normalizing flows.
\newblock In \emph{Uncertainty in Artificial Intelligence (UAI)}, 2020.

\bibitem[Towers et~al.(2024)Towers, Kwiatkowski, Terry, Balis, De~Cola, Deleu,
  Goul{\~a}o, Kallinteris, Krimmel, KG, et~al.]{towers2024gymnasium}
M.~Towers, A.~Kwiatkowski, J.~Terry, J.~Balis, G.~De~Cola, T.~Deleu,
  M.~Goul{\~a}o, A.~Kallinteris, M.~Krimmel, A.~KG, et~al.
\newblock Gymnasium: A standard interface for reinforcement learning
  environments.
\newblock \emph{arXiv preprint arXiv:2407.17032}, 2024.

\bibitem[Tran et~al.(2022)Tran, Rossi, Milios, and Filippone]{tran2022all}
B.-H. Tran, S.~Rossi, D.~Milios, and M.~Filippone.
\newblock All you need is a good functional prior for bayesian deep learning.
\newblock \emph{Journal of Machine Learning Research}, 2022.

\bibitem[Van~Hasselt(2010)]{hasselt2010double}
H.~Van~Hasselt.
\newblock Double {Q}-learning.
\newblock \emph{Advances in Neural Information Processing Systems}, 2010.

\bibitem[Wu et~al.(2024)Wu, Zhang, Chérief-Abdellatif, and
  Seldin]{wu2024recursive}
Y.-S. Wu, Y.~Zhang, B.-E. Chérief-Abdellatif, and Y.~Seldin.
\newblock Recursive {PAC-Bayes: A} frequentist approach to sequential prior
  updates with no information loss.
\newblock In \emph{Advances in Neural Information Processing Systems}, 2024.

\bibitem[Yang et~al.(2024)Yang, Tao, Lyu, and Li]{yang2024exploration}
K.~Yang, J.~Tao, J.~Lyu, and X.~Li.
\newblock Exploration and anti-exploration with distributional random network
  distillation.
\newblock In \emph{Proceedings of the International Conference on Machine
  Learning}, 2024.

\bibitem[Yu et~al.(2020)Yu, Quillen, He, Julian, Hausman, Finn, and
  Levine]{yu2020meta}
T.~Yu, D.~Quillen, Z.~He, R.~Julian, K.~Hausman, C.~Finn, and S.~Levine.
\newblock {Meta-world}: A benchmark and evaluation for multi-task and meta
  reinforcement learning.
\newblock In \emph{Conference on Robot Learning (CoRL)}, 2020.

\bibitem[Zhang et~al.(2021)Zhang, Rashidinejad, Jiao, Tian, Gonzalez, and
  Russell]{zhang2021made}
T.~Zhang, P.~Rashidinejad, J.~Jiao, Y.~Tian, J.~Gonzalez, and S.~Russell.
\newblock Made: Exploration via maximizing deviation from explored regions.
\newblock \emph{Advances in Neural Information Processing Systems}, 2021.

\end{thebibliography}

\appendix
\onecolumn

\begin{center}
    \Huge 
    \textsc{Appendix}
\end{center}

In this appendix, we include proofs and derivations for all theoretical statements made in the main text in Section~\ref{appsec:proofs}.
Section~\ref{appsec:fard} contains a short discussion on differences between our bound and \citet{fard2012pac}'s result. 
We summarize all experimental details and provide a pseudocode algorithm for our model in Sections~\ref{appsec:exp_details} and \ref{appsec:pseudocode}.
Section~\ref{appsec:results} closes the appendix with a series of further results and ablations.

\section{Proofs and derivations}\label{appsec:proofs}

For practical purposes, we assume in the proofs that all probability measures except $\rho$  have densities. The proofs can be straightforwardly extended by lifting this assumption.

To prove our main theorem, we require the following three lemmas.

\paragraph{\cref{prop:mse2value}.}
For any $\rho$ defined on $X$, the following identity holds
\begin{align*}
\widetilde{L}(\rho)= \mdE_{X \sim\rho}||T_\pi X- X||_{P_\pi}^2  + \gamma^2 \mdE_{X \sim\rho}[ \varp{s \sim P_\pi}{ X(s,\pi(s))}].
\end{align*}

\begin{proof}[Proof of \cref{prop:mse2value}]
For a fixed $X$  we have
\begin{align*}
\widetilde{L}(X)&=\Ep{s \sim P_\pi} {\Ep{s' \sim P(\cdot|s,\pi(s))}  {\left((\widetilde{T}_\pi X)(s,\pi(s),s') - X(s,\pi(s))\right)^2 \Big| s }} \\
  &~=\mdE_{s \sim P_\pi} \Big[\mdE_{s' \sim P(\cdot|s,\pi(s))}  \big[(\widetilde{T}_\pi X)(s,\pi(s),s')^2  + X(s,\pi(s))^2  - 2 (\widetilde{T}_\pi X)(s,\pi(s),s') X(s,\pi(s)) | s \big] \Big] \\
  &~=\mdE_{s \sim P_\pi} \Big[\Ep{s' \sim P(\cdot|s,\pi(s))}  {(\widetilde{T}_\pi X)(s,\pi(s),s')^2| s } \\
  &\qquad \qquad \qquad \qquad + X(s,\pi(s))^2  - 2 \Ep{s' \sim P(\cdot|s,\pi(s))}  {(\widetilde{T}_\pi X)(s,\pi(s),s') |s} X(s,\pi(s))  \Big] \\
  &~=\mdE_{s \sim P_\pi} \Big[\Ep{s' \sim P(\cdot|s,\pi(s))}  {(\widetilde{T}_\pi X)(s,\pi(s),s')| s }^2 \\
  & \qquad \qquad \qquad \qquad+ \varp{s' \sim P(\cdot|s,\pi(s))}{(\widetilde{T}_\pi X)(s,\pi(s),s')|s } + X(s,\pi(s))^2  \\
  &\qquad \qquad \qquad \qquad - 2 \Ep{s' \sim P(\cdot|s,\pi(s))}  {(\widetilde{T}_\pi X)(s,\pi(s),s') |s} X(s,\pi(s)) \Big] \\
  &~=\mdE_{s \sim P_\pi} \Big[\left(\Ep{s' \sim P(\cdot|s,\pi(s))}  { (\widetilde{T}_\pi X)(s,\pi(s),s')| s } - X(s,\pi(s))\right)^2  + \varp{s' \sim P(\cdot|s,\pi(s))}{(\widetilde{T}_\pi X)(s,\pi(s),s')|s} \Big] \\
  &~=\Ep{s \sim P_\pi} {\Big( (T_\pi X)(s,\pi(s)) - X(s,\pi(s))\Big)^2  + \varp{s' \sim P(\cdot|s,\pi(s))}{(\widetilde{T}_\pi X)(s,\pi(s),s')  |s} }\\
  &~=||T_\pi X- X||_{P_\pi}^2   + \Ep{s \sim P_\pi} {\varp{s' \sim P(\cdot|s,\pi(s))}{(\widetilde{T}_\pi X)(s,\pi(s),s')  |s} } \\
  &~=||T_\pi X- X||_{P_\pi}^2   + \Ep{s \sim P_\pi} {\varp{s' \sim P(\cdot|s,\pi(s))}{r(s,\pi(s)) + \gamma X(s',\pi(s'))  |s} } \\
  &~=||T_\pi X- X||_{P_\pi}^2 + \gamma^2 \Ep{s \sim P_\pi} {\varp{s' \sim P(\cdot|s,\pi(s))}{ X(s',\pi(s'))  |s} }.
\end{align*}

Since the Markov process is stationary we get
\begin{align*}
    \Ep{s \sim P_\pi}{ \Ep{s' \sim P(\cdot|s, \pi(s))}{X(s', \pi(s'))|s}} = \Ep{s \sim P_\pi}{X(s, \pi(s))}
\end{align*}
as well as
\begin{align*}
    \Ep{s \sim P_\pi}{ \Ep{s' \sim P(\cdot|s, \pi(s))}{X(s', \pi(s'))^2|s}} = \Ep{s \sim P_\pi}{X(s, \pi(s))^2}.
\end{align*}
Hence, the variance term can be expressed as
\begin{align*}
    \Ep{s \sim P_\pi}{\varp{s' \sim P(\cdot|s, \pi(s))} {X(s', \pi(s'))|s}} &= 
\Ep{s \sim P_\pi}{X(s, \pi(s))^2} - \Ep{s \sim P_\pi}{X(s, \pi(s))}^2\\
&=\varp{s \sim P_\pi}{X(s,\pi(s))}.
\end{align*}
Taking the expectation over $X$ with respect to $\rho$ concludes the proof.
\end{proof}

\begin{lemma} \label{lem:contraction}
For any $Q_1, Q_2: \St \times \Ac \rightarrow \mdR$ and stationary state visitation distribution $P_\pi$, the following inequality holds:
\begin{align*}
    ||T_\pi Q_1 - T_\pi Q_2||_{P_\pi} \leq \gamma  ||Q_1 - Q_2||_{P_\pi}.
\end{align*}
That is, the Bellman operator $T_\pi$ is a $\gamma$-contraction with respect to the $P_\pi$-weighted  $L_{2}$-norm $\lVert\cdot\rVert_{P^{\pi}}$
\end{lemma}

\begin{proof}[Proof of \cref{lem:contraction}]
\begin{align*}
\lVert T_{\pi} & Q_{1} - T_{\pi} Q_{2} \rVert_{P_{\pi}}^{2}  = \Ep{s\sim P_{\pi}}{\Big ((T_{\pi}Q_{1})(s,\pi(s))-(T_{\pi}Q_{2})(s,\pi(s))\Big )^{2}} \\ 
&= \gamma^2 \mdE_{s \sim P_{\pi}}\Big [\Big (\Ep{s' \sim P(\cdot|s,\pi(s))} {(T_{\pi}Q_{1})(s',\pi(s'))-(T_{\pi}Q_{2})(s',\pi(s')) | s}\Big )^{2} \Big ] \\
&\leq \gamma^2 \mdE_{s \sim P_{\pi}}\Big [ \mdE_{s' \sim P(\cdot|s,\pi(s))} \Big [\Big ((T_{\pi}Q_{1})(s',\pi(s'))-(T_{\pi}Q_{2})(s',\pi(s'))  \Big )^{2} \Big | s \Big  ]\Big ] && \textit{\color{white!40!black}(Jensen)} \\
&=\gamma^2 \mdE_{s' \sim P_{\pi}} \Big [\Big ((T_{\pi}Q_{1})(s',\pi(s'))-(T_{\pi}Q_{2})(s',\pi(s'))  \Big )^{2} \Big ]. && \textit{\color{white!40!black}(Stationarity)}
\end{align*}
Taking the square root of both sides yields the result.
\end{proof}
\begin{lemma} \label{Lem:bellman2value}
For any $\rho$ defined on $X$, the following inequality holds:
\begin{equation*}
\mdE_{X \sim \rho} \lVert X- Q_{\pi}\rVert_{P_{\pi}} \leq  \frac{\mdE_{X \sim \rho} \lVert T_{\pi} X- X \rVert_{P_{\pi}}}{1-\gamma}.
\end{equation*}
\end{lemma}

\begin{proof}[Proof of \cref{Lem:bellman2value}]
For any fixed $X$, we have
\begin{align*}
\lVert X-Q_{\pi}\rVert_{P_{\pi}} &= \lVert X-T_{\pi}X+T_{\pi}X-Q_{\pi}\rVert_{P_{\pi}} \\
&= \lVert X-T_{\pi}X + T_{\pi} X - T_{\pi} Q_{\pi} \rVert_{P_{\pi}} \\
&\leq \lVert X-T_{\pi} X \rVert_{P_{\pi}} + \lVert T_{\pi} X - T_{\pi} Q_{\pi} \rVert_{P_{\pi}} && \textit{\color{white!40!black}(Triangle~ineq.)} \\
&\leq \lVert X-T_{\pi} X \rVert_{P_{\pi}} + \gamma \lVert X- Q_{\pi} \rVert_{P_{\pi}}. && \textit{\color{white!40!black}(\cref{lem:contraction})}
\end{align*}
Rearranging the terms and integrating over all $X$'s weighted by $\rho$ yields the result.
\end{proof}

\paragraph{\cref{prop:main_result}.}
\emph{For any posterior and prior measures ${\rho, \rho_0 \in \mathcal{P}}$, error tolerance $\delta \in (0,1]$, and deterministic policy $\pi$, the following inequality holds simultaneously with probability at least $1-\delta$:}
\begin{align*}
\begin{split}
\mdE_{X \sim \rho} ||Q_\pi - X||_{P_\pi}^2 &\leq  \Big ( \widehat{L}_\mathcal{D}(\rho) + \frac{\nu \bar \lambda B^2}{8n} + \frac{\KL{\rho}{\rho_0} - \log\delta}{\nu}
 - \gamma^2 \mdE_{X \sim \rho} \varp{s \sim P_\pi}{X(s,\pi(s))}  \Big )\big/(1 - \gamma)^2
\end{split}
\end{align*}
\emph{where $\bar \lambda \deq \frac{1 + \max_{t\in [n]}\lambda_t}{1 - \max_{t \in [n]}\lambda_t}$, 
with $\lambda_t \in [0,1)$ as the operator norm of the transition kernel of the time-inhomogeneous Bellman error Markov chain at time $t$, $B=R^2/(1 - \gamma)^2$ and $\nu > 0$ an arbitrary constant.}

To prove this theorem, we rely on the following result by \citet{fan2021hoeffding}, which we quote below.\footnote{Theorem 5 in \citet{fan2021hoeffding}.}

\paragraph{Theorem \citep{fan2021hoeffding}.}
\emph{Consider a time-inhomogeneous $\pi$-stationary Markov chain $\{X_i\}_{i\geq 1}$. If its transition probability kernel $P_i$ admits absolute spectral gap $1 - \lambda_i$, $\forall i \geq 1$, then, for any bounded function $f_i:\mcX \to [a_i,b_i]$, the sum $\sum_i f_i(X_i)$ is sub-Gaussian with variance proxy 
\begin{equation*}
    \sigma^2 \leq \frac{1 + \max_i \lambda_i}{1 - \max_i \lambda_i}\sum_i\frac{(b_i - a_i)^2}{4}.
\end{equation*}
}

We partially follow \citet{alquier2024user}'s proof on the Catoni bound, adapting it to our specific setup.\footnote{Section 2.1 in \citet{alquier2024user}.} 
\begin{proof}[Proof of \cref{prop:main_result}]
    Consider the Bellman error as a time-inhomogeneous Markov chain, ${\{X_t \deq (\tilde T_{\pi_t}(X(s_t,\pi_t(s_t),s_{t+1})) - X(s_t,a_t))\}_{t\geq 1}}$, whose state transition kernel has an absolute spectral gap $1 - \lambda_t$ for each $t \geq 1$.\footnote{Note that this transition kernel differs from the environmental kernel $P$.} 
    Let $\bar \lambda \deq (1 + \max_{t \in [n]}\lambda_t)/(1 - \max_{t\in [n]}\lambda_t)$.
    For a bounded function $f:\mcX \to [0,B]$, we have with \citet{fan2021hoeffding} that $\sum_t f(X_t)$ is sub-Gaussian\footnote{A random variable $X$ is sub-Gaussian with variance proxy $\sigma^2$ if $\E{X}$ is finite and $\E{\exp(t(X - \E{X}))} \leq \exp(\sigma^2t^2/2)$ for all $t \in \mdR$ \citep{boucheron2013concenetration}.}
    with variance proxy
\begin{equation*}
    \sigma^2 \leq \bar \lambda \frac{nB^2}{4}.
\end{equation*}
Defining $F_n \deq \frac1n \sum_t f(X_t)$, we therefore have
\begin{equation*}
    \E{\exp(\nu \left(F_n - \E{F_n}\right)} \leq \exp\left(\frac{\bar \lambda B^2 \nu^2}{8n}\right).
\end{equation*}

Integrating both sides with respect to $\rho_0$ and applying Fubini's theorem gives us 
\begin{equation*}
    \Ep{\mcD}{\Ep{\rho_0}{\exp(\nu(\E{F_n} - F_n))}} \leq \exp\left(\frac{\nu^2\bar\lambda B^2}{8n}\right),
\end{equation*}
where $\mdE_\mcD$ is the expectation over the data set $\mcD$.
Applying Donsker and Varadhan's variational formula \citep{donsker1976asymptotic}, we get that 
\begin{equation*}
    \Ep{\mcD}{\exp\left(\sup_\rho\Big(\nu \Ep{\rho}{\E{F_n} - F_n} - \KL{\rho}{\rho_0}\Big)\right)} \leq \exp\left(\frac{\nu^2\bar \lambda B^2}{8n}\right).
\end{equation*}
With $f(x) = x^2$, we have that $B = R^2/(1 - \gamma)^2$ and can rewrite the inequality in terms of $L(\rho)$ and $\widehat L(\rho)$, such that 
\begin{equation*}
    \Ep{\mcD}{\exp\left(\sup_\rho\Big(\nu (L(\rho) - \widehat L_\mcD(\rho)) - \KL{\rho}{\rho_0}\Big) - \frac{\nu^2\bar \lambda B^2}{8n}\right)} \leq 1.
\end{equation*}
With Chernoff's bound \citep{chernoff52ameasure} we have for a fixed $s > 0$
\begin{align*}
    &\mdP_\mcD\left(\sup_\rho\Big(\nu (L(\rho)  - \widehat L_\mcD(\rho)) - \KL{\rho}{\rho_0}\Big) - \frac{\nu^2\bar \lambda B^2}{8n} > s\right) \\
    &\qquad\leq
    \Ep{\mcD}{\exp\left(\sup_\rho\Big(\nu (L(\rho) - \widehat L_\mcD(\rho)) - \KL{\rho}{\rho_0}\Big) - \frac{\nu^2\bar \lambda B^2}{8n}\right)} e^{-s}  \leq e^{-s}.
\end{align*}
With a choice of $s = \log(1/\delta)$ and rearranging the terms we get that with probability smaller than $\delta$
\begin{equation*}
    \sup_\rho\Big(\nu (L(\rho) - \widehat L_\mcD(\rho)) - \KL{\rho}{\rho_0}\Big) - \frac{\nu^2\bar \lambda B^2}{8n} > \log\frac1\delta,
\end{equation*}
that is, for all $\rho$
\begin{equation}
    L(\rho) \leq  \widehat L(\rho) + \frac{\nu \bar \lambda B^2}{8n} + \frac{\KL{\rho}{\rho_0} - \log\delta}{\nu},\label{eq:loss_bound}
\end{equation}
with probability $1 - \delta$. Finally, with \cref{prop:mse2value} we have that 
\begin{align*}
    &\Ep{X\sim \rho}{||T_\pi X - X||_{P_\pi}^2} \leq \widehat L(\rho) + \frac{\nu \bar \lambda B^2}{8n} + \frac{\KL{\rho}{\rho_0} - \log\delta}{\nu} - \gamma^2\mdE_{X\sim \rho}\varp{s\sim P_\theta}{X(s,\pi(s)},
\end{align*}
and applying \cref{Lem:bellman2value} on the left-hand side gives the claimed inequality.
\end{proof}

\subsection{Discussion on the bound}\label{appsec:bound}
We close this section with a series of remarks on the theorem.
\begin{enumerate}[label=(\roman*)]
    \item \emph{The chain $\{X_i\}$ and $f_i$.} While we focused on the Bellman error and a squared function $f = f_i(x) = x^2$, the proof can be trivially adapted to other chains and bounded functions.
    \item \emph{The spectral gap factor $\bar \lambda$.} As $\max_t \lambda_t\to 1$, $\bar \lambda \to \infty$. However, this implies increasing auto correlation in the transition kernel, which we consider to be unrealistic in practice. 
    \item \emph{The bound $B^2$.} The bound of $f$ appears squared in the final bound, i.e., $B^2 = R^4/(1-\gamma)^4$. Depending on $R$ and the choice of $\gamma \in (0,1)$, its influence could be large. This is mitigated in practice as the related term disappears with $n\to \infty$ or $\nu \to 0$. Note that the choice of $\nu$ is arbitrary. Theorem 2.4 together with Corollary 2.5 from \citet{alquier2024user} extends the bound with an infimum over a discretized grid of $\nu$ values.
    \item \emph{A justification for minimizing the right-hand side of the bound.} Corollary 2.3 in \citet{alquier2024user} shows that the Gibbs posterior, ${\hat \rho_\nu = \argmin_\rho \left(\hat L(\rho) + \KL{\rho}{\rho_0}/\nu\right)}$, minimizes the right-hand side of the bound. Therefore, we have that 
\begin{equation*}
    \mdP\left(L(\hat\rho_\nu) \leq \inf_\rho\left( \widehat L(\rho) + \frac{\nu \bar \lambda B^2}{8n} + \frac{\KL{\rho}{\rho_0} - \log\delta}{\nu}\right)\right) \geq 1 - \delta, 
\end{equation*}
which justifies our choice of using the right-hand side of the bound as our training objective.
\end{enumerate}

\subsubsection{Derivation of a related bound}\label{appsec:fard}
We conclude the discussion on the generalization bound with an alternative that we obtain by adapting \citet{fard2012pac}.

Using the sub-Gaussianity theorem by \citet{fan2021hoeffding} cited above, we have the following result.

\paragraph{Corollary.}
\emph{
Assuming a Markov chain $\{X_i\}_{i \geq 1}$ and bounded functions $f_i$ following the assumptions of Theorem~5 by \citet{fan2021hoeffding}, we have that}
\begin{equation*}
    \mdP\left(\sum_i f_i(X_i) - \sum_i \E{f_i(X_i)} \geq \veps \right) \leq \exp\left(-\frac{\veps^2}{2\bar\lambda \sum_i(b_i - a_i)^2/4}\right),
\end{equation*}
\emph{where $\bar \lambda = \frac{1 + \max_i\lambda_i}{1 - \max_i\lambda_i}$ and $\veps > 0$. Equivalently,}
\begin{equation*}
    \sum_i f_i(X_i) - \sum_i \E{f_i(X_i)} \leq \sqrt{\bar\lambda \sum_i(b_i - a_i)^2/2\log(1/\delta)},
\end{equation*}
\emph{with probability $1 - \delta$ where $\delta \in (0,1)$.}

The result follows after a straightforward application of Chernoff's bound using the sub-Gaussianity of $\sum_i f_i$.  

Assuming $f = f_i$ for all $i$,  and $f(X) \in [0,B]$, we can further simplify the result of the corollary to  
\begin{equation*}
    \frac1n\sum_i f(X_i) - \frac1n\sum_i \E{f(X_i)} \leq \sqrt{\frac{\bar\lambda B^2}{2n}\log(1/\delta)},
\end{equation*}
with probability $1 - \delta$. Assuming a sufficiently large number of observations such that $n > \bar \lambda B^2 / 2$, the assumptions of \citet{fard2012pac}'s Theorem 1 hold, and 
we have, with probability $1 - \delta$, that
\begin{equation*}
\frac1n\sum_i f(X_i)- \frac1n\sum_i \E{f(X_i)} \leq  \sqrt{\frac{\log(2n/\bar \lambda B^2) - \log \delta + \KL{\rho}{\rho_0}}{2n/\bar \lambda B^2 - 1}},
\end{equation*}
where $\rho$ and $\rho_0$ are distributions over a measurable hypothesis space $\mcF$ and $f \in \mcF$.  
Compared with the bound in $\eqref{eq:loss_bound}$, it suffers from the constraint on the number of data points. Since $n > R^4/(1-\gamma)^4$ will rarely be fulfilled in practice for typical values of $\gamma$, e.g., $\gamma = 0.99$ in our paper, it cannot be used as a theoretically justified training objective. 
However, if its constraints were fulfilled, we could use their approach to obtain a tighter generalization bound. 

\subsection{Derivation of the approximation to the KL divergence term}
The approximation to the Kullback-Leibler divergence term discussed in the main paper is derived via
\begin{align*}
    \text{KL}&\big(\rho(s,\pi(s)) || \rho_0(s,\pi(s)) \big) \approx \mdE_{X \sim \rho} [\log f_\rho(X|s,a) - \log f_{\rho_0}(X|s,a)]\\
    &\approx  \frac{1}{K} \sum_{k=1}^{K} \log f_\rho(X_k|s,\pi_k(s)) - \log f_{\rho_0}(X_k|s,\pi_k(s))\\
    &= \frac{1}{2}\frac{1}{K} \sum_{k=1}^{K} \left (-\log \sigma_\pi^2(s) - \frac{(r+\gamma \mu_\pi(s)-X_k)^2}{\sigma_\pi^2(s)} + \log (\gamma^2 \sigma_0^2) + \frac{(r+\gamma \bar{\mu}_\pi(s')-X_k)^2}{\gamma^2 \sigma_0^2} \right )\\
    &= \frac{1}{2} \Bigg (-\log \sigma_\pi^2(s) -  \frac{1}{\sigma_\pi^2(s)} \underbrace{ \frac{1}{K} \sum_{k=1}^{K} (r+\gamma \mu_\pi(s)-X_k)^2}_{=\frac{K-1}{K}\sigma_\pi^2(s)} + K\log (\gamma^2 \sigma_0^2) + \frac{1}{K} \sum_{k=1}^{K}\frac{(r+\gamma \bar{\mu}_\pi(s')-X_k)^2}{\gamma^2 \sigma_0^2} \Bigg )\\
    &= \frac{1}{2K} \sum_{k=1}^{K} \Bigg(\frac{(r+\gamma \bar{\mu}_\pi(s')-X_k)^2}{\gamma^2 \sigma_0^2} - \log \sigma_\pi^2(s) \Bigg ) + \text{const}.
\end{align*}

\subsection{Comparison with BootDQN-P \citep{osband2018randomized}'s loss}\label{appsec:comp_bootdqn}
Assuming for notational simplicity a single observation and no bootstrapping, our loss in \eqref{eq:critic_loss} simplifies to 
\begin{equation*}
  \mcL_\text{PBAC} \deq
  \underbrace{\frac{1}{K} \sum_{k=1}^K \Big(r + \gamma \bar{X}_k(s', \pi(s'))-X_k(s,\pi(s)) \Big)^2}_\text{Diversity}
  +\underbrace{\frac{1}{ K}   \sum_{k=1}^{K}  \frac{ \big (r+\gamma \bar{\mu}_\pi(s')-X_k(s,\pi(s)) \big)^2}{2\gamma^2 \sigma_0^2} }_{\text{Coherence}}
  - \underbrace{\frac{2\gamma^2+1}{2}  \log \sigma_\pi^2(s)}_{\text{Propagation}}. 
\end{equation*}

Compare this with \citet{osband2018randomized}'s loss (equation (6) therein), translated to our notation as,
\begin{equation*}
    \mathcal{L}_\text{BDQN}^{(k)} \deq \left(r + \gamma(\bar X_k + p_k)(s',\pi(s')) - (X_k + p_k)(s,\pi(s))\right)^2,
\end{equation*}
where $p_k$ is the $k$-th fixed prior function, and $\bar X_k$ the $k$-th target. 
It can be interpreted as a perturbed version of the \emph{diversity} component of our loss term, lacking the coherence component term, which as our ablation shows, is the more important one.

\section{Experimental details} \label{appsec:exp_details}

\subsection{Prior work on delayed rewards and reward sparsity in continuous control}\label{appsec:prior_sparsity}

How to structure a delayed reward environment is a matter of debate. 
We identify two broad groups of environments in the current state-of-the-art literature, without claiming this to be an exhaustive survey.
These consist of either of native, i.e., inherent sparsity in the original environment, or are custom modifications of dense reward environments made by the respective authors.
We follow the naming convention for each of the environments as used in the Gymnasium~\citep{towers2024gymnasium}, DMControl~\citep{tassa2018deepmind}, and Meta-World \citep{yu2020meta} libraries unless otherwise noted. 

\paragraph{(i) Binary rewards based on proximity to a target state.} \citet{houthooft2016vime} and \citet{fellows2021bayesian} study Mountain Car continuous and sparse cartpole swing up. 
\citet{fellows2021bayesian} reach a reward of around 500 in cartpole within more than 1.5 million environment interactions, far behind the levels we report in our experiments. \citet{curi2020efficient} study reacher, pusher, a custom sparsified version of reacher with a reward squashed around the goal state, and action penalty with an increased share. \citet{houthooft2016vime}, \citet{mazoure19leveraging}, and \citet{amin2021locally} sparsify various MuJoCo locomotors by granting binary reward based on whether the locomotor reaches a target $x$-coordinate. This reward design has limitations: (i) Since there is no action penalty, the locomotor does not need to use its actuators in the most efficient way. (ii) Since there is no forward reward proportional to velocity, the locomotor's performance after reaching a target location becomes irrelevant. For instance, a humanoid can also fall down or stand still after reaching the target.

\paragraph{(ii) Increased action penalties.} \citet{curi2020efficient} study cartpole and MujoCo HalfCheetah with increased action penalties.  \citet{luis2023model} use pendulum swingup with nonzero reward only in the close neighborhood of the target angle. They also apply an action cost and perturb the pendulum angle with Gaussian white noise. They also report results on PyBullet Gym locomotors HalfCheetah, Walker2D, and Ant with dense rewards. Their follow-up work \citep{luis2023value} addresses a larger set of DMC environments where locomotors are customized to receive action penalties. The limitation of this reward design is that increasing the action penalty is not sufficient to sparsify the reward, as the agent can still observe relative changes in the forward reward from its contributions to the total reward. In some MuJoCo variants, the agent also receives rewards on the health status, which is a signal about intermediate success against the sparse-reward learning goal. Furthermore, action penalties are typically exogenous factors determined by energy consumption in the real world, making the challenge artificial.

In \cref{sec:experiments}, we report results on the most difficult subset of some natively sparse environments and devise new locomotion setups that overcome the aforementioned limitations. We discuss them in the next subsection.

\subsection{Experiment pipeline design} \label{app:env}

The entire experimental pipeline is implemented in PyTorch \citep[version $2.4.1$]{paszke2019pytorch}. The experiments are conducted on eleven continuous control environments from two physics engines, MuJoCo \citep{todorov2012mujoco,brockman2016openai} and DeepMind Control (DMControl) Suite \citep{tassa2018deepmind}, along with  seven additional robotic tasks from sparse Meta-World \citep{yu2020meta}. %
All MuJoCo environments used are from version 4 of the MuJoCo suite (V4), and sparse DMControl environments. Multi-task sparse Meta-World environments are from version 2. 

\subsubsection{MuJoCo}\label{appsec:mujocodelay}
From the two locomotors that can stand without control, we choose \emph{ant} as it is defined in $3D$ space compared to the $2D$ defined \emph{halfcheetah}. From the two locomotors that must learn to maintain their balance, we choose \emph{hopper} instead of \emph{walker2d} as hopping is a more dexterous locomotion task than walking. We also choose \emph{humanoid} as an environment where a $3D$ agent with a very large state and action space must learn to maintain its balance.  The reward functions of the MuJoCo locomotion tasks have the following generic form:
\begin{equation*}
     r \deq \underbrace{\frac{dx_t}{dt}}_{\text{Forward reward}}   - \underbrace{w_a ||a_t||^2}_{\text{Action cost}} + \underbrace{H}_{\text{Health reward}}.
\end{equation*}
\emph{Forward reward} comes from the motion speed of the agent toward its target direction and drives the agent to move efficiently toward the goal. 
\emph{Health reward} is an intermediate incentive an agent receives to maintain its balance. The \emph{action cost} ensures that the agent solves the task using minimum energy. We implements sparsity to the locomotion environments in the following way using the following reward function template:
\begin{equation*}
     r_{\mathrm{delayed}} := \underbrace{\frac{dx_t}{dt} \mathbf{1}_{x_t> c}}_{\text{Forward reward}}   - \underbrace{w_a ||a_t||^2}_{\text{Action cost}}. 
\end{equation*}
This function delays forward rewards until the center of mass of the locomotor $x_t$ reaches a chosen target position $c$, which is called the {\it positional delay}. This reward also removes the healthy reward to ensure that the agent does not get any incentive by solving an intermediate task. The sparsity of a task can be increased by increasing the positional delay. Detailed information on the sparsity levels used in our experiments is listed in \cref{tab:propmujoco}. For each environment, the largest positional delay is chosen, i.e., maximum sparsity, where at least one model can successfully solve the task. Beyond this threshold, all models fail to collect positive rewards within $300000$ environment interactions. The structure of these experiments follows the same structure as the $n$-chains thought experiment, which is studied extensively in theoretical work. The essential property is that there is a long period of small reward on which an agent can overfit. See, e.g., the discussion by \citet{strens2000bayesian} or \citet{osband2018randomized} for further details. Information on the state and action space dimensionalities for each of the three environments is available in \cref{tab:prop}.

\begin{table}
\centering
\caption{\emph{MuJoCo environment reward hyperparameters.} See the description in the text for an explanation on each of the parameters.}
\vspace{1.0em}
\label{tab:propmujoco}
\adjustbox{max width=0.98\textwidth}{
\begin{tabular}{lccc}
\toprule
\textsc{Task} & Positional delay $c$  & Action cost weight $w_a$ & Health reward $H$ \\ 
\midrule
 ant &0 &5e-1&1\\
 ant (delayed) &0 &5e-1&0\\
 ant (very delayed) & 2 &5e-1&0\\
\cmidrule(lr){1-4}
 hopper &0 & 1e-3&1\\
 hopper (delayed) & 0&1e-3&0\\
 hopper (very delayed) & 1&1e-3&0\\
\cmidrule(lr){1-4}
  humanoid & 0 &1e-1&5\\ 
 humanoid (delayed) &0 &1e-1&0\\
\bottomrule
\end{tabular}
}

\end{table}

\subsubsection{Deepmind control}

From DMControl (DMC), we choose \emph{ballincup}, \emph{cartpole}, and \emph{reacher}, as they have sparse binary reward functions given based on task completion. In \emph{ballincup}, the task is defined as whether the relative position of the ball to the cup centroid is below a distance threshold. In \emph{cartpole}, it is whether the cart position and pole angle are in respective ranges $(-0.25, 0.25)$ and $(0.995 , 1)$.  Finally, in \emph{reacher}, it is the distance between the arm and the location of a randomly placed target coordinate. Information on the state and action space dimensionalities for each of the three environments is available in \cref{tab:prop}. We did not consider the DMC locomotors as they use the same physics engine as MuJoCo and their reward structure is less challenging due to the absence of the action penalty and the diminishing returns given to increased velocities.

\subsubsection{Meta-World}
We conducted our experiments using robotic tasks from the multi-task Meta-World environment, originally introduced by \citet{yu2020meta}, with sparse rewards where the agent receives a reward of one only by reaching the goal and zero otherwise. In all environments, the goal location is hidden. The experiments were conducted over \num{1000000} environment steps across five different seeds, with success rates evaluated for all baselines. Our experimental scheme follows \citet{fu2024furl}.
Three environments were excluded from our evaluation compared to them because none of the baseline methods were able to solve these three tasks. These tasks are particularly challenging under sparse reward conditions, as they require mastering multiple subtasks, which significantly increases the difficulty \cite{fu2024furl}.
See \cref{tab:hyper} for additional details on the hyperparameters. We provide a description of each task in \cref{tab:TaskDescriptions}. The area under the success rate curve and final episode success rate results are summarized in \cref{tab:success_result_table} and visualized in \cref{fig:metaworld_curves}.

\begin{table}
\centering
\caption{\emph{State and action space dimensionalities for MuJoCo (MJC) and DMControl (DMC).}}
\vspace{1.0em}
\label{tab:prop}
\adjustbox{max width=0.98\textwidth}{
\begin{tabular}{llcc}
\toprule
& \textsc{Task} & $|\mathcal{S}|$ & $|\mathcal{A}|$ \\
 \midrule
\parbox[t]{2mm}{\multirow{3}{*}{\rotatebox[origin=c]{90}{MJC}}}
& ant &27& 8\\
& hopper &11&3\\
& humanoid&376&17\\
\cmidrule(lr){2-4}
\parbox[t]{2mm}{\multirow{3}{*}{\rotatebox[origin=c]{90}{DMC}}}
& ballincup &8& 2\\
& cartpole & 5&1\\
& reacher&11&2\\

\bottomrule
\end{tabular}
}
\end{table}

\subsection{Evaluation methodology}\label{appsec:eval_methodology}

\paragraph{Performance metrics.}
We calculate the \emph{Interquartile Mean (IQM)} of the final episode reward and of the  \emph{area under learning curve (AULC)} as our performance scores, where the former indicates how well the task has been solved and the latter is a measure of learning speed. The AULC is calculated using evaluation episodes after every \num{20000} steps.
We calculate these rewards over ten repetitions on different seeds, where each of the methods gets the same seeds. All methods, including our approach and the baselines, utilize the same warmup phase of \num{10000} steps to populate the replay buffer before initiating the learning process. 
For evaluating sparse Meta-World environments, we provide average success rate evaluations. In this setup, we consider both the average of the final episode success values and the average area under the success rate curve, computed over five seeds.
\paragraph{Statistical Significance.}
We always highlight the largest IQM in bold in each of the tables. For each of the alternatives, we run a one-sided paired t-test on the null hypothesis that the two models have the same mean. If this null is not rejected at a significance level $p<0.05$, we also highlight it in bold.

\subsection{Hyperparameters and architectures}\label{appsec:hyper}

\paragraph{PBAC specific hyperparameters.}
PBAC has three hyperparameters: bootstrap rate, posterior sampling rate, and prior variance. We observe PBAC to work robustly on reasonably chosen defaults. 
See \cref{appsec:ablation} for an ablation on a range of these for the cartpole and delayed ant environments. We list the hyperparameters we used for each environment in \cref{tab:hyper}.

\begin{table}
\centering

\caption{\emph{Hyperparameters specific to PBAC.} The chosen bootstrap rate (BR), posterior sampling rate (PSR), prior variance (PV) , and Coherence Propagation rate  (CPR) for each of the environments.}
\vspace{1.0em}
\label{tab:hyper}
\adjustbox{max width=0.98\textwidth}{
\begin{tabular}{lcccc}
\toprule
\textsc{Task} & BR & PSR & PV\\
 \midrule
 ant & 0.05& 5& 1.0\\
 hopper &0.01 &5 & 1.0 \\
 humanoid& 0.05& 5& 1.0 \\
 ballincup& 0.05&5 &1.0 \\
 cartpole & 0.05& 10 & 1.0\\
 reacher& 0.05&5 & 1.0\\
 ant (delayed) & 0.05&5 &1.0 \\
 ant (very delayed) &0.05 &5 &5.0 \\
 hopper (delayed) & 0.05&5 & 1.0 \\
 hopper (very delayed) & 0.05&5 & 1\\
 humanoid (delayed) & 0.05&5 &  1.0 \\
 window close (sparse) &0.01&10&5.0 \\
 window open (sparse) &0.05&10&1.0 \\
 drawer close (sparse) &0.05&10&1.0  \\
 drawer open (sparse)&0.05&10&1.0 \\
 reach hidden (sparse)&\textbf{0.01}&10&1.0 \\
 button press topdown (sparse)&0.05&10&1.0 \\
 door open (sparse)&0.05&10&1.0 \\
\bottomrule
\end{tabular}
}
\end{table}

\begin{table}
\centering
\caption{\emph{Task description}. Description of all tasks for Meta-World sparse environments.}
\vspace{1em}
\label{tab:TaskDescriptions}
\adjustbox{max width=0.98\textwidth}{
\begin{tabular}{lcccc}
\toprule
\textsc{environment} & \textsc{instruction}\\
 \midrule

 window-close-v2  & Push and close a window.\\
 window-open-v2  & Push and open a window.\\
 drawer-close-v2  & Push and close a drawer.\\
 drawer-open-v2 & Open a drawer.\\
 reach-hidden-v2 & Reach a goal position.\\
 button-press-topdown-v2 & Press a button from the top.\\
 door-open-v2 & Open a door with a revolving joint.\\
\bottomrule
\end{tabular}
}
\end{table}

\paragraph{Shared hyperparameters and design choices.}
We use \emph{layer normalization} \citep{ball2023efficient} after each layer to regularize the network, and a \emph{concatenated ReLU (CReLU)} activation function ~\citep{shang16understanding} instead of the standard ReLU activation, which enhances the model by incorporating both the positive and negative parts of the input and concatenating the results. 
This activation leads to potentially better feature representations and the ability to learn more complex patterns. 
Moreover, we rely on a high \emph{replay ratio (RR)} and a small \emph{replay buffer} size, which reduce the agent’s dependence on long-term memory and encourage it to explore different strategies. 
These are employed to improve the plasticity of the learning process. Recently, \citet{nauman2024overestimation} showed that a combination of these design choices can overall greatly improve the agent's learning ability.
Additionally, we opted to use the Huber loss function for all baseline models after observing in our preliminary trials that it consistently provided performance advantages across different baselines. The temperature parameter $\alpha$ is automatically adjusted during training to balance exploration and exploitation by regulating policy entropy, following \citet{haarnoja2018soft-app}.
All design choices found advantageous for our model and not harmful to others have also been applied to the baselines.
\cref{tab:hyperparameters} provides details on the hyperparameters and network configurations used in our experiments.

\begin{table}
\centering
\caption{\emph{Shared hyper-parameters.} Hyperparameters used by all methods.}
\vspace{1em}
\label{tab:hyperparameters}
\adjustbox{max width=0.98\textwidth}{
\begin{tabular}{lcc}
\toprule
Hyper-parameter & Value \\
\midrule
Evaluation episodes &10\\
Evaluation frequency & Maximum timesteps / 100\\
Discount factor $(\gamma)$ & 0.99\\
$n$-step returns &1 step\\
Replay ratio &5\\
number-of-critic-networks&10\\
Replay buffer size &100,000\\
Maximum timesteps$^{*}$ &300,000\\ 
Number of hidden layers for all networks &2 \\
Number of hidden units per layer &256\\
Nonlinearity & CReLU\\
Mini-batch size $(n)$ & 256\\
Network regularization method & Layer Normalization (LN) \citep{ball2023efficient}\\
Actor/critic optimizer &Adam \citep{kingma2014adam} \\
Optimizer learning rates $(\eta_{\phi},\eta_{\theta})$ &3e-4\\
Polyak averaging parameter $(\tau)$ &5e-3\\
\bottomrule
\end{tabular}}\\
{\tiny ${}^{*}$ Ballincup, reacher, and cartpole use a reduced number of maximum steps. The former two use 100.000 and the latter 200.000.}
\end{table}

\paragraph{Actor and critic networks.} 
Our implementation of PBAC along with the proposed baselines shares the architectural designs provided in \cref{tab:arch} for each critic network in the ensemble and actor network. The quantities $d_s$ and $d_a$ denote the dimensions of the state space and the action space, respectively. 
The output of the actor network is passed through a $\tanh(\cdot)$ function for the deterministic actor networks used in BEN and BootDQN-P. 
We implement the probabilistic actors of DRND and PBAC as a \emph{squashed Gaussian head} that uses the first $d_a$ dimensions of its input as the mean and the second $d_a$ dimensions as the variance of a normal distribution.
A squashed Gaussian refers to a Gaussian distribution transformed to the $(-1,1)$ interval via $\tanh$. The network predicts the mean and log variance of this Gaussian distribution.

\begin{table}
\centering
\caption{\emph{Actor and critic architectures.} Here, $d_s$ and $d_a$ are the dimensionalities of the state and action spaces.} 
\vspace{1em}
\label{tab:arch}
\adjustbox{max width=0.98\textwidth}{
\begin{tabular}{cc}
\toprule
\multicolumn{1}{c}{\small Actor network} & \multicolumn{1}{c}{\small Critic network} \\ \midrule
\begin{tabular}[c]{@{}c@{}}\small\texttt{Linear($d_s$, 256)}\\ \small\texttt{Layer-Norm}\\\small\texttt{CReLU()}\end{tabular} & \begin{tabular}[c]{@{}c@{}}\small\texttt{Linear($d_s+d_a$, 256)} \\ \small\texttt{Layer-Norm}\\ 
\small\texttt{CReLU()}  \end{tabular} \\ 
\begin{tabular}[c]{@{}c@{}}\small\texttt{Linear(256, 256)}\\ \small\texttt{Layer-Norm}\\ \small\texttt{CReLU()}\end{tabular}  & \begin{tabular}[c]{@{}c@{}}\small\texttt{Linear(256, 256)}\\ \small\texttt{Layer-Norm}\\ \small\texttt{CReLU()}\end{tabular} \\ 
\begin{tabular}[c]{@{}c@{}}\small\texttt{Linear(256,  $2\times d_a$)}\\ \end{tabular} & \small\texttt{Linear(256,1)} \\ 
\begin{tabular}[c]{@{}c@{}}\small\texttt{SquashedGaussian($2\times d_a$,  $ d_a$)}\\ \end{tabular} &  \\
 \bottomrule
\end{tabular}}
\end{table}

\subsubsection{Baselines} \label{appsec:baselines}
We compare PBAC against a range of state-of-the-art general-purpose actor-critic methods that are all empowered by ensembles.
All design choices mentioned above that were found advantageous for our model and not harmful to others have also been applied to the baselines. Below, we also explain further changes compared to the original works that we found to be beneficial in preliminary experiments.

\paragraph{BEN.}
Bayesian Exploration Networks (BEN), introduced by \citet{fellows2024ben} serves as the best representative that can learn a Bayesian optimal policy and handle the exploration versus exploitation tradeoff. 
We modify BEN by relying on Bayesian deep ensembles~\citep{lakshminarayanan2017simple} instead of the normalizing flow-based approach used in the original work. 
An ensemble of $K-1$ heteroscedastic critics learns a heteroscedastic univariate normal distribution over the Bellman target, while the $K$th critic, regularized by the ensemble, guides the actor network.

\paragraph{BootDQN-P.}
Bootstrapped DQN with randomized prior functions \citep{osband2018randomized}, a Bayesian model-free approach, serve as a close relative to our method. 
As BootDQN-P was introduced for discrete action spaces rather than continuous ones, we adapt it as follows. 
We use a deterministic version of PBAC's actor, i.e., an actor ensemble with a shared trunk and $K$ \emph{deterministic} heads.
As with PBAC a random head is chosen and fixed for a number of steps. See the \emph{posterior sampling rate} parameter in \cref{tab:hyper_bdqnp} and the references to posterior sampling provided in the main paper.
Additionally, instead of a fixed bootstrap mask as proposed by \citet{osband2018randomized}, we sample a mask at every objective computation as discussed in Step (vi) in the main paper.
Throughout all environments we share most parameters with PBAC. Changes for specific parameters are discussed in \cref{tab:hyper_bdqnp}.
BootDQN-P's randomized priors allow the model to explore even in the presence of sparse rewards.
A prior scaling (PS) parameter regulates their influence.

\begin{table}
\centering
\caption{\emph{Hyperparameters specific to BootDQN-P.} The chosen bootstrap rate (BR), posterior sampling rate (PSR), and prior scaling (PS) for each of the eleven environments. Changes from the defaults are marked in bold.}
\vspace{1em}
\label{tab:hyper_bdqnp}
\adjustbox{max width=0.98\textwidth}{
\begin{tabular}{lcccc}
\toprule
\textsc{Task} & BR & PSR & PS\\
 \midrule
 ant & 0.05& 5& 5.0\\
 hopper &0.05 &5 & 5.0 \\
 humanoid& \textbf{0.1}& \textbf{1}& 5.0 \\
 ballincup& 0.05&5 &5.0 \\
 cartpole & 0.05&\textbf{10} & \textbf{1.0} \\
 reacher& 0.05& \textbf{1} & \textbf{1.0} \\
 ant (delayed) & \textbf{0.1}&\textbf{1} &\textbf{1.0} \\
 ant (very delayed) &0.05 &5 &5.0 \\
 hopper (delayed) & 0.05&5 & 5.0 \\
 hopper (very delayed) & 0.05&5 & 5.0 \\
 humanoid (delayed) & \textbf{0.1}&\textbf{1} &  \textbf{9.0} \\
 window close (sparse) &0.05&5&5.0 \\
 window open (sparse) &0.05&5&5.0  \\
 drawer close (sparse) &0.05&5&5.0   \\
 drawer open (sparse)&0.05&5&5.0  \\
 reach hidden (sparse)&0.05&5&5.0  \\
 button press topdown (sparse)&0.05&5&5.0  \\
 door open (sparse)&0.05&5&5.0  \\ 
\bottomrule
\end{tabular}
}
\end{table}

\paragraph{DRND.}
Distributional randomized network distillation (DRND)~\citep{yang2024exploration} models the distribution of prediction errors from a random network. This distributional information is used as a signal to guide exploration. As in the original work, we integrate it into a soft actor-critic (SAC)~\citep{haarnoja2018soft} framework.
This random predictor network is trained via the same objective and uses the same architecture as proposed by \citet{yang2024exploration}. Actor and critic networks follow the architectural choices described above. Additionally, we optimize the $\alpha$ scaling parameter in SAC, as is common practice.

\section{Pseudocode}\label{appsec:pseudocode}
For illustrative purposes, we demonstrate the algorithmic steps for a deterministic actor training setting in \cref{alg:PBAC}. However, in the experiments, we use its soft actor version, where the policy network follows a squashed Gaussian distribution, Bellman targets are accompanied by a policy entropy bonus, and the entropy coefficient is learned following the scheme introduced in Soft Actor Critic applications \citep{haarnoja2018soft}.
The automatic entropy tuning mechanism dynamically adjusts the temperature parameter $\alpha$, balancing exploration and exploitation by controlling entropy’s influence on the objective function.

\begin{algorithm}%
  \caption{PAC-Bayesian Actor Critic (PBAC) }
  \begin{algorithmic}[1]
    \label{alg:PBAC}
    \STATE {\bfseries Input:} Polyak parameter $\tau\in(0,1)$, mini-batch size $n\in\mdN$, bootstrap rate $\kappa$, posterior sampling rate $\texttt{PSR}$,  prior variance $ \sigma_0^2$, number of ensemble elements $K$
    \STATE {\bfseries Initialize:} replay buffer $\mathcal{D}\gets \emptyset$, critic parameters $\{\theta_{k}\}$ and targets $\bar\theta_k\leftarrow \theta_k$, actor network trunk $g$ and heads $h_1, \ldots h_K$.
    \STATE \texttt{$s \gets $ env.reset()} and $e \gets 0$ (interaction counter)
    \WHILE{training}
    \IF{ $\bmod(e,\texttt{PSR})=0$}
      \STATE $j\sim \mathrm{Uniform}(K)$ 
      \STATE $\pi \gets \pi_{g \circ h_j}(s)$  (Update active critic)
    \ENDIF
    \STATE $a \gets \pi_j(s)$ and $(r,s') \gets \texttt{env.step}(a)$ and $e\gets e + 1$
    \STATE Store new observation: $\mathcal{D}\gets\mathcal{D} \cup (s,a,r,s')$
    \STATE Sample minibatch: $B \sim \mathcal{D}$ with $|B| = n$
    \STATE Sample a bootstrap mask: $b_{ik} \sim \mathrm{Bernoulli}(1-\kappa),\quad \forall [n] \times  [K]$ 
    \STATE Compute prior mean and posterior moments: $\forall (s_i,a_i,r_i,s'_i) \in B$ do
    \begin{align*}  
     \bar{\mu}_{\pi_{g \circ h_j}}(s'_i) &\gets \frac{1}{K} \sum_{k=1}^K b_{ik}(\bar{X}_k(s'_i,\pi_{g \circ h_j}(s'_i)))\\   
      \mu_{\pi_{g \circ h_j}}(s_i) &\gets \frac{1}{K} \sum_{k=1}^K b_{ik}(X_k(s_i,{\pi_{g \circ h_j}}(s_i)))\\
     \sigma_{\pi_{g \circ h_j}}^2(s_i) & \gets \frac{1}{K-1} \sum_{k=1}^K b_{ik}(X_k(s_i,{\pi_{g \circ h_j}}(s_i))-\mu_{\pi_{g \circ h_j}}(s_i))^2
    \end{align*}
    \STATE Update critics $k \in  [K]$: \begin{align*}
        &\theta_k \gets  \arg \min_{\theta_k} \Bigg \{ \frac{1}{n K} \sum_{i=1}^{n} \sum_{k=1}^K b_{ik} \Big(r_i + \gamma \bar{X}_k(s'_i, {\pi_{g \circ h_j}}(s'_i))-X_k(s_i,{\pi_{g \circ h_j}}(s_i)) \Big)^2\\
        &+ \frac{1}{ n K}  \sum_{i=1}^{n}  \sum_{k=1}^{K}  \frac{ b_{ik} \Big (r_i+\gamma \bar{\mu}_{\pi_{g \circ h_j}}(s'_i)-X_k(s_i,{\pi_{g \circ h_j}}(s_i)) \Big)^2}{2\gamma^2 \sigma_0^2} - \frac{\gamma^2+1/2}{n} \sum_{i=1}^{n} \log \sigma_{\pi_{g \circ h_j}}^2(s_i)\Bigg\} 
    \end{align*}   
    
    \STATE Update actor: \begin{align*} (g, h_1, \ldots, h_K) \gets \argmax_{g, h_1, \ldots, h_K}  \left[ \frac{1}{n K} \sum_{i=1}^n \sum_{k=1}^K X_k(s_i, \pi_{g \circ h_k}(s_i))\right]\\ \end{align*}
    \STATE Update critic targets: $\bar{\theta}_{k}\gets \tau \theta_{k} + (1-\tau)\bar{\theta}_{k} \; \text{ for } \; k \in [K]$
    \STATE {\bf if} episode end {\bf then} \texttt{$s \gets $env.reset()} {\bf else} $s \gets s'$

    \ENDWHILE
  \end{algorithmic}
\end{algorithm}

\section{Further results}\label{appsec:results}

\subsection{Reward curves}\label{appsec:curves}
See \cref{fig:all_curves} for the full reward curves of all MuJoCo and DMControl environments corresponding to the results presented in \cref{tab:result_table} in the main text. Success rate for Meta-World environments are available in \cref{fig:metaworld_curves}, and their corresponding results are available in \cref{tab:success_result_table}

\begin{figure}
    \centering
    \includegraphics[width=0.32\textwidth]{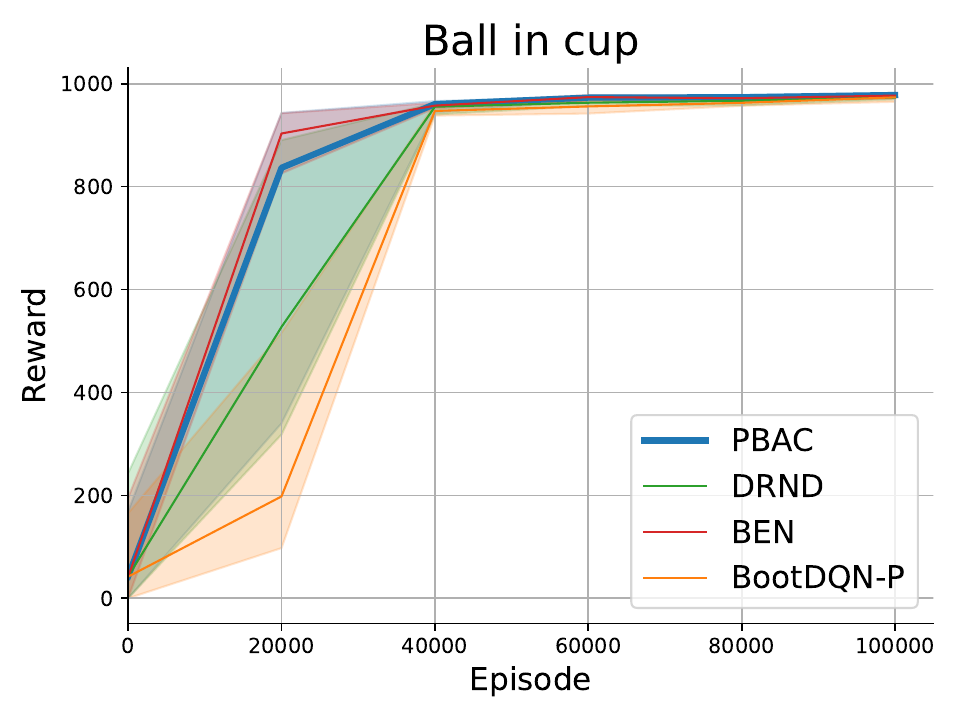}
    \includegraphics[width=0.32\textwidth]{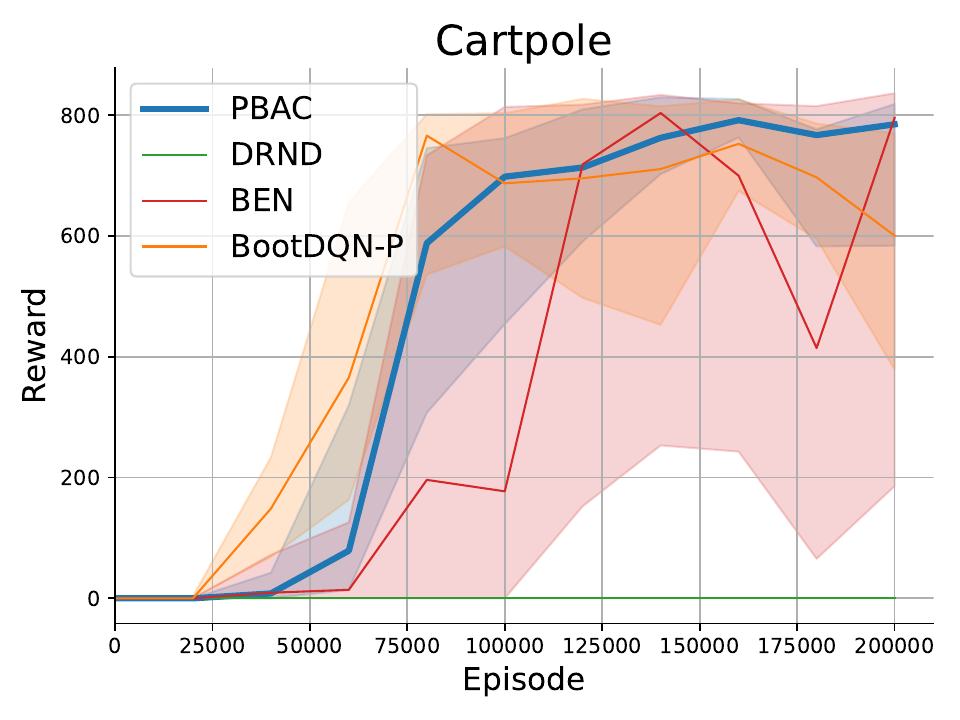}
    \includegraphics[width=0.32\textwidth]{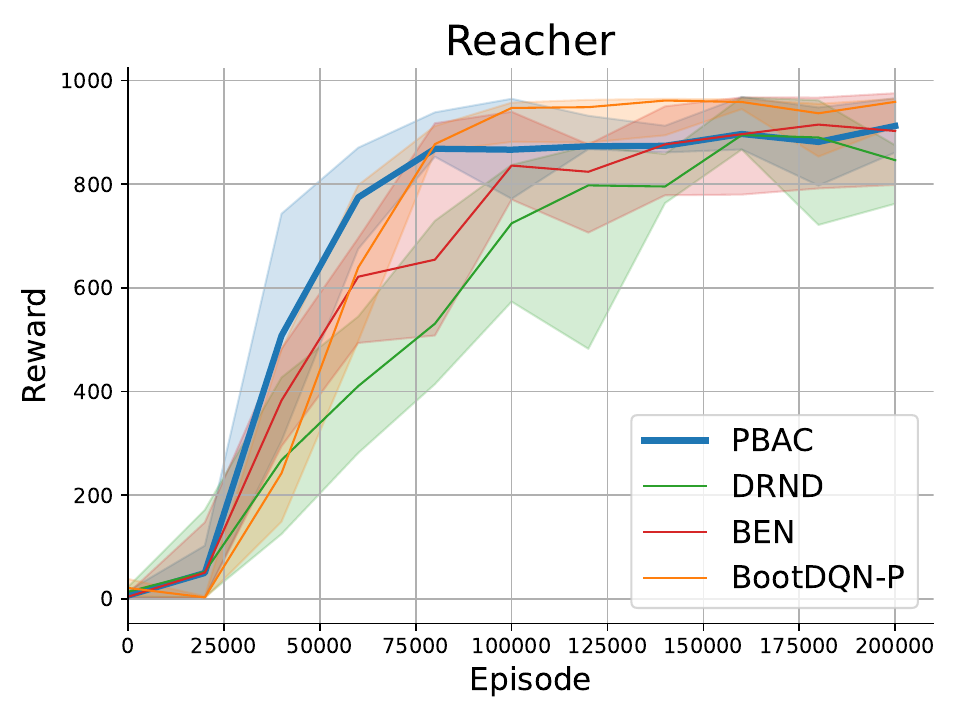}
    \includegraphics[width=0.32\textwidth]{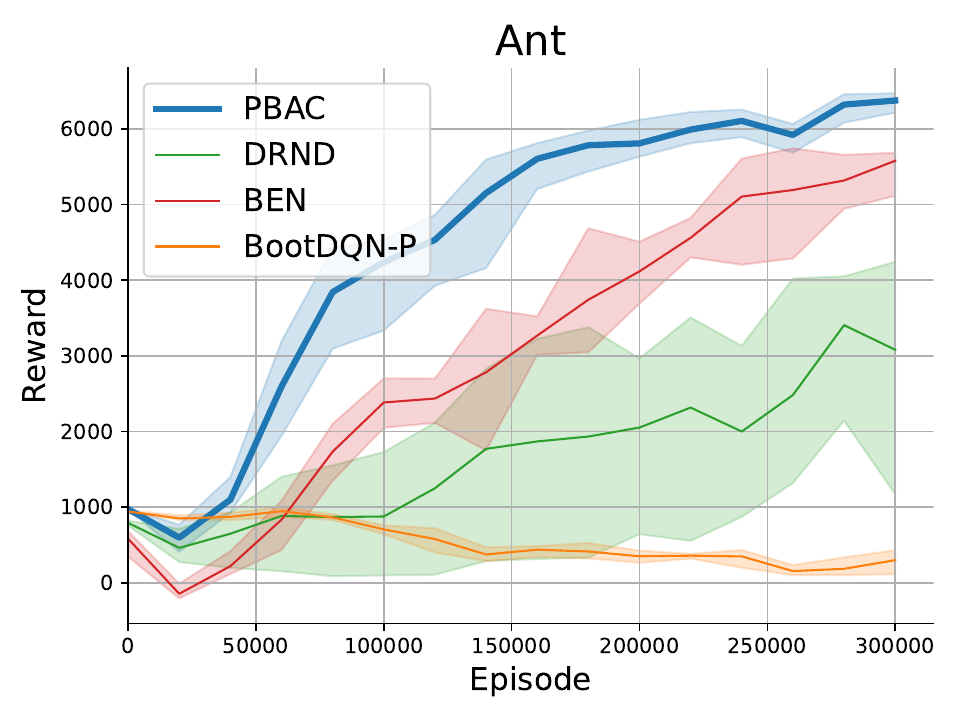}
    \includegraphics[width=0.32\textwidth]{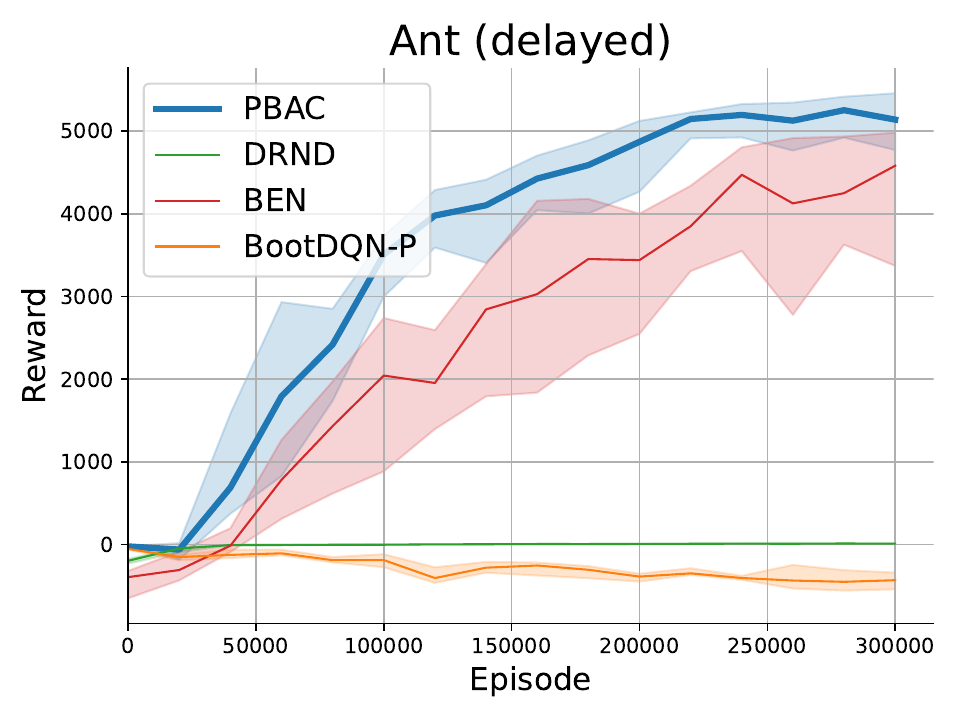}
    \includegraphics[width=0.32\textwidth]{figures/upd_result-sp-2.0-ant.pdf}
    \includegraphics[width=0.32\textwidth]{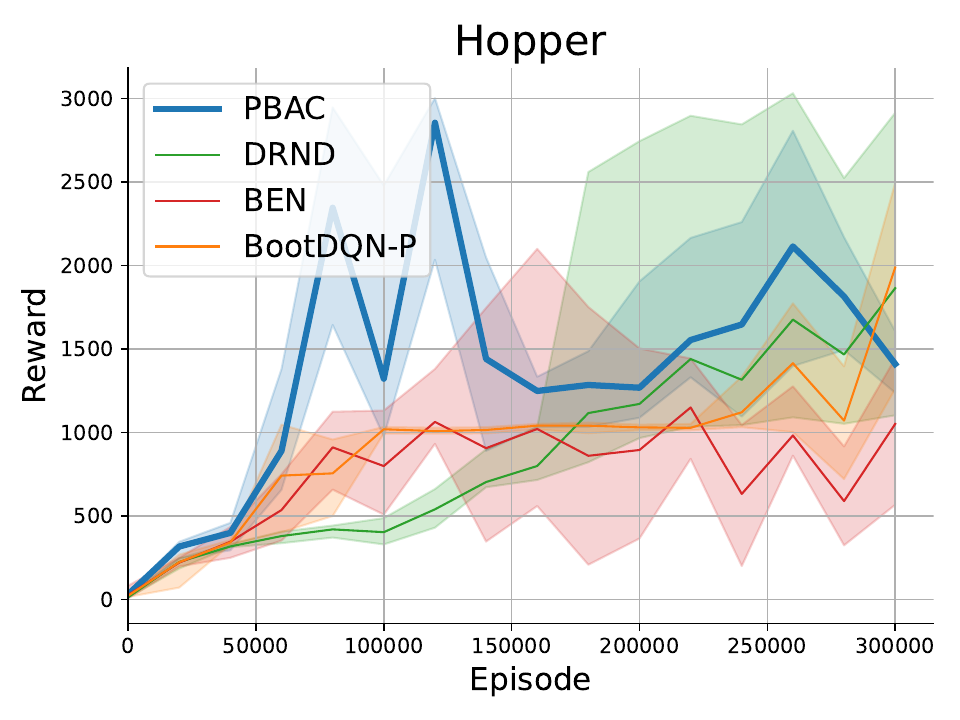}
    \includegraphics[width=0.32\textwidth]{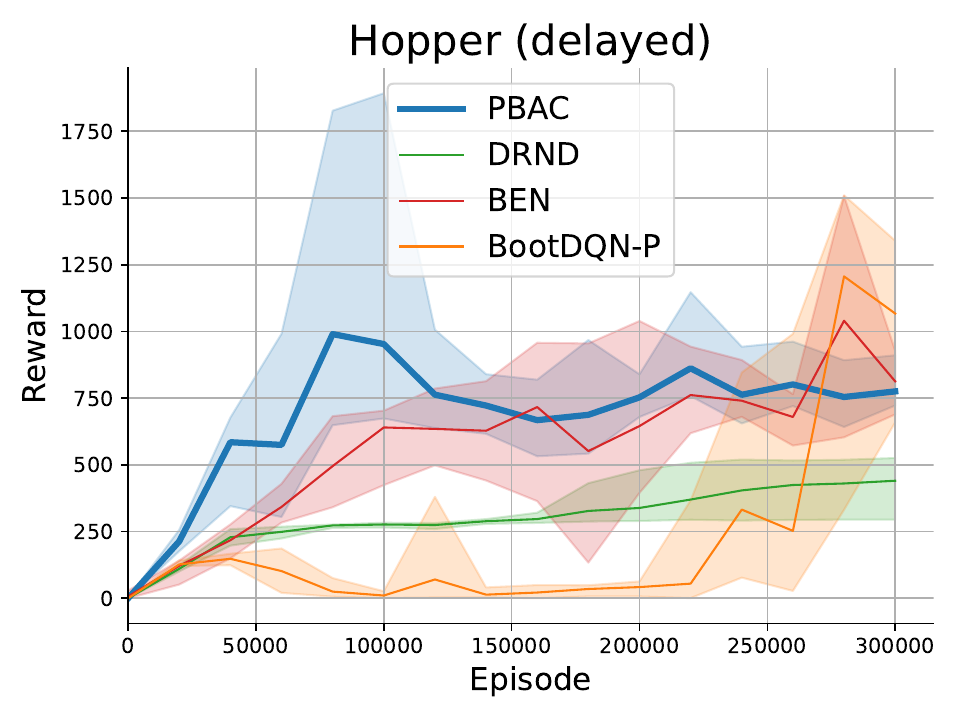}
    \includegraphics[width=0.32\textwidth]{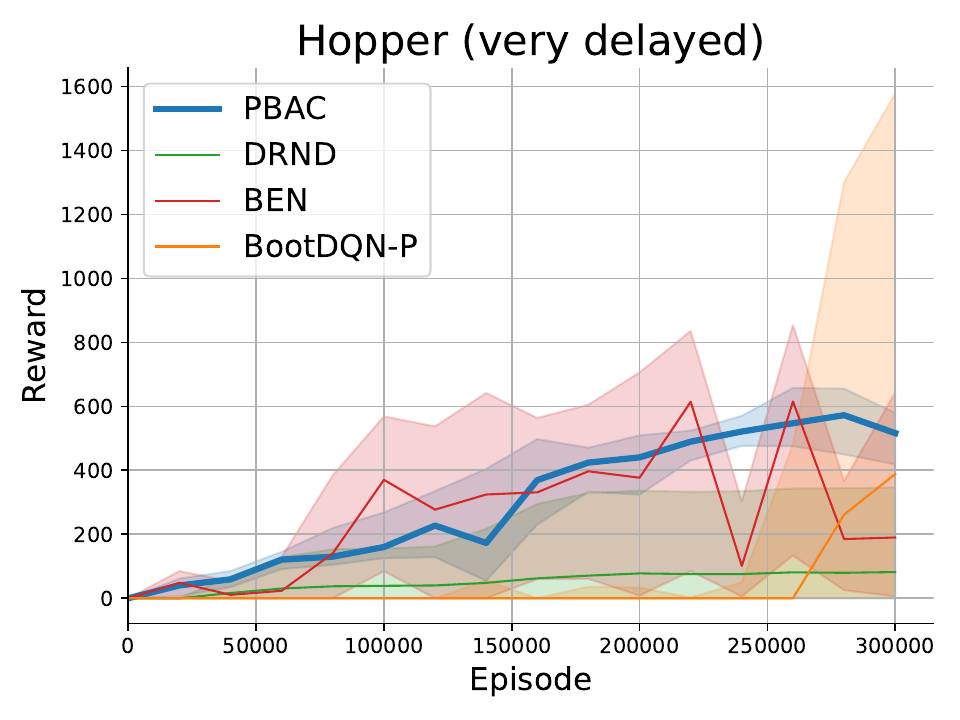}
    \includegraphics[width=0.32\textwidth]{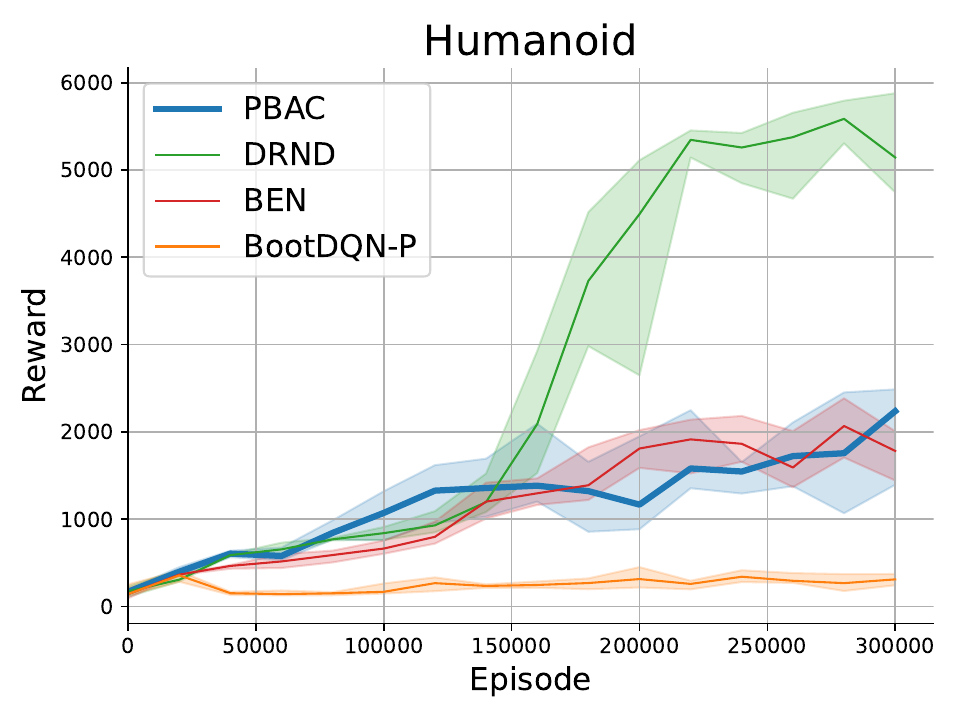}
    \includegraphics[width=0.32\textwidth]{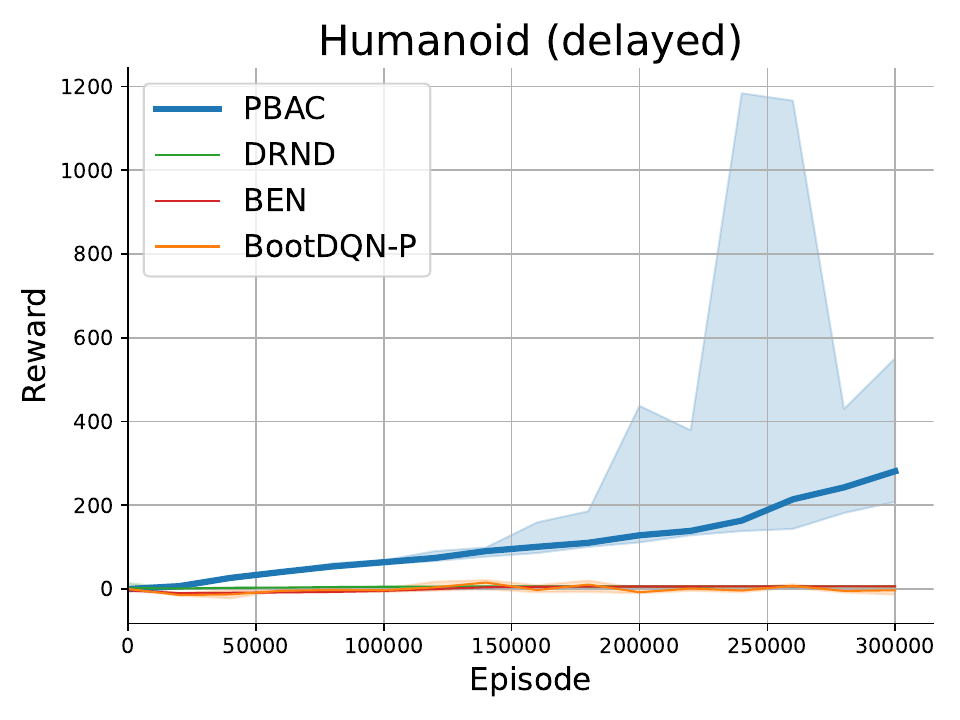}
    \caption{
    \emph{Reward curves.} The reward curves for all environments throughout training corresponding to the results presented in \cref{tab:result_table}. 
    Visualized are the interquartile mean together with the interquartile range over ten seeds.
    }
    \label{fig:all_curves}
\end{figure}

\begin{figure}
    \centering

    \includegraphics[width=0.33\textwidth]{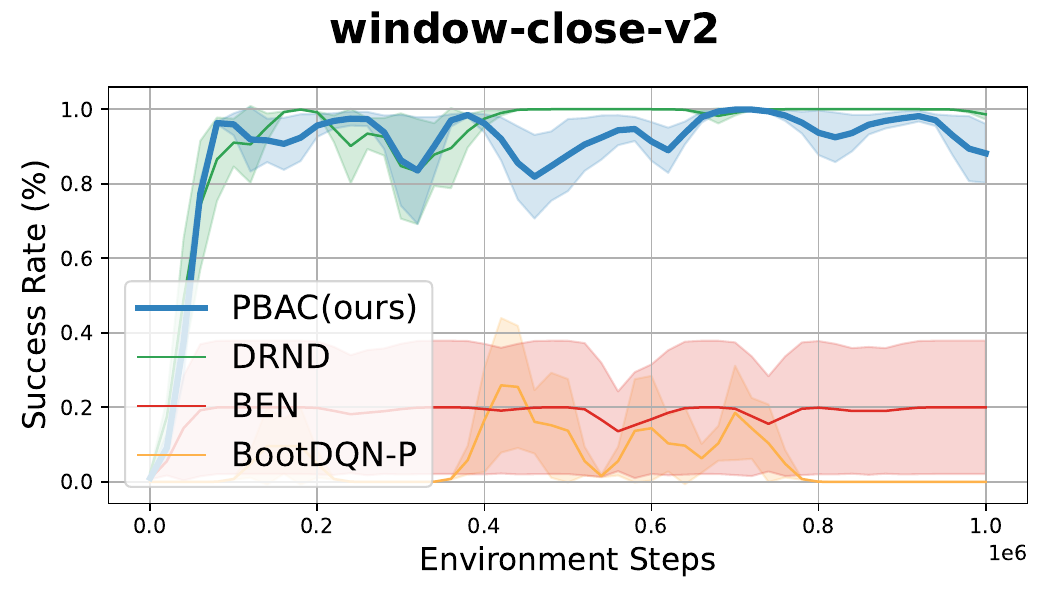}
    \includegraphics[width=0.33\textwidth]{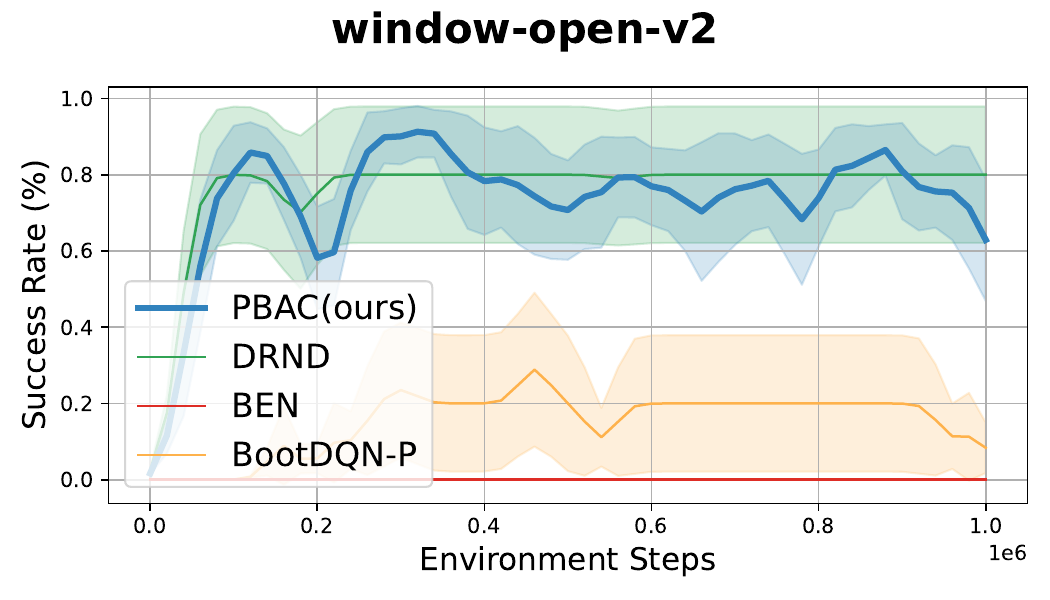}
    \includegraphics[width=0.33\textwidth]{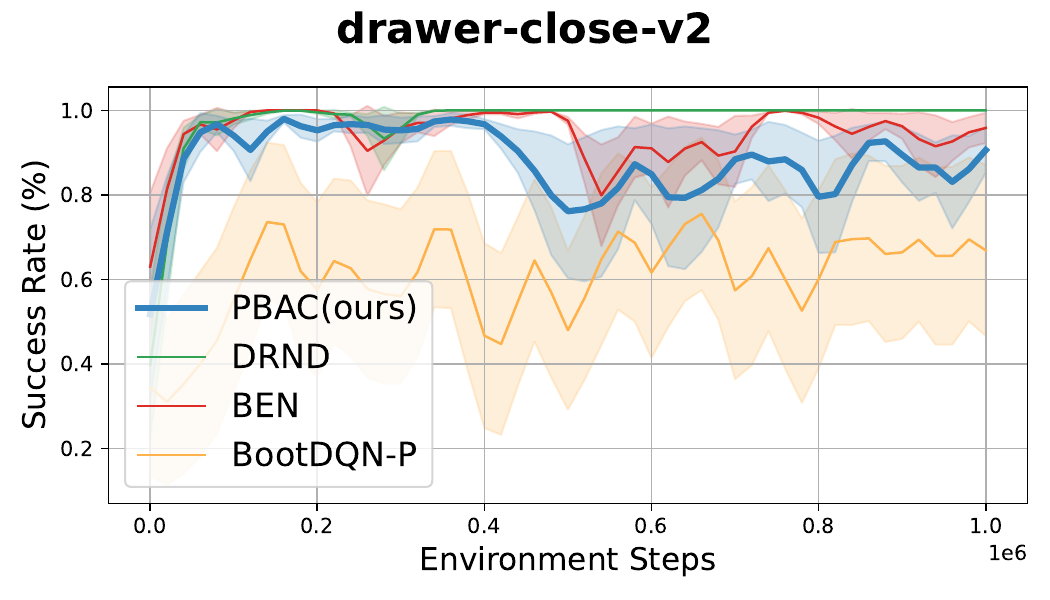}
    \includegraphics[width=0.33\textwidth]{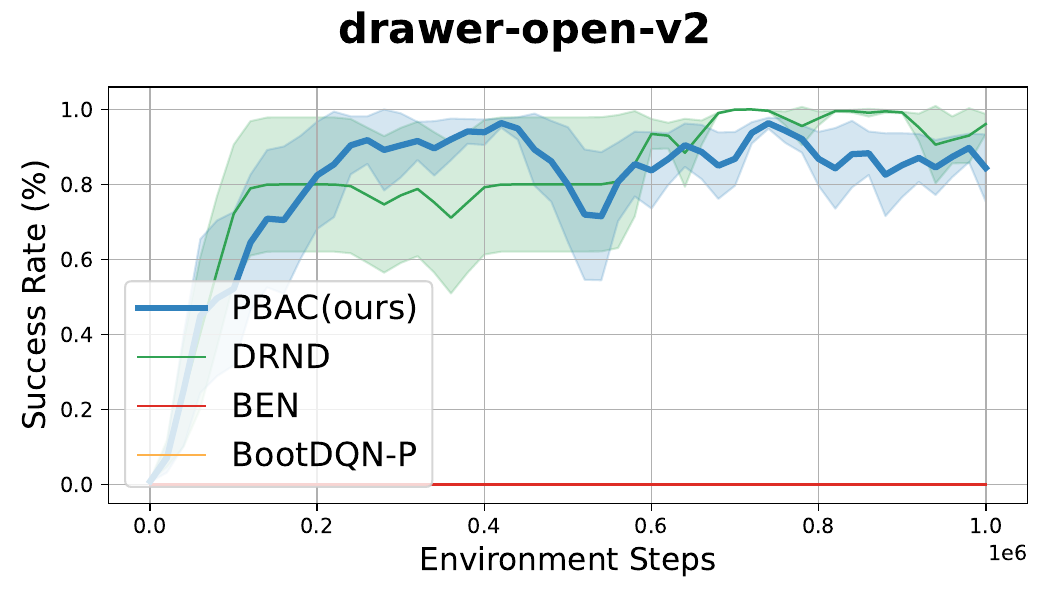}
    \includegraphics[width=0.33\textwidth]{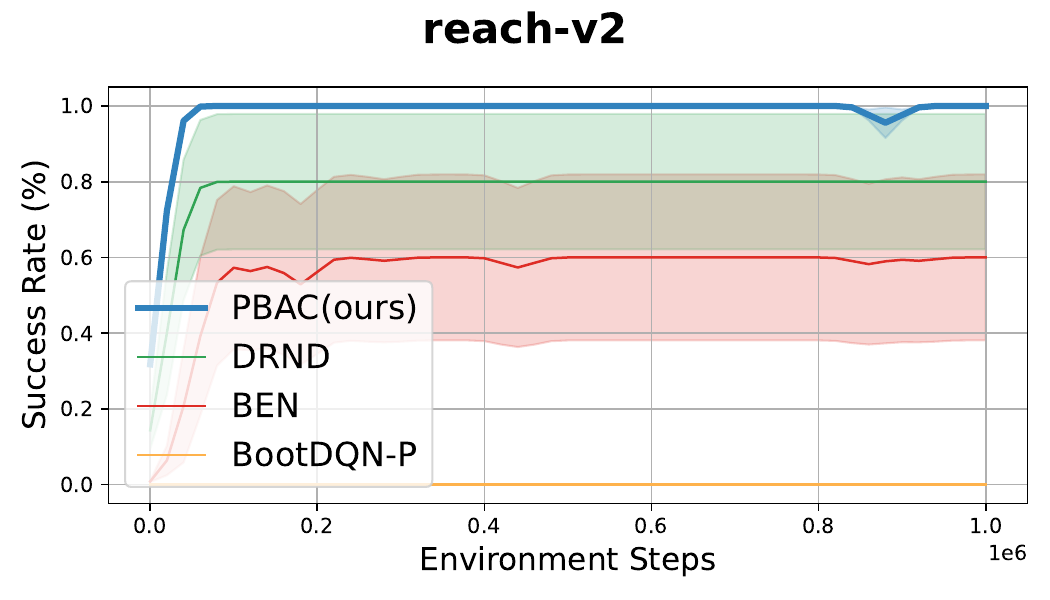}
    \includegraphics[width=0.33\textwidth]{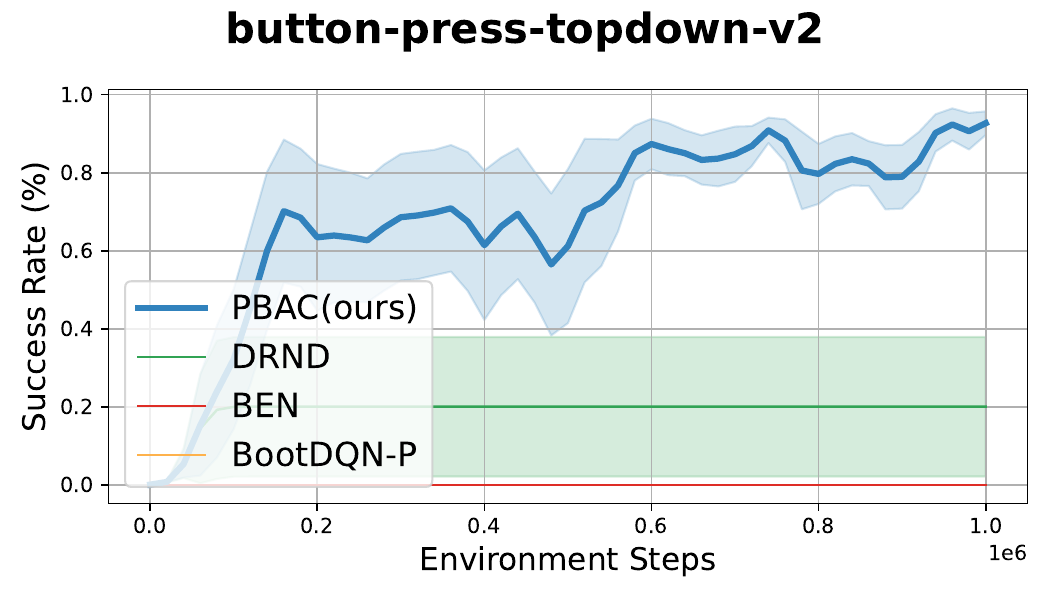}
    \includegraphics[width=0.33\textwidth]{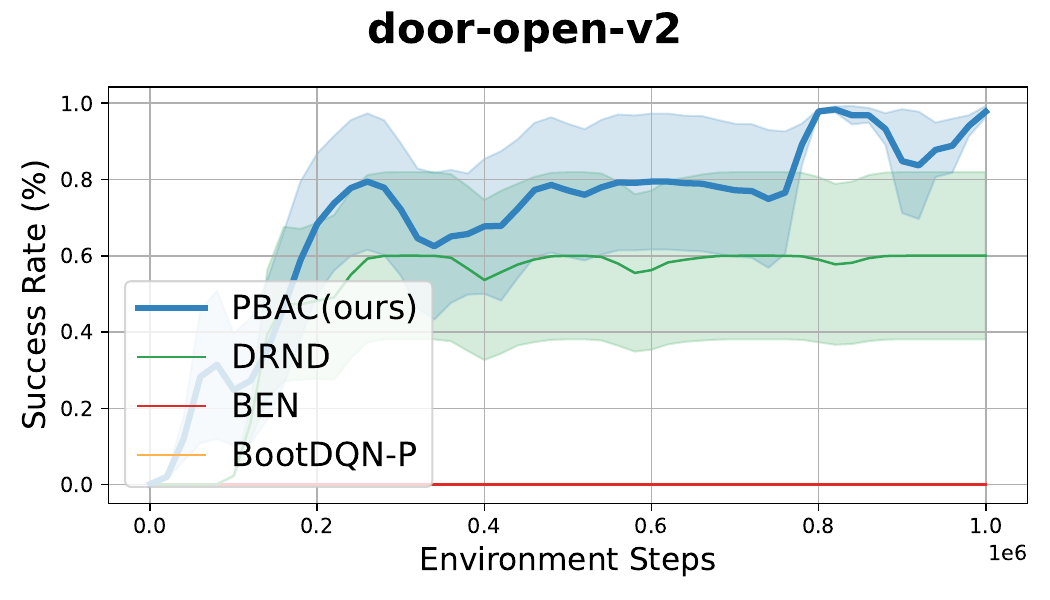}
    
    \caption{
    \emph{Success rate curves.} The success rate curves for sparse Meta-World
    environments throughout training corresponding to the results presented in \cref{tab:success_result_table}. 
    Visualized are the mean over five seeds.
    }
    \label{fig:metaworld_curves}
    \vspace{2.0em}
\end{figure}

\subsection{State space visualizations}\label{appsec:ssvis}
Throughout the training, we record the currently visited state at regular intervals and plot them in five groups, for PBAC and each of the baselines.
We visualize two environments: cartpole and delayed ant. 
In each case, we record every 500th step over the whole training process and record the corresponding state. The whole set is split into five groups, and each is plotted as its own scatter plot.

\subsubsection{Cartpole}

Of cartpole's five state dimensions, only the first two are interpretable, providing the position of the cart and the cosine of the angle of the pole. We visualize them in \cref{fig:teaser_exploration}. As discussed above, the agent receives rewards only if it manages to rmains close to zero with an upward pole, i.e., an angle close to one. 

PBAC is able to quickly explore the state space and then concentrate on visiting the states with high reward (red zone), while remaining flexible enough to explore further regions.
BEN similarly is able to quickly identify the target region but begins to fully exploit it and loses its flexibility in the process. As reported in \cref{tab:result_table}, they end each episode with similar performance. 

Similar to BEN, BootDQN-P has no problem exploring the state space; however, it never manages to find the narrow target and thus never converges.
BEN starts exploring a wide range of states but ultimately becomes stuck in this seed without being able to find the target. As shown in \cref{fig:all_curves}, BEN's performance varies greatly depending on the random initial seed in this environment. As such, this represents a random example of a failure case, not of its general performance on cartpole.
DRND quickly becomes stuck as well within a small subset of the state space, essentially exploring only the position of the cart without ever being able to significantly change the angle of the pole.

\subsection{Effects of the hyperparameters}\label{appsec:ablation}
We evaluate the sensitivity of the training process of PBAC with respect to its three main hyperparameters: bootstrap rate (BR) $\kappa$, posterior sampling rate (PSR), and prior variance (PV) $\sigma_0$ of $\rho_0$, on two environments, cartpole and delayed ant. 
Depending on the environment, PBAC is either sensitive to their choice (see \cref{fig:ablation_pbac_cartpole} on the cartpole environment) and shows clearly interpretable patterns, or it remains insensitive to their choice as in the delayed ant environment visualized in \cref{fig:ablation_pbac_ant}.

\begin{figure}
    \centering
    \includegraphics[width=0.95\linewidth]{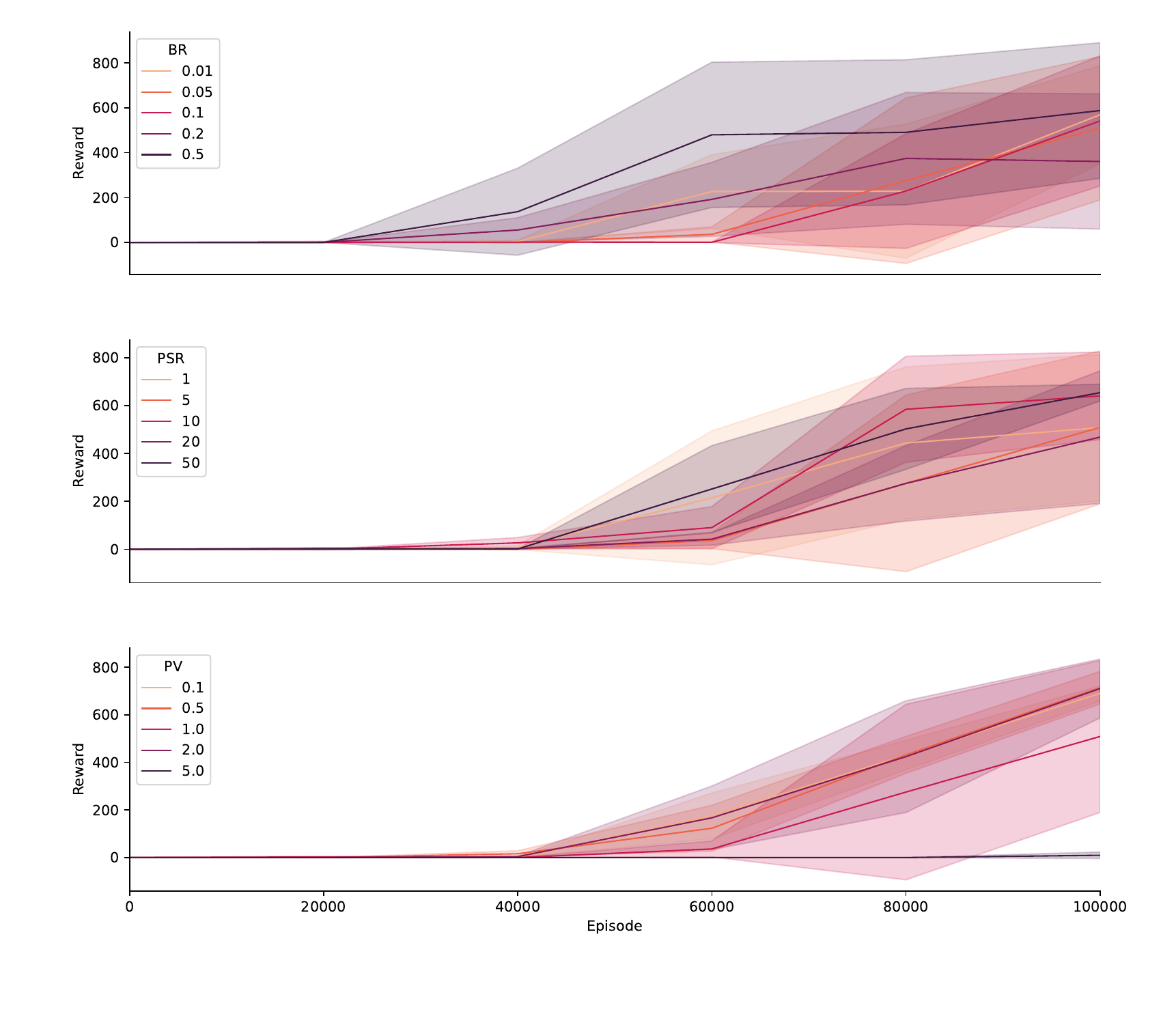}
    \caption{\emph{Ablation results on cartpole.} We show varying bootstrap rates (BR), posterior sampling rates (PSR), and prior variances (PV). While PBAC is mostly robust in terms of varying prior variances, increases in the bootstrap rate and decreases in the posterior sampling rate delay the learning process. Visualized are the interquartile mean together with the interquartile range over three seeds.
    }
    \label{fig:ablation_pbac_cartpole}
\end{figure}

\begin{figure}
    \centering
    \includegraphics[width=0.95\linewidth]{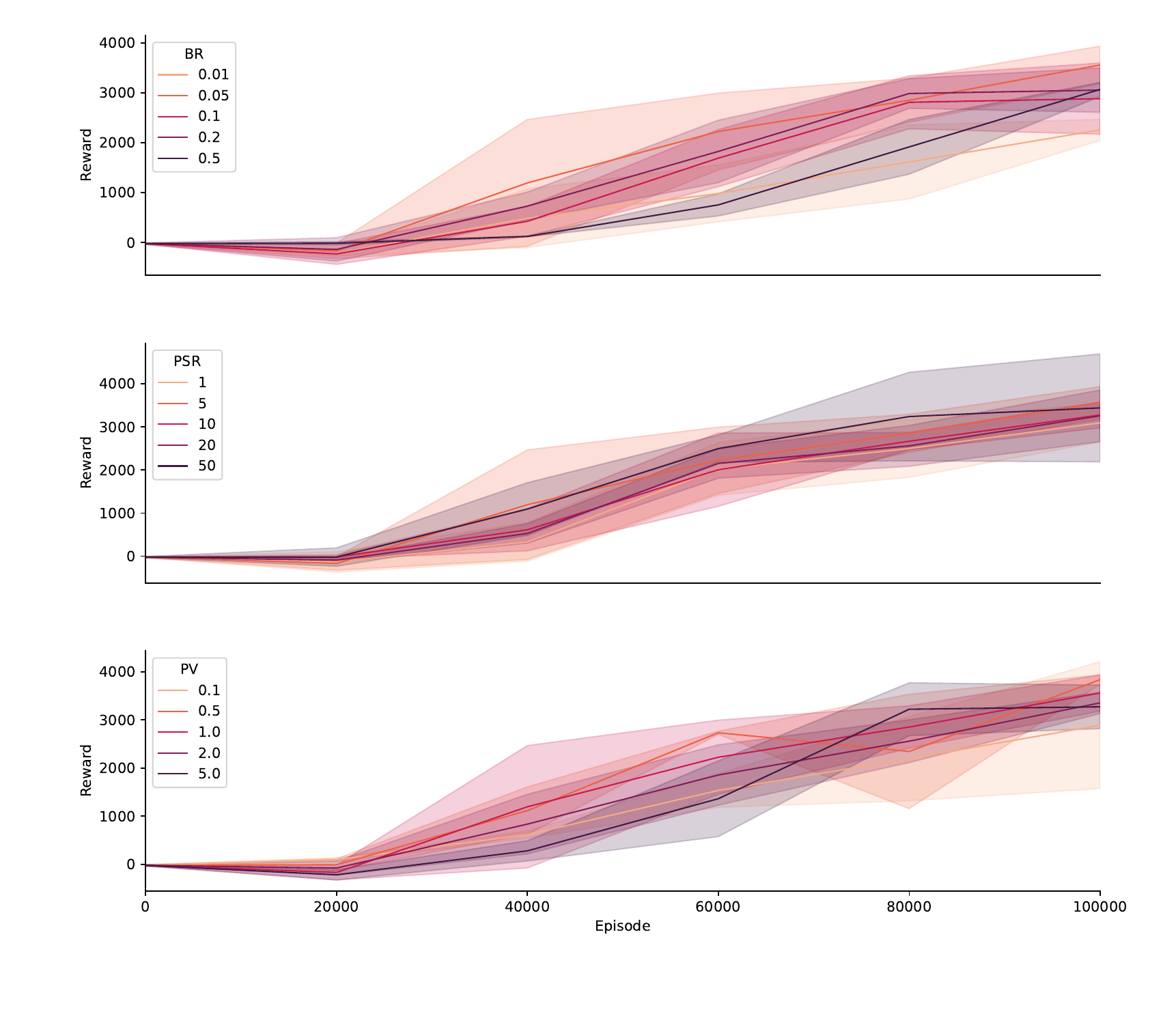}
    \caption{\emph{Ablation results on delayed ant.} We show varying bootstrap rates (BR), posterior sampling rates (PSR), and prior variances (PV). Compared to the ablation on cartpole (see \cref{fig:ablation_pbac_cartpole}), PBAC is robust against variations in all three hyperparameters in the delayed ant environment. Visualized are the interquartile mean together with the interquartile range over three seeds.
    }
    \label{fig:ablation_pbac_ant}
\end{figure}

\subsection{Effects of the individual loss terms}\label{appsec:lossablation}
The results are provided in \cref{fig:individual_term}, \cref{tab:combined_table}, and \cref{tab:area_under_curve}. Of the individual terms, the coherence term is the most important, often even surpassing the full bound in performance. However, usage of this term by itself would reduce the principled bound to using a heuristic.

\begin{table}
    \centering
    \caption{Combined Performance Table: Final IQM reward values over five seeds}
     \vspace{2.0em}
    \label{tab:combined_table}
    \resizebox{\textwidth}{!}{%
    \begin{tabular}{lccccccc}
    \toprule
    & \textbf{Full} & \textbf{Coh + Pro} & \textbf{Div + Pro} & \textbf{Div + Coh} & \textbf{Coh} & \textbf{Div} & \textbf{Pro} \\
    \midrule
    ant & 6486.38 $\sd{\pm 46.34}$ & 6443.09 $\sd{\pm 69.78}$ & 6114.77 $\sd{\pm 88.27}$ & 6235.94 $\sd{\pm 52.95}$ & \textbf{6575.82 $\sd{\pm 106.43}$} & 6430.10 $\sd{\pm 57.56}$ & -28.73 $\sd{\pm 13.38}$ \\
    ant (delayed) & 5370.94 $\sd{\pm 47.58}$ & 5285.30 $\sd{\pm 106.44}$ & 5512.53 $\sd{\pm 54.91}$ & 5482.34 $\sd{\pm 52.85}$ & 5496.36 $\sd{\pm 68.61}$ & \textbf{5580.13 $\sd{\pm 48.50}$} & -74.01 $\sd{\pm 28.26}$ \\
    ant (very delayed) & 4584.05 $\sd{\pm 463.67}$ & 2265.29 $\sd{\pm 758.75}$ & 4795.91 $\sd{\pm 147.57}$ & 3504.50 $\sd{\pm 1581.00}$ & \textbf{5062.85 $\sd{\pm 129.49}$} & 4071.91 $\sd{\pm 1509.62}$ & -82.74 $\sd{\pm 25.41}$ \\
    \midrule
    hopper & 1445.37 $\sd{\pm 174.89}$ & 1757.47 $\sd{\pm 330.15}$ & 1740.61 $\sd{\pm 97.33}$ & 1669.02 $\sd{\pm 94.84}$ & \textbf{2395.87 $\sd{\pm 451.57}$} & 1473.30 $\sd{\pm 307.96}$ & 12.64 $\sd{\pm 3.61}$ \\
    hopper (delayed) & 804.54 $\sd{\pm 67.49}$ & 883.97 $\sd{\pm 112.21}$ & 820.58 $\sd{\pm 80.15}$ & 869.30 $\sd{\pm 86.83}$ & \textbf{1016.46 $\sd{\pm 123.19}$} & 837.01 $\sd{\pm 88.56}$ & -0.05 $\sd{\pm 0.04}$ \\
    hopper (very delayed) & 522.52 $\sd{\pm 39.47}$ & 627.71 $\sd{\pm 59.22}$ & 625.09 $\sd{\pm 22.20}$ & 633.29 $\sd{\pm 90.05}$ & \textbf{723.93 $\sd{\pm 59.10}$} & 672.12 $\sd{\pm 154.96}$ & -0.01 $\sd{\pm 0.00}$ \\
    \midrule
    humanoid & 1627.70 $\sd{\pm 381.57}$ & 1751.76 $\sd{\pm 985.24}$ & 1362.18 $\sd{\pm 395.08}$ & 867.66 $\sd{\pm 314.19}$ & 650.86 $\sd{\pm 146.60}$ & \textbf{3023.49 $\sd{\pm 1423.07}$} & 86.38 $\sd{\pm 6.77}$ \\
    humanoid (delayed) & 723.77 $\sd{\pm 727.71}$ & 1184.91 $\sd{\pm 833.83}$ & \textbf{2324.15 $\sd{\pm 158.99}$} & 579.67 $\sd{\pm 249.60}$ & 924.13 $\sd{\pm 449.09}$ & 607.18 $\sd{\pm 471.98}$ & -10.52 $\sd{\pm 1.36}$ \\
    \bottomrule
    \end{tabular}%
    }
\end{table}

\begin{table*}
    \centering
    \caption{Area under learning curve over five seeds.}
     \vspace{2.0em}
    \label{tab:area_under_curve}
    \begin{tabular}{lccccccc}
    \toprule
    & \textbf{Full} & \textbf{Coh + Pro} & \textbf{Div + Pro} & \textbf{Div + Coh} & \textbf{Coh} & \textbf{Div} & \textbf{Pro} \\
    \midrule
    ant & 4296.69 & 4367.53 & 3499.05 & 4181.68 & \textbf{4567.02} & 4026.58 & -59.30 \\
    ant (delayed) & 3447.80 & 3553.92 & 3476.81 & 3622.39 & 3634.59 & \textbf{3671.60} & -433.75 \\
    ant (very delayed) & 2106.61 & 1028.84 & 2441.88 & 2107.06 & \textbf{2548.01} & 2156.64 & -440.14 \\
    \midrule
    hopper & 1406.96 & 1629.28 & 1546.43 & 1558.10 & 1415.29 & \textbf{1724.74} &  21.51 \\
    hopper (delayed) & 777.06 &  772.01 & \textbf{871.38} & 792.61 & 683.21 & 759.15 & 2.34 \\
    hopper (very delayed) & 319.34 & 429.80  & 565.81 & \textbf{588.71} & 365.99 & 366.56 & -0.02 \\
    \midrule
    humanoid & 1164.12 & 1244.46 & 1314.70 & 1104.44 & 999.07  & \textbf{1448.97} & 120.41 \\
    humanoid (delayed) & 301.35 & 384.02 & \textbf{580.31} & 223.59 & 268.41 & 310.43 & -9.02 \\
    \bottomrule
    \end{tabular}%
\end{table*}

\begin{figure}
    \centering
    \includegraphics[width=0.49\textwidth]{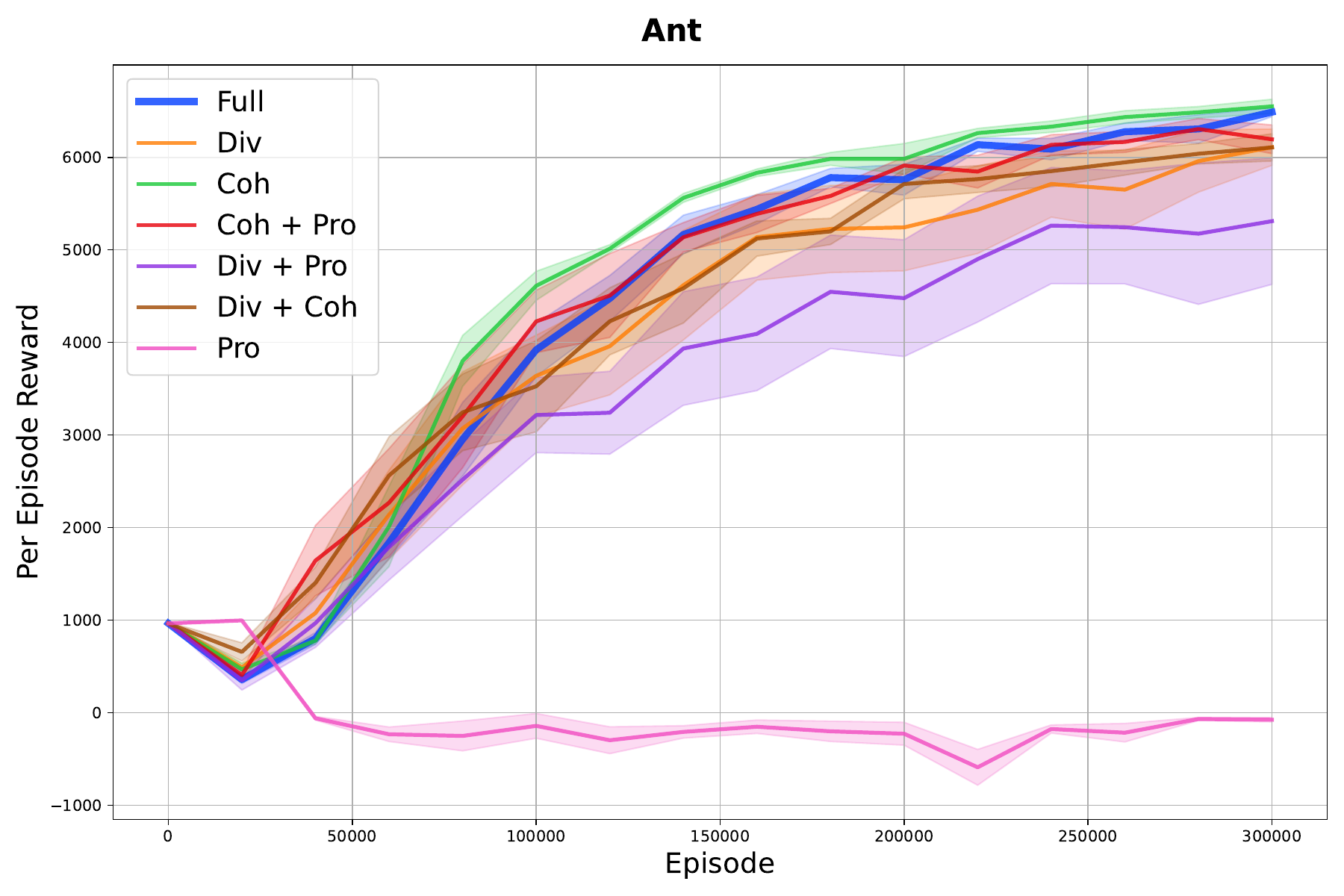}
    \includegraphics[width=0.49\textwidth]{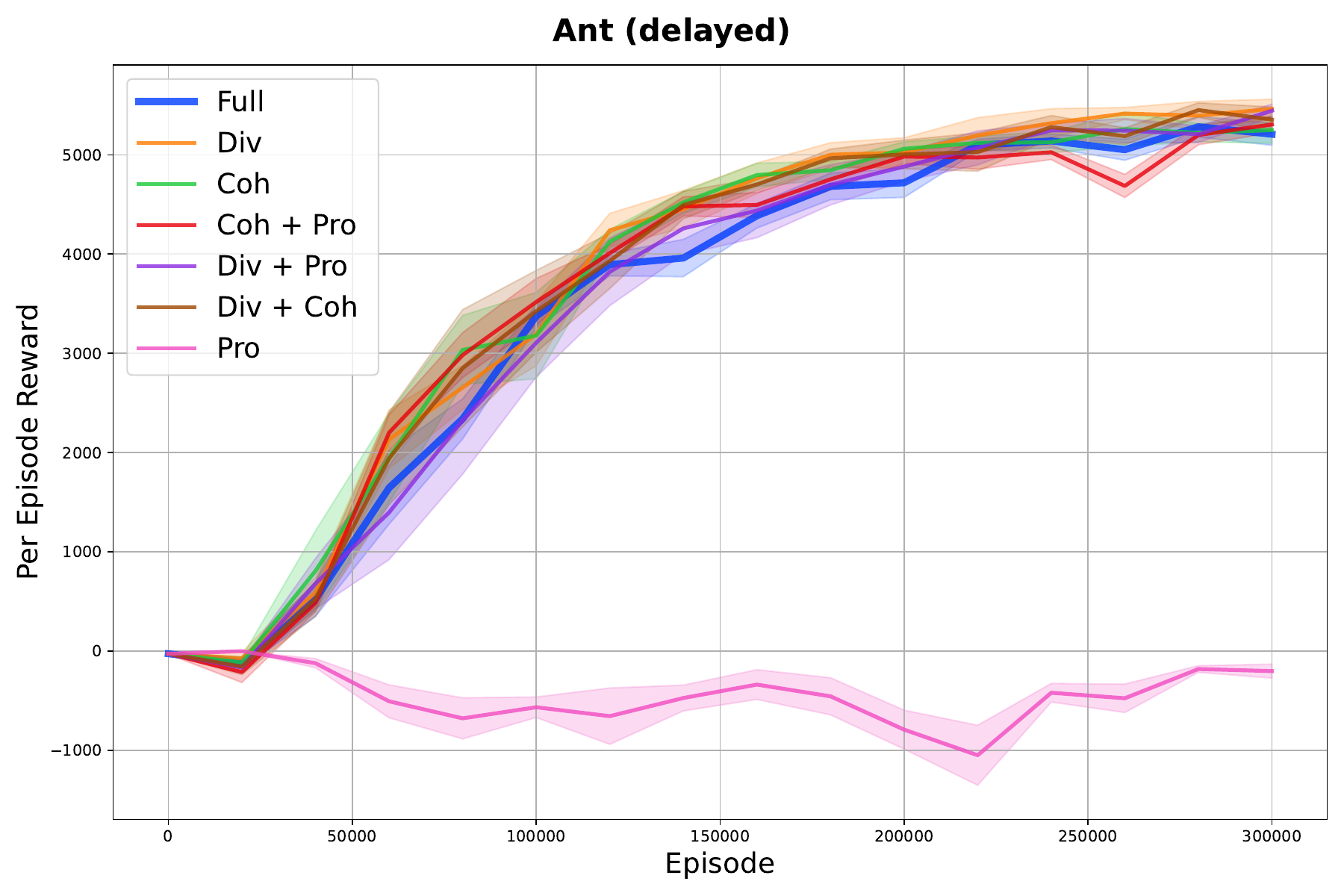}
    \includegraphics[width=0.49\textwidth]{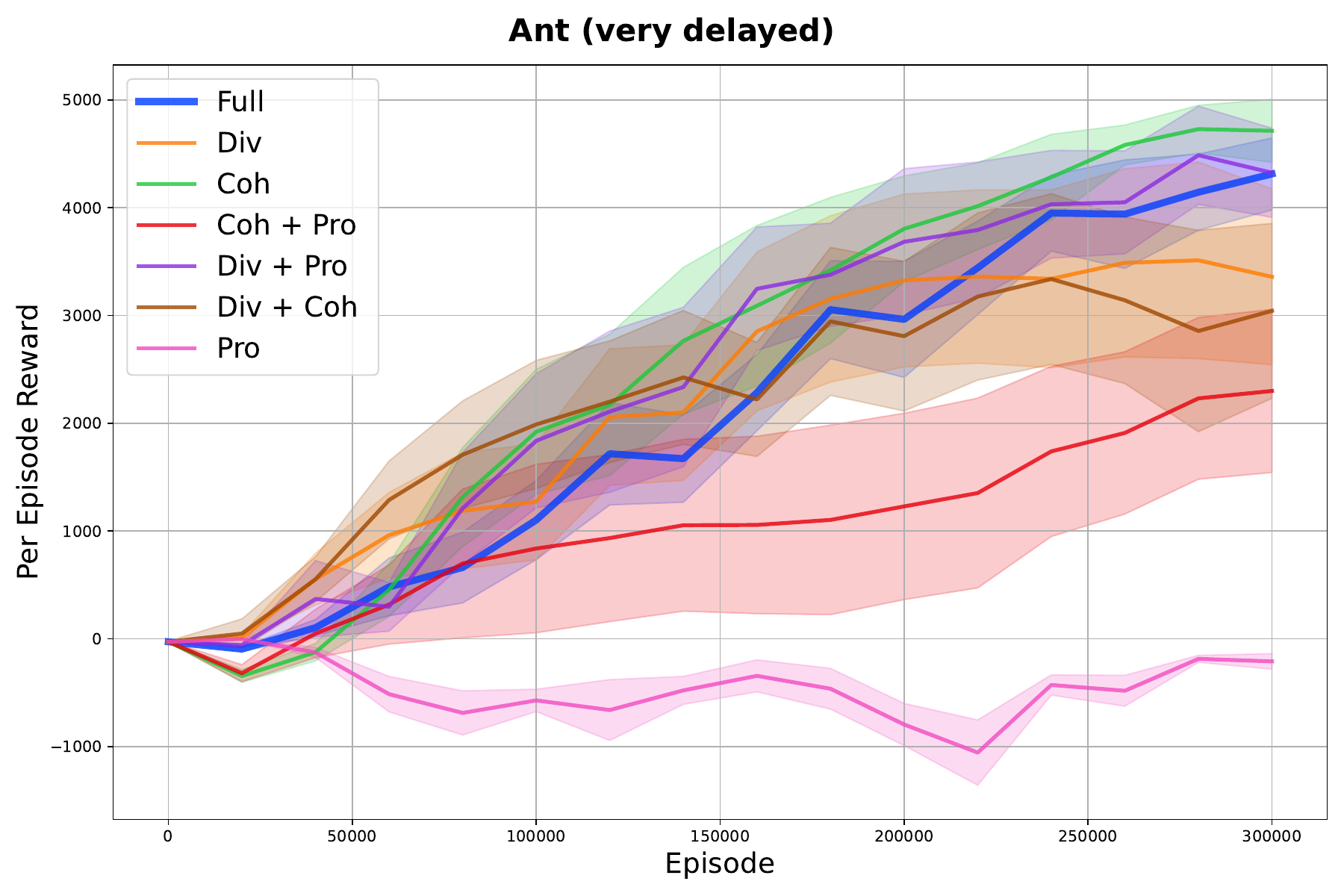}
    \includegraphics[width=0.49\textwidth]{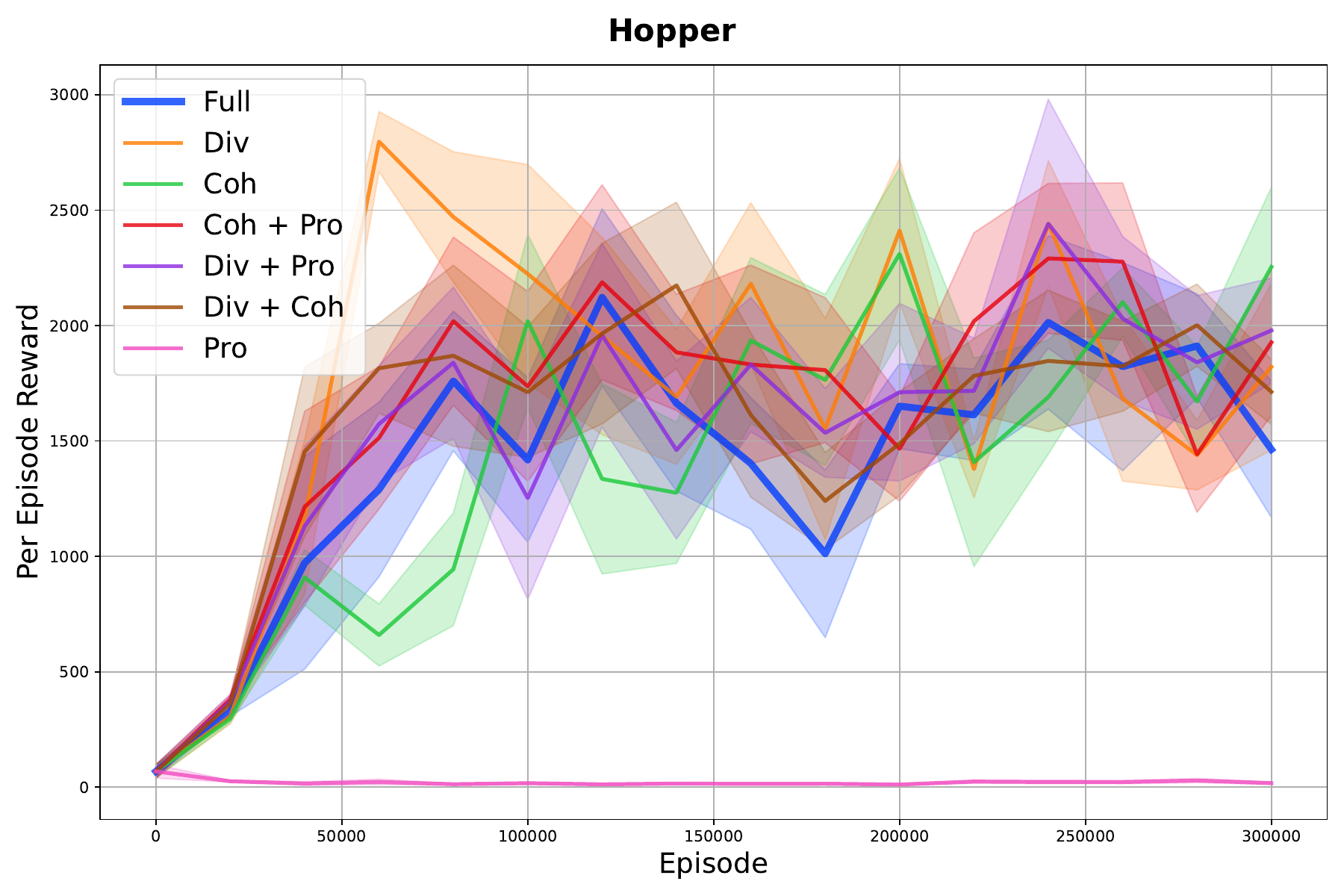}
    
    \caption{
    \emph{Reward curves.} Effect of individual loss terms on performance
    Visualized where the interquartile mean together with the interquartile range over five seeds are shown.
    }
    \label{fig:individual_term}
\end{figure}

\begin{figure}
    \centering
    \includegraphics[width=0.49\textwidth]{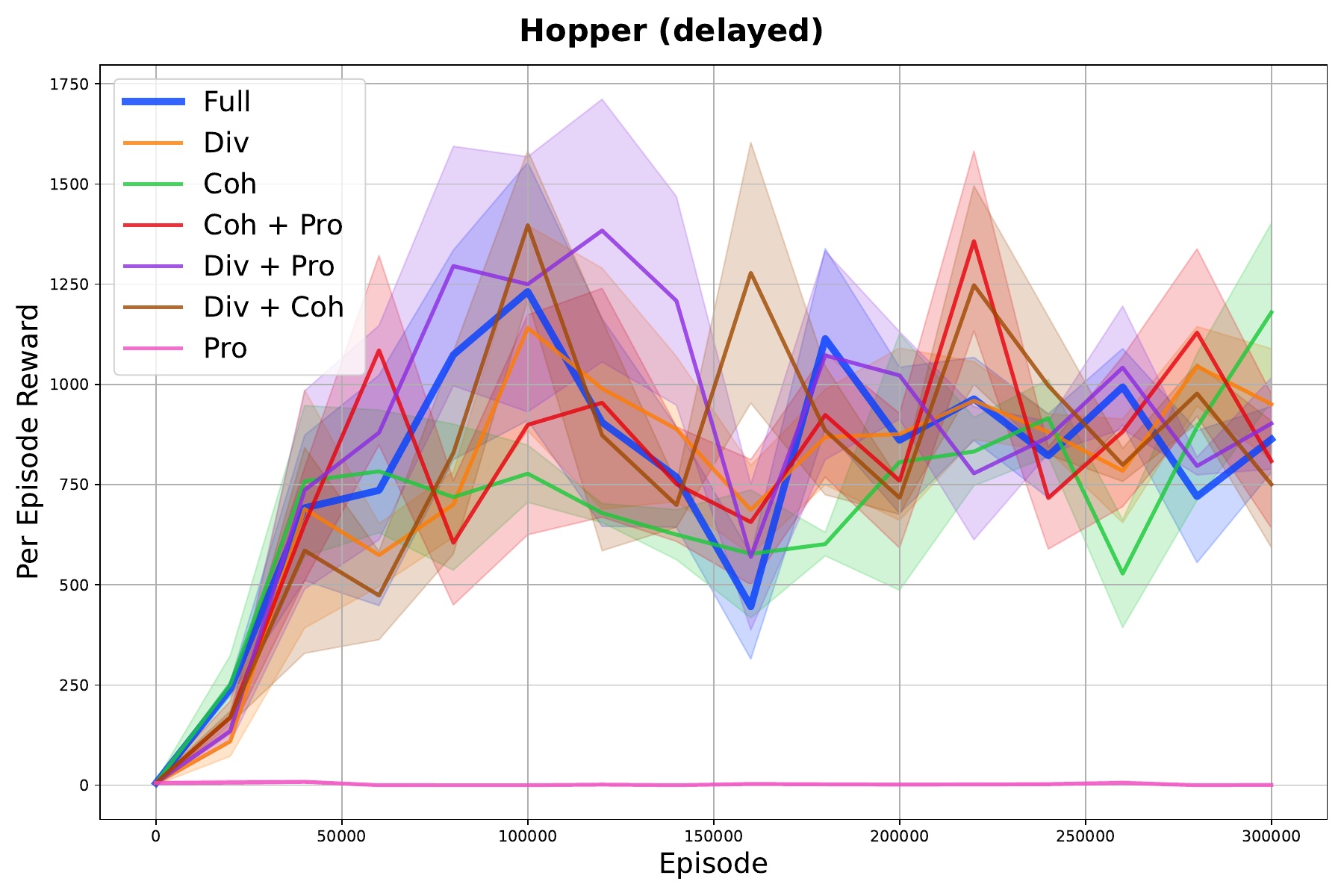}
    \includegraphics[width=0.49\textwidth]{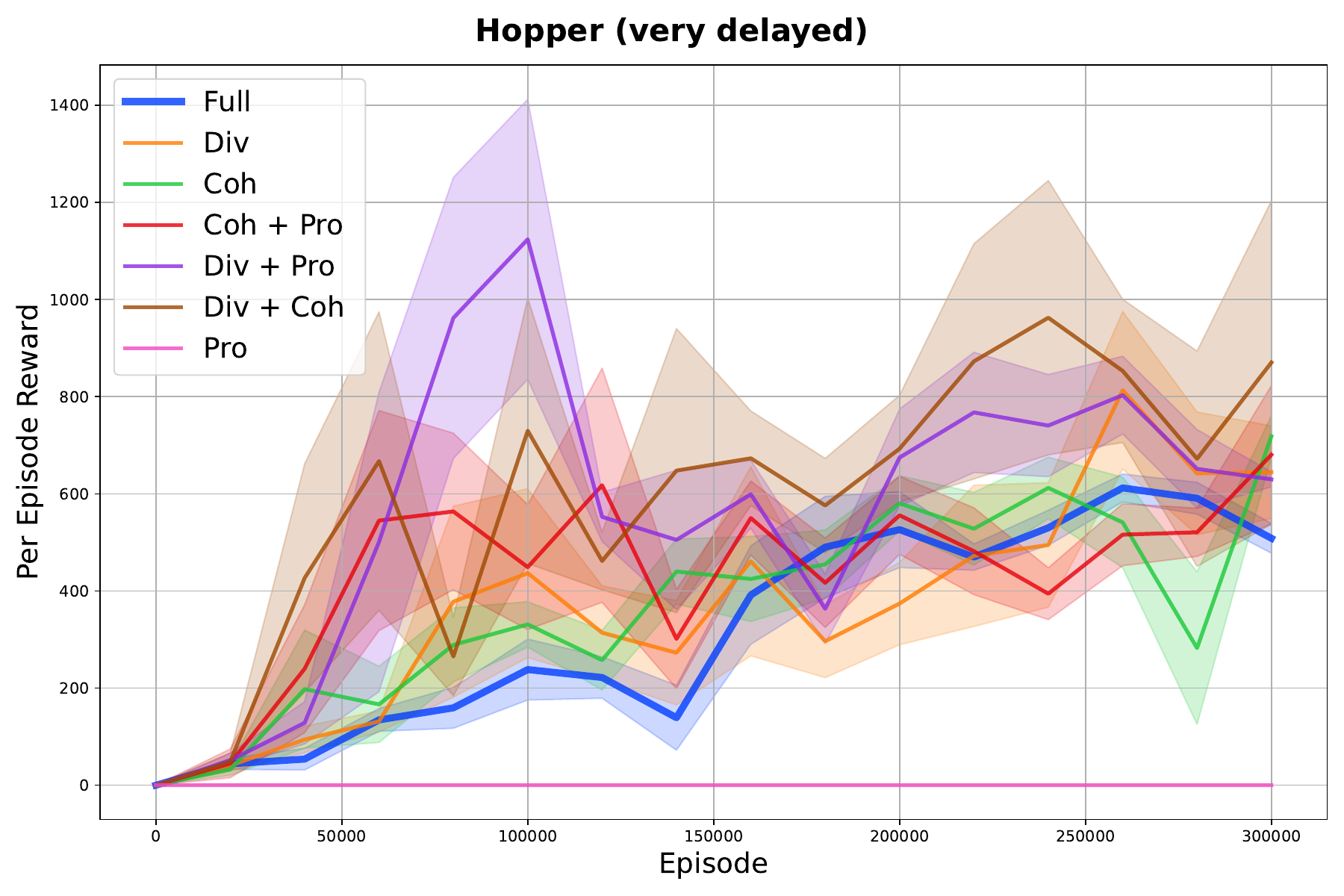}
    \includegraphics[width=0.49\textwidth]{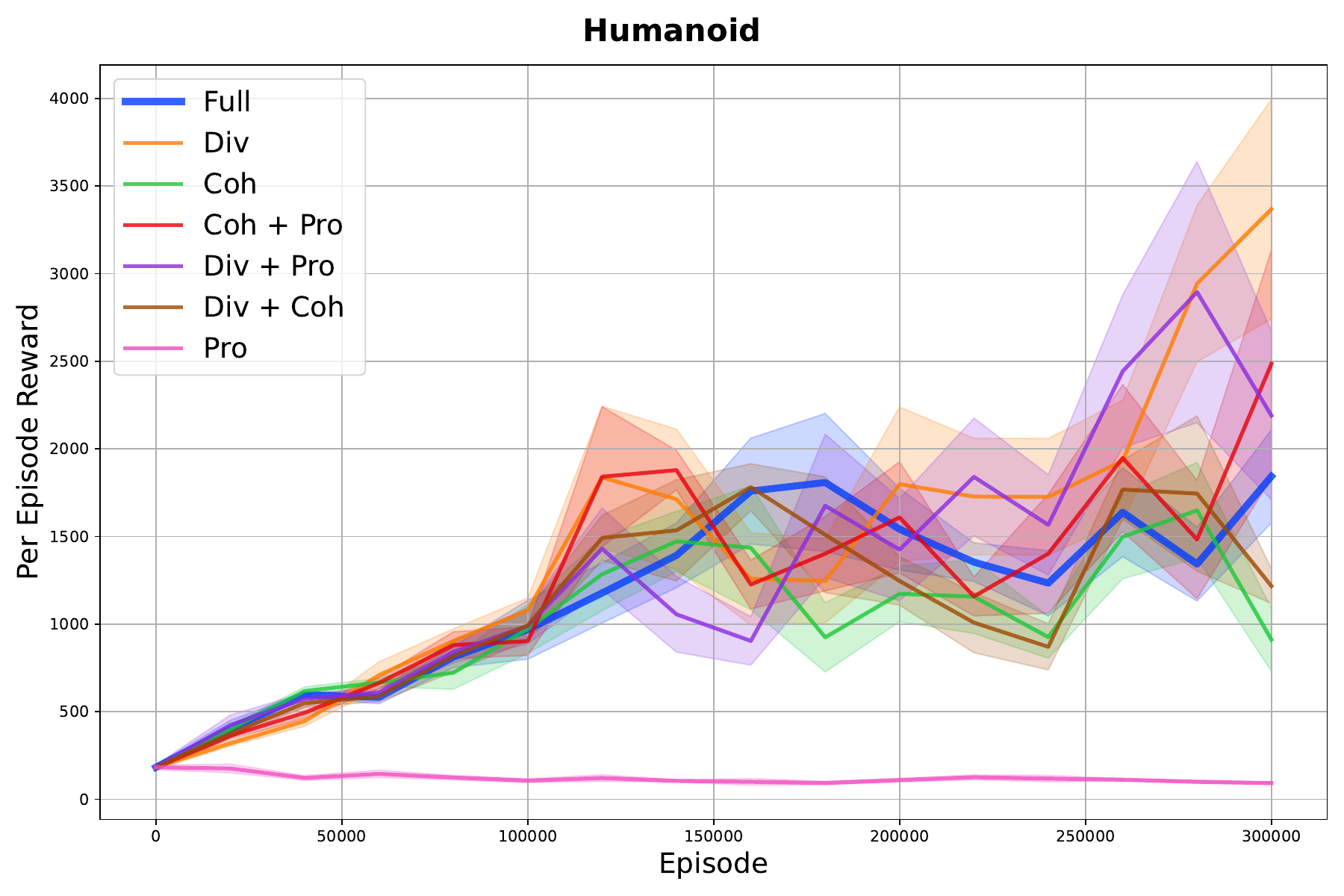}
    \includegraphics[width=0.49\textwidth]{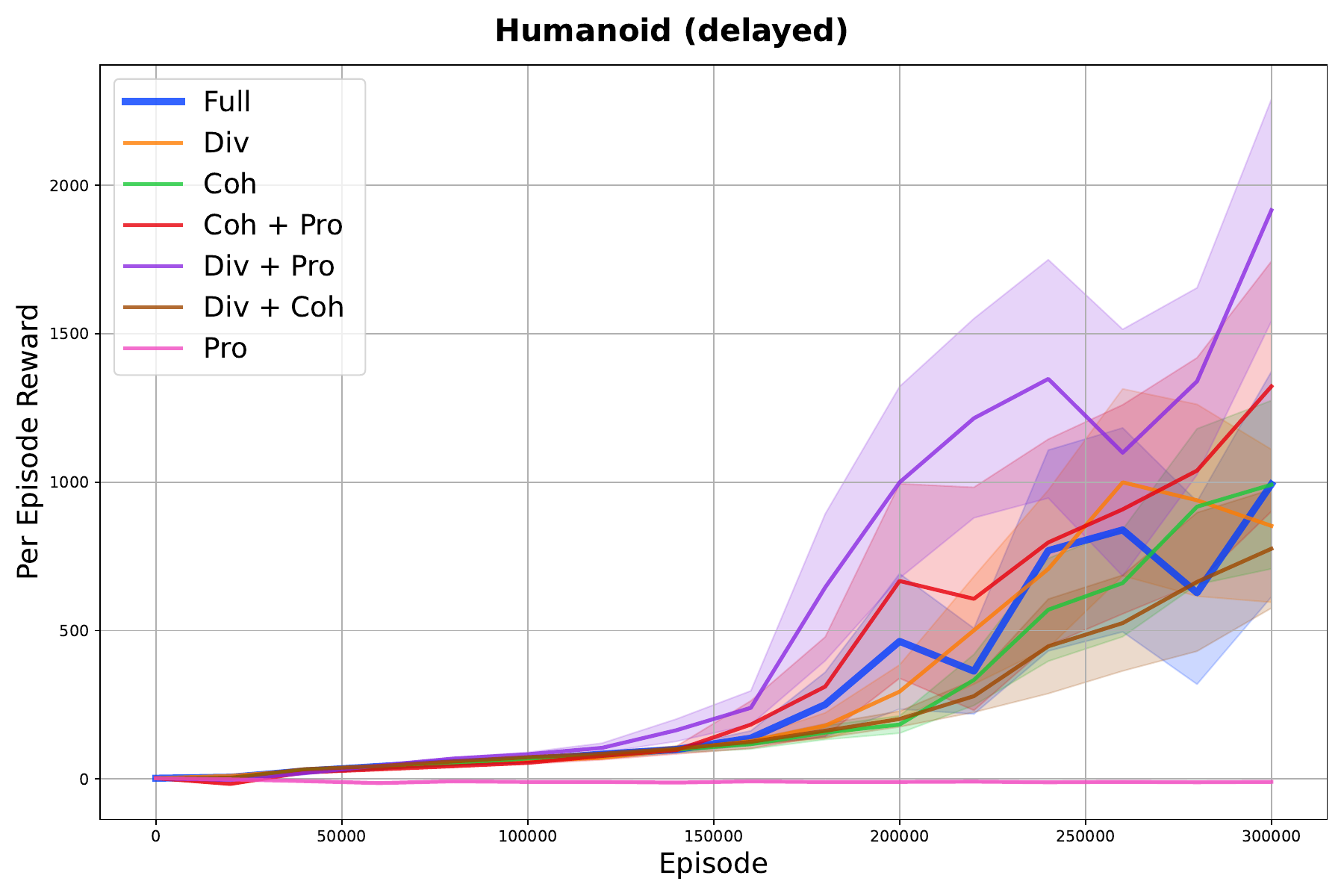}
    
    \caption{
    \emph{Reward curves (continued).} Effect of individual loss terms on performance
    Visualized where the interquartile mean together with the interquartile range over five seeds are shown.
    }
    \label{fig:individual_term}
\end{figure}

\subsection{Computational cost}
To compare the computational efficiency of the proposed method with existing baselines, we ran all four models for \num{50000} environment steps on the Ant benchmark task. This short-run evaluation was not intended to assess final performance, but rather to illustrate the relative training cost over a fixed period. The goal was to demonstrate that our method does not impose significant additional computational overhead. The wall-clock time (in seconds) required to complete 50k steps was as follows: PBAC: \SI{998}{\second}, BEN: \SI{997}{\second}, BootDQN-P:\SI{1154}{\second}, and DRND:\SI{1677}{\second}. These results indicates that PBAC maintains a runtime that is in line with other baseline methods, indicating that its additional theoretical structure does not result in noticeably higher computational demands.

\end{document}